\documentclass[11pt,fullpage]{article}
\usepackage{amsfonts,amsmath,amssymb,amsthm}
\usepackage{graphicx,caption,subcaption}
\usepackage{fullpage}
\usepackage{psfrag}
\usepackage{fancyhdr}
\usepackage{color,enumerate}
\usepackage{amsxtra,amscd}
\usepackage{setspace}
\usepackage{rotating}
\usepackage{placeins}
\usepackage{hyperref}
\hypersetup{colorlinks=true,breaklinks=true,linkcolor=blue,anchorcolor=red,citecolor=blue}
\usepackage{ragged2e}
\usepackage{multirow,rotating,makecell,longtable,threeparttable,booktabs,caption,lscape}
\usepackage{algorithm,algorithmicx,algpseudocode}
\usepackage{mathtools}
\usepackage{tikz,pgfplots}
\pgfplotsset{compat=newest}
\usepackage{datetime}
\ifdim\overfullrule>0pt % if draft option is active
  \usepackage{environ}

  \NewEnviron{tikzpicture}{%
    \begin{pgfpicture}
    \pgfpathrectanglecorners{\pgfpointorigin}{\pgfpoint{3cm}{3cm}}%
     \pgfusepath{stroke}\end{pgfpicture}%
  }
\fi
\newtheorem{theorem}{Theorem}

\newtheorem{lemma}{Lemma}

\newtheorem{assumption}{Assumption}
\usepackage{natbib}
 \bibpunct[, ]{(}{)}{,}{a}{}{,}%
\newcommand{\comment}[1]{}

\newcommand{\field}[1]{\ensuremath{\mathbb{#1}}}
\newcommand{\N}{\ensuremath{\field{N}}} % natural numbers
\newcommand{\R}{\ensuremath{\field{R}}} % real numbers
\newcommand{\PR}{\ensuremath{\mathsf{P}}} % probability
\newcommand{\E}{\ensuremath{\mathsf{E}}} % expectation

\newcommand{\banditcount}[2]{C_{#1,#2}}
\newcommand{\banditcountoracle}[2]{\tilde{C}_{#1,#2}}

\newcommand{\banditmean}{m}
\newcommand{\banditmeanoracle}{\tilde{m}}
\newcommand{\banditwidth}{w}
\newcommand{\banditwidthoracle}{\tilde{w}}

\newcommand{\tpi}{\ensuremath{\tilde{\pi}}}
\newcommand{\setk}{\ensuremath{\mathcal{K}}}
\newcommand{\sett}{\ensuremath{\mathcal{T}}}
\newcommand{\sets}{\ensuremath{\mathcal{S}}}
\newcommand{\setd}{\ensuremath{\mathcal{D}}}
\newcommand{\setf}{\ensuremath{\mathcal{F}}}

\newcommand{\NB}{\sup_{v\in[0,1/2]}|\tilde{\epsilon}_k(v)|}
\newcommand{\PB}{\sup_{v\in[0,1/2]}|\tilde{y}_k(v)|}
\newcommand{\NBtight}{\sup_{v\in\Uscr_k}\abs{\tilde{\epsilon}_k{(v)}}}
\newcommand{\ProbNB}{1-\frac{48}{n^{H^2-1}}}
\newcommand{\ProbNBtight}{1-\frac{200}{n^{0.867 H^2-1}}-\frac{200}{n^{0.694  H^2 -1}}}
\newcommand{\ProbNBall}{1-\frac{48}{n^{H^2-1}}-\frac{200}{n^{0.867 H^2-1}}-\frac{200}{n^{0.694  H^2 -1}}}
\newcommand{\RegretBound}{\sqrt{Td\log^2(T)\log(T/d)}}

\newcommand{\banditset}{\ensuremath{\mathcal E}}

\newcommand{\Ascr}{\ensuremath{\mathcal A}}

\newcommand{\Escr}{\ensuremath{\mathcal E}}

\newcommand{\Gscr}{\ensuremath{\mathcal G}}
\newcommand{\Hscr}{\ensuremath{\mathcal H}}

\newcommand{\Kscr}{\ensuremath{\mathcal K}}

\newcommand{\Mscr}{\ensuremath{\mathcal M}}
\newcommand{\Nscr}{\ensuremath{\mathcal N}}

\newcommand{\Tscr}{\ensuremath{\mathcal T}}
\newcommand{\Uscr}{\ensuremath{\mathcal U}}
\newcommand{\Vscr}{\ensuremath{\mathcal V}}

\DeclareMathOperator*{\argmin}{\mathrm{argmin}}
\DeclareMathOperator*{\argmax}{\mathrm{argmax}}
\DeclarePairedDelimiter\abs{\lvert}{\rvert}
\DeclarePairedDelimiter\floor{\lfloor}{\rfloor}
\DeclarePairedDelimiter\ceil{\lceil}{\rceil}

\setcitestyle{authoryear,round,semicolon,aysep={,},yysep={,},notesep={, }}
\begin{document}

\title{\Large \bf Learning and Optimization with Seasonal Patterns}
\author{
{\bf Ningyuan Chen} \\
{\small {\em Department of Management, University of Toronto Mississauga,}} \\
{\small {\em Rotman School of Management, University of Toronto, Canada, ningyuan.chen@utoronto.ca}} \vspace{3mm} \\
{\bf Chun Wang} \\
{\small {\em School of Economics and Management, Tsinghua University, Beijing, China, wangch5@sem.tsinghua.edu.cn}} \vspace{3mm}\\
{\bf Longlin Wang} \\
{\small {\em School of Economics and Management, Tsinghua University, Beijing, China, wangll3.16@sem.tsinghua.edu.cn}} \\ \\
%{\bf Preliminary and Incomplete. Please Do Not Circulate} \\
}

%\date{ \today \\  }
\date{This version: August 2021. The first version: June 2020.}

\maketitle
\thispagestyle{empty}

\begin{abstract}
A standard assumption adopted in the multi-armed bandit (MAB) framework is that the mean rewards are constant over time.
This assumption can be restrictive in the business world as decision-makers often face an evolving environment where the mean rewards are time-varying.
In this paper, we consider a non-stationary MAB model with $K$ arms whose mean rewards vary over time in a periodic manner.
The unknown periods can be different across arms and scale with the length of the horizon $T$ polynomially.
We propose a two-stage policy that combines the Fourier analysis with a confidence-bound-based learning procedure to learn the periods and minimize the regret.
In stage one, the policy correctly estimates the periods of all arms with high probability.
In stage two, the policy explores the periodic mean rewards of arms using the periods estimated in stage one and exploits the optimal arm in the long run.
We show that our learning policy incurs a regret upper bound $\tilde{O}(\sqrt{T\sum_{k=1}^K T_k})$ where $T_k$ is the period of arm $k$.
Moreover, we establish a general lower bound $\Omega(\sqrt{T\max_{k}\{ T_k\}})$ for any policy.
Therefore, our policy is near-optimal up to a factor of $\sqrt{K}$.
\end{abstract}
\vspace{-15mm} ~\\

\begin{center}
\end{center}
\clearpage
%%%%%%%%%%%%%%%%%%%%%%%%%%%%%%%%%%%%%%%%%%%%%%%%%%%%%%%%%%%%%%%%%%%%%%

\section{Introduction}
\label{sec:introduction}

\subsection{Motivation}
Online learning, or more specifically, the multi-armed bandit (MAB) problem, focuses on the task of learning the reward distributions from an unknown environment while simultaneously optimizing cumulative rewards over a fixed time horizon $T$.
This problem has been studied extensively when the environment (reward distributions) is stationary over time, with numerous algorithms proposed to tackle the trade-off between exploration and exploitation when making decisions (see \citealt{bubeck2012regret} for a comprehensive review).

While the stationarity assumption about the reward distributions greatly simplifies the analysis, it does not hold in many decision problems in OR/MS and other fields when the environment is time-varying.
For example, when experimenting with different prices and learning the optimal one, a fashion retailer should take into account the seasonal demand shift when setting the prices for apparel.
Despite the practical relevance, it is difficult to develop a learning policy for non-stationary rewards, especially when the dynamics can change arbitrarily over time.
Recent studies (e.g., \citealt{besbes2015non}) have considered cases in which the environment does not change fast with respect to the length of the time horizon,
e.g., when a budget sublinear in $T$ is imposed on the total variation of the underlying reward distributions.
The restriction on the total variation plays a key role in keeping the MAB problem tractable, as a fast-changing environment would render any learned information obsolete immediately.
It seems hopeless to develop an effective learning policy in a fast-changing environment.

However, there is still a silver lining in spite of the challenge. We note that many non-stationary dynamics in practice display seasonality.
For example, the demand for winter apparel usually has a yearly cycle. If the fashion company manager correctly estimates the demand fluctuation within a year, she may set retail prices differently over seasons to maximize the total revenue.
Online advertising provides another motivating example. The advertisers bid on impressions in an ad exchange.
When the values of different types of impressions are unknown, the advertiser is facing an online learning problem. Moreover, the rewards (click-through rate) of the impressions from a website have seasonal patterns depending on the traffic and demographics of the visitors (see \citealt{villamediana2019destination}).
A periodicity assumption on the environment where the model parameters repeat values over cycles may capture the essential feature of real-world decision problems of this kind.
At the same time, this assumption could make it possible to design efficient learning algorithms while allowing the changes to occur quite rapidly
(Suppose that the magnitude of change over a cycle is a positive constant, then the total changes is $O(T)$ as $T$ grows).

In this paper, we study online learning for non-stationary environments with seasonal patterns.
Our research is motivated by the aforementioned practical examples and the fact that the current learning algorithms have limitations when handling rapid changes (linear in $T$).
Specifically, we study the problem under the MAB framework with $K$ arms (decisions) for the decision-maker (DM) to choose from at each epoch.
We assume that each arm $k$ generates a random reward, whose mean varies over time periodically with period $T_k$.
The DM does not know the length of the period or the mean reward of any arm initially, and her goal is to maximize the total expected reward over the horizon.

\subsection{Novelty and Contribution}
We propose a new formulation for the non-stationary MAB problem, and we contribute to the online learning literature in the domain of algorithm design and analysis.

\noindent
\textbf{Formulation.}
We impose few structural assumptions except for the periodic reward distributions.
Hence, our formulation is very general and could provide an adequate representation of many real-world applications.
Since the DM has little knowledge of the environment and we allow for sufficient flexibility in modeling the periodicity, the standard MAB techniques are not applicable due to the following challenges:
\begin{itemize}
\item
The lengths of the periods $T_k$'s are unknown.
Without first learning $T_k$'s, it is impossible for the DM to track the evolution of the rewards and estimate their means over cycles.
Although most periods in real-world applications are either daily, weekly, monthly or annual, it is typically unclear which one is the most prominent beforehand.
For example, \citet{chen2018can} detect the presence of a surprising weekly cycle in addition to the usually believed daily cycle, in the setting of healthcare management.

\item
The periods of arms may be asynchronous.
If all arms share a common period, i.e., $T_1=...=T_K$, the DM can treat the decision scenarios at the same phase of each cycle as an independent stationary MAB, and then the whole learning problem is simply decomposed into $T_1$ independent subproblems after estimating the value of $T_1$.
In our setting, arms may have different periods, and thus the learning and decision making are inevitably nested across arms when they are in different phases of their cycles.
This setting is motivated by asset allocation, when the DM may face different asset classes with distinct business cycles.

\item
The lengths of periods $T_k$'s may scale with the horizon $T$.
If $T_k$'s are negligible relative to $T$, one may hope to use the above decomposition idea by considering the least common multiplier (LCM) of $T_k$'s, which is a common period for all arms.
It turns out that this scheme is practically infeasible, because $T_k$'s are often not small relative to $T$, and their LCM can be much larger than $T$ for even a moderate number of arms.
For example, most studies (e.g., \citealt{brown2005statistical}) investigating the arrival process of service systems explore datasets that span at most a few years.
A monthly cycle is hardly negligible within the time frame of several years.
In our formulation, we allow $T_k$'s to scale with $T$ at a polynomial rate, and explore whether the problem is still learnable.
\end{itemize}

\noindent
\textbf{Algorithm.} To resolve the above challenges, we develop a two-stage learning policy that features the following novel designs.
\begin{itemize}
\item
In stage one, we estimate the lengths of periods $T_k$'s for all arms using the discrete Fourier transform and techniques from signal processing, specifically spectral analysis.
Our policy has an intuitive threshold structure, and is capable of correctly identifying all $T_k$'s with high probability, even when they scale with $T$.

\item
In stage two, we draw on the estimated lengths of periods from stage one and use a confidence- bound-based algorithm to make the trade-off between exploration and exploitation.
One notable feature of our algorithm is that each arm retains its own confidence bound for each phase of its period, which enables us to circumvent the intractable approach of LCM subproblems and achieve significantly better performance.

\item \label{pg:reuse1}
While we estimate the periods and learn the rewards separately in stages one and two, respectively,
we make efforts to efficiently utilize the observed data in an integrated way.
In particular, the samples collected for period estimation in stage one are reused in stage two for learning the reward distributions.
This allows the algorithm to use data more efficiently when the horizon $T$ is small or moderate.

\end{itemize}

\noindent
\textbf{Analysis.} We analyze the performance of our policy in terms of regret.
\begin{itemize}
\item
We provide a \emph{finite-sample bound} for the probability of the correct identification of the lengths of periods of all arms.
Although spectral analysis is a classic topic, such a theoretical guarantee is new in the literature to the best of our knowledge.
Our result implies that not knowing the periods doesn't impede learning, as the lengths of periods can be learned efficiently relative to the learning of the mean rewards in each phase.

\item \label{pg:reuse2}
Data reuse introduces dependency among the samples and thus complicates the regret analysis.
By carefully controlling the dependency structure, we still manage to prove the regret bound.
This technique might be applicable to other algorithms with a similar two-stage design.

\item
As the main result of this paper, we prove a regret upper bound $\tilde{O}(\sqrt{T\sum_{k=1}^K T_k})$ for the proposed learning policy, where $\tilde{O}$ denotes the asymptotic rate omitting logarithmic terms.
Moreover, we establish a general lower bound $\Omega(\sqrt{T\max_{k}\{ T_k\}})$ for any policy.
Therefore, our algorithm is near-optimal up to a factor of $\sqrt{K}$.
We also derive the optimal regret for various special cases in Appendix \ref{sec:appendix-extended-results}, including all arms sharing a common period and having the same/different seasonality.
\end{itemize}

%\vspace{-0.5cm}
\subsection{Literature Review}
\label{sec:literature}
The study of non-stationary reward distributions combined with the classic framework of MAB is receiving significant attention recently.
In their seminal work, \citet{auer2002nonstochastic} propose an upper-confidence-bound (UCB) algorithm EXP3.S that can handle the MAB problem with a \emph{finite} number of changes.
Recent works have generalized the framework by incorporating continuous changes \citep{besbes2014stochastic,besbes2015non,besbes2019optimal}.
Unlike the classic stationary MAB problem, in this stream of literature the mean rewards of all arms are allowed to vary continuously over time.
The objective is to minimize the regret compared to the benchmark of the optimal arm at each epoch in that changing environment.
It is clear that if the change is arbitrary (e.g., an unpredictable shift at each epoch), then no algorithm can achieve a regret sublinear in $T$.
Hence, a budget is imposed on the total variation of mean rewards over the horizon, and this budget appears in the regret consequently.
The budget is known to the DM in \citet{besbes2014stochastic,besbes2015non}.
In later papers, the changing budget can be unknown and learned \citep{karnin2016multi,luo2017efficient,auer2019adaptively,cheung2019hedging,mao2021near}.
Our paper differs from the above literature in that changes can be linear in $T$, and we show that the regret is still controllable due to the periodic structure.

Various other settings of non-stationary bandits are investigated recently \citep{allesiardo2015exp3,allesiardo2017non,levine2017rotting,raj2017taming,liu2018change}.
\citet{jaksch2010near,zhou2020regime} focus on problems with specific structures such as MDP or POMDP, which allow linear changing budget.
\citet{di2020linear} investigate linear bandits in a seasonal setting which however follows a different definition from ours.
In particular, they assume the non-stationary rewards with change points while the past stationary states of the environment may reoccur, and they do not study periodic rewards.
\cite{cai2021periodic} assume that all (continuous) arms have the same and known period, and they use Gaussian processes to learn the seasonal pattern by specifying a novel periodic kernel function.
\cite{traca2021regulating} study MAB problems with a simpler seasonal structure: there is a known periodic function modulating the otherwise stationary rewards, and they investigate the modification of existing algorithms such as UCB and $\epsilon$-greedy.
\cite{lykouris2020bandits} adopt a similar setting but the modulating function can be adversarial.
The literature doesn't address our question when the periods are unknown and asynchronous.

Online learning with non-stationary dynamics has been studied extensively in the context of dynamic pricing.
Early papers assume one or a few change points \citep{besbes2011minimax,besbes2014dynamic} upon which the objective function changes abruptly.
Recent papers have been focusing on specific structures of changes, including an additive term of time-varying price-independent components \citep{den2015tracking}, a privacy pricing setting \citep{xu2016dynamic}, varying parameters of linear demand \citep{keskin2017chasing}, a dynamic inventory system \citet{zhang2018perishable}, smooth or discontinuous linear changes \citep{chen2019dynamic},
growing market demand \citep{zhu2020demands}
and changing preferences for quality \citep{keskin2020selling}.
A comprehensive review
%of papers
studying online learning in revenue management can be found in \citet{den2015dynamic}.
The rapid growth of literature reflects the importance of online learning in the business world, particularly in a non-stationary market environment.
The periodic pattern in this paper has not been studied before, but its application in dynamic pricing is highly relevant.
In fact, one of the motivating examples of this study is the seasonal demand patterns that are ubiquitous in retailing.

This paper is also related to the classic topic in statistics and signal processing, specifically how to estimate frequencies from a noisy signal.
For example, \citet{babtlett1948smoothing} suggests aggregating a few segments of the signal to reduce variance.
\citet{bartlett1963spectral,vere1982estimation,chen2019super} study the same problem for arrival data generated from point processes.
See standard textbooks such as \citet{stoica2005spectral} for a summary of the vast literature in this area.
In our problem, the observations are independent, non-stationary, and periodic.
There are many papers devoted to the asymptotic properties of the periodogram where the number of observations tends to infinity, such as \citet{olshen1967asymptotic,brillinger1969,shao2007asymptotic,shao2011modelling}.
However, to the best of our knowledge, no finite-sample analysis is available for the probability of correctly estimating the periods in our framework, which is essential for the regret analysis of online learning.
Some papers with finite-sample analysis either focus on stationary time series \citep{thomson1982spectrum} or point processes \citep{chen2019super}, which do not apply to our case.
In this paper, we develop a frequency identification algorithm, whose theoretical guarantee explicitly depends on the sample size and other parameters of the learning problem.

\smallskip
\noindent
\textbf{Paper Outline.} The remainder of this paper is organized as follows.
In Section \ref{sec:formulation}, we describe the formulation of the periodic MAB problem.
In Section \ref{sec:policy}, we propose our two-stage learning policy.
In Section \ref{sec:analysis}, we show that a regret upper bound $\tilde{O}(\sqrt{T\sum_{k=1}^K T_k})$ is achieved by our policy.
In Section \ref{sec:lower-bound}, we establish a general regret lower bound $\Omega(\sqrt{T\max_{k}\{ T_k\}})$ for any policy.
In Section \ref{sec:conclusion}, we present some concluding remarks.
The detailed proofs can be found in the appendices, and we also provide a study on the optimal regret for special cases where all arms have the same period.

\section{Problem Formulation}
\label{sec:formulation}
We consider an MAB problem over a finite-time horizon. Let $\sett = \{1,...,T\}$ denote the sequence of decision epochs, and let $\setk = \{1,...,K\}$ denote the set of arms (possible actions). At each epoch, the DM pulls one of the $K$ arms.
If arm $k \in \setk$ is chosen at epoch $t \in \sett$, the DM receives a random reward $Y_{k,t}$. We assume that the reward is specified as $Y_{k,t} \coloneqq \mu_{k,t} + \epsilon_{t}$ where the mean $\mu_{k,t} = \E[Y_{k,t}]$ is time-varying and the noise $\epsilon_{t}$ is an independent mean-zero random variable.

The DM's objective is to maximize the cumulative expected rewards over the horizon of $\sett$, but she has no information about any $\mu_{k,t}$ for all $k \in \setk$ and $t \in \sett$ initially.
Therefore, the DM needs to acquire the information of $\mu_{k,t}$ (exploration) and optimize immediate rewards by pulling the best arm $\argmax_{k}\{\mu_{k,t}\}$ at each epoch as often as possible (exploitation). It is well understood that this objective is not achievable when $\mu_{k,t}$ changes arbitrarily in $t$ since the knowledge learned in the past cannot be used to predict the future.
We study the case which assumes that the expected reward of each arm repeats its values after a period,
i.e., $\mu_{k,t+T_k} = \mu_{k,t}$ for all $k \in \setk$ where $T_k \in \N^{+}$ denotes the (minimum) period of arm $k$.
We also impose the following technical assumptions on the mean reward and the random noise, which are common in the MAB literature.
\begin{assumption}
\label{asp:mean-reward}
For all $k \in \setk$ and $t \in \sett$, the mean reward $\mu_{k,t} \in [0,1]$.
\end{assumption}
\begin{assumption}
\label{asp:sub-Gaussian}
The noise $\epsilon_t$ for $t \in \sett$ are independent sub-Gaussian random variables with parameter $\sigma$.
That is, $\E[\exp(\lambda \epsilon_t)] \leq \exp\left(\frac{1}{2} \sigma^2 \lambda^2\right)$ for all $\lambda\in \R$ and $ \PR(|\epsilon_t|>x)\le 2\exp(-\frac{x^2}{2\sigma^2})$ for all $x>0$.
\end{assumption}

The DM knows the values of $K$, $T$ and $\sigma$ and the fact that $\mu_{k,t}$ changes periodically, but she is not aware of the value of $\mu_{k,t}$ or $T_{k}$ for any arm $k \in \setk$ initially.
Let $\pi_{t} \in \setk$ denote the arm pulled by the DM at epoch $t$. With a little abuse of notation, we let $\pi\coloneqq\{\pi_{t}: t\in \sett\}$ denote an admissible policy which takes the action $\pi_{t}$ at epoch $t$ depending on the historical rewards observed and actions taken, i.e., $\{\pi_{1},Y_{\pi_{1},1},...,\pi_{t-1}, Y_{\pi_{t-1},t-1}\}$.
In the MAB literature, a policy $\pi$ is usually evaluated in terms of regret: the gap between the performance of pulling at each $t$ the arm which has the highest expected reward (optimal decisions made with full information) and the expected performance under the policy $\pi$.
That is, we define the pseudo-regret $R_{T}^{\pi} \coloneqq \sum_{t=1}^T \left(\max_{k \in \setk}\mu_{k,t}-\mu_{\pi_t,t}\right)$, and then the expected regret is given as
\begin{equation} \label{eq:regret-def}
\E[R_{T}^{\pi}] = \sum_{t=1}^T \left(\max_{k \in \setk}\mu_{k,t}-\E\left[\mu_{\pi_t,t}\right]\right),
\end{equation}
where the expectation $\E$ is taken with respect to the policy $\pi$ which is contingent on the (stochastic) history.
In the following sections, we propose a policy that helps the DM to learn and optimize the rewards, and we analyze the corresponding expected regret.

\section{The Proposed Two-Stage Learning Policy}
\label{sec:policy}
To learn the periodic pattern and the values of the expected rewards of each arm, our policy consists of two stages in sequence.
In stage one, we develop Algorithm \ref{alg:frequency-identification} based on spectral analysis to estimate the lengths of periods of all arms.
In stage two, we propose a confidence-bound-based learning Algorithm \ref{alg:learning} to further explore and exploit arm rewards simultaneously.

\subsection{Stage One: Period Estimation}
\label{sec:phaseone}
We adapt techniques in spectral analysis to identify the frequency components of the observed reward sequence for each arm, and thus estimate the corresponding period.
To motivate our frequency identification algorithm, we first briefly review some background in Section \ref{sec:phaseone-dft} (more related knowledge may be referred to \citealt{brigham1988fast}), and then describe the details of the algorithm in Section \ref{sec:phaseone-algorithm}.

\subsubsection{Discrete Fourier Transform and Periodogram}
\label{sec:phaseone-dft}
Fourier analysis implies that a periodic function $\mu_{k,t}$ can be represented as a sum of sinusoids. In this paper it is more convenient to work with the complex representation:
\begin{equation}
\label{eq:exp-rep}
\mu_{k,t} = \sum_{j=0}^{T_k-1} b_{k,j} \exp\left( 2\pi i\frac{j }{T_k}t\right),
\end{equation}
where $i=\sqrt{-1}$, $b_{0,k}\in \R$, and $(b_{k,j}, b_{k,T_k-j})$ is a pair of complex conjugates $b_{k,j}=\overline{b_{k,T_k-j}}$ for $1\le j < T_k$.
The decomposition of \eqref{eq:exp-rep} contains the components of the fundamental frequency $1/T_{k}$, the harmonics $j/T_{k}$ for $j=2,...,T_k-1$ and the constant term with $j=0$.
Note that $j \ge T_{k}$ are not needed because of the discrete sampling. For example, a frequency component $(T_k+1)/T_{k}$ is indistinguishable from $1/T_k$ since $\exp(2\pi i (T_k+1) t/T_k)=\exp(2\pi i  t/T_k)$ for $t\in \N^{+}$, which is referred to as ``aliasing'' in the language of signal processing.

Suppose that the DM observed a sequence of $n$ rewards $\{Y_{k,1},...,Y_{k,n}\}$ from arm $k$. We apply the discrete Fourier transform (DFT) to conduct analysis in the frequency domain. The DFT of the reward sequence $\tilde{y}_k(v)$ is a function which maps a frequency $v\in[0,1]$ to a complex value:
\begin{equation} \label{eq:dft-def}
\tilde{y}_k(v) \coloneqq \frac{1}{n} \sum_{t=1}^n Y_{k,t} \exp(-2\pi iv t).
\end{equation}
Note that we only need to consider the domain $v \in [0,1]$ because all frequency components $j/T_k\in [0,1]$.
Recall that $Y_{k,t} = \mu_{k,t} + \epsilon_{t}$ and thus we can decompose $\tilde{y}_k(v)$ as
\begin{equation}\label{eq:dft-decomp}
\tilde{y}_k(v) = \underbrace{\frac{1}{n} \sum_{t=1}^n \mu_{k,t} \exp(-2\pi ivt)}_{\tilde{\mu}_{k}(v)} + \underbrace{\frac{1}{n} \sum_{t=1}^n \epsilon_t \exp(-2\pi ivt)}_{\tilde{\epsilon}_{k}(v)}
\end{equation}
where $\tilde{\mu}_k(v)$ and $\tilde{\epsilon}_k(v)$ denote the DFT of the mean reward and the noise respectively.

To identify frequency components $j/T_k$, we inspect the \emph{periodogram} which estimates the spectral density by plotting the \emph{modulus} of a DFT against the frequency.
Since the noise $\epsilon_{t}$ is random, its DFT $\tilde{\epsilon}_k(v)$ is not expected to show any pattern in the frequency domain.
To illustrate, the periodogram of a possible realization of $\tilde{\epsilon}_k(v)$ is shown in the left panel of Figure~\ref{fig:DFT}.
On the other hand, using expression \eqref{eq:exp-rep}, the DFT of the mean reward can be rewritten as
\begin{align}
\label{eq:dft-deterministic}
\tilde{\mu}_k(v) &= \frac{1}{n} \sum_{t=1}^n \sum_{j=0}^{T_k-1} b_{k,j}\exp\left(2\pi i \left(\frac{j}{T_k}-v\right)t\right)= \sum_{j=0}^{T_k-1} \underbrace{\frac{1}{n} \sum_{t=1}^n b_{k,j}\exp\left(2\pi i \left(\frac{j}{T_k}-v\right)t\right)}_{\tilde{\mu}_{k,j}(v)},
\end{align}
where $\tilde{\mu}_{k,j}(v)$ denotes the DFT associated with the frequency component $j/T_{k}$.
As shown in Appendix \ref{sec:appendix-proof-phaseone}, we have
\begin{equation}\label{eq:dft-sinform}
\tilde{\mu}_{k,j}(v) = \frac{b_{k,j}}{n} \exp\left(2\pi i \left(\frac{j}{T_k}-v\right) \frac{n+1}{2}\right) \frac{\sin\left(\pi (j/T_k-v)n\right)}{\sin\left(\pi(j/T_k-v)\right)}.
\end{equation}
If $v = j/T_k$ and $|b_{k,j}|>0$, we have $|\tilde{\mu}_{k,j}(j/T_k)| = |b_{k,j}|$; otherwise if $v\neq j/T_k$, we have $\displaystyle{\lim_{n\to\infty} |\tilde{\mu}_{k,j}(v)| = 0}$. Therefore, when the sample size $n$ goes to infinity, we expect to see a ``spike'' appearing at $v=j/T_k$ in the periodogram of $\tilde{\mu}_{k,j}(v)$ which is illustrated in the middle panel of Figure~\ref{fig:DFT}.
However, since the sample size is limited in practice, $|\tilde{\mu}_{k,j}(v)|$ in general is not zero at $v\neq j/T_k$ as illustrated in the right panel of Figure~\ref{fig:DFT}.
This phenomenon of non-zero periodogram at $v \neq j/T_k$ due to finite sample size is referred to as \emph{spectral leakage}. We also note that the \emph{main lobe} surrounding the frequency component $v=j/T_k$ is of width $2/n$ and the each \emph{side lobe} is of width $ 1/n$.
\begin{figure}[t]
\centering
\begin{tikzpicture} [baseline]
    \begin{axis}[width=0.375\textwidth, height=0.375\textwidth,
    xmin=0, xmax=1, ymin=0, ymax=5,
    xtick={0,0.5,1}, xticklabels={0, 1/2, 1}, ytick={10}, yticklabels={},
    xlabel=$v$, x label style={at={(axis description cs:0.5,-0.12)},anchor=north},
    ylabel=$|\tilde{\epsilon}_k(v)|$, y label style={at={(axis description cs:-0.15,0.9)},rotate=0, anchor=north}],
    \addplot+[mark=none] table [x=freq, y=dft_noise, col sep=comma] {./data.csv};
    \end{axis}
\end{tikzpicture}
\begin{tikzpicture}[baseline]
    \begin{axis}[width=0.375\textwidth, height=0.375\textwidth,
    xmin=0, xmax=1, ymin=0, ymax=1.6,
    xtick={0.4}, xticklabels={$\frac{j}{T_k}$}, xticklabel style={rotate=0,anchor=near xticklabel}, ytick={1}, yticklabels={},
    xlabel=$v$, x label style={at={(axis description cs:0.5,-0.12)},anchor=north},
    ylabel=$|\tilde{\mu}_{k,j}(v)|$, y label style={at={(axis description cs:-0.15,.87)},rotate=0, anchor=north}],
    \addplot+[ycomb, mark=none] coordinates {(0.4,1)};
    \draw[] ([yshift = 3cm] axis description cs:0.12, 0) --  ([yshift = 3cm] axis description cs:0.12, 0) node [midway, rotate=0, fill=white, yshift=-1pt, inner sep=0.2ex] {$b_{k,j}$};
    \end{axis}
\end{tikzpicture}
\begin{tikzpicture}[baseline]
    \begin{axis}[width=0.375\textwidth,height=0.375\textwidth,
    xmin=0, xmax=1, ymin=0, ymax=1.6,
    xtick={0.4}, xticklabels={}, xticklabel style= {rotate=0,anchor=near xticklabel}, ytick={1},yticklabels={},
    xlabel=$v$, x label style={at={(axis description cs:0.5,-0.12)},anchor=north},
    ylabel=$|\tilde{\mu}_{k,j}(v)|$, y label style={at={(axis description cs:-0.15,.87)},rotate=0, anchor=north}, clip=false]
    \addplot+[mark=none] table[x=freq,y=dft_finite, col sep=comma] {./data.csv};
    \draw[|<->|] ([yshift = -0.3cm] axis description cs:0.1, 0) --  ([yshift = -0.3cm] axis description cs:0.2, 0) node [midway, rotate=0, fill=white, yshift=-1pt, inner sep=0.2ex] {$\frac{1}{n}$};
    \draw[|<->|] ([yshift = -0.3cm] axis description cs:0.6, 0) --  ([yshift = -0.3cm] axis description cs:0.7, 0) node [midway, rotate=0, fill=white, yshift=-1pt, inner sep=0.2ex] {$\frac{1}{n}$};
    \draw[|<->|] ([yshift = -0.3cm] axis description cs:0.3, 0) --  ([yshift = -0.3cm] axis description cs:0.5, 0) node [midway, rotate=0, fill=white, yshift=-1pt, inner sep=0.2ex] {$\frac{2}{n}$};
    \draw[] ([yshift = 0.5cm] axis description cs:0.4, 0) --  ([yshift = 0.5cm] axis description cs:0.4, 0) node [midway, rotate=0, fill=white, yshift=-1pt, inner sep=0.2ex] {$\frac{j}{T_k}$};
    \draw[] ([yshift = 3cm] axis description cs:0.12, 0) --  ([yshift = 3cm] axis description cs:0.12, 0) node [midway, rotate=0, fill=white, yshift=-1pt, inner sep=0.2ex] {$b_{k,j}$};
    \end{axis}
\end{tikzpicture}
\caption{The periodograms of the noise term $|\tilde{\epsilon}_k(v)|$ (left panel),
$|\tilde{\mu}_{k,j}(v)|$ for sample size $n\to\infty$ (middle panel)
and $|\tilde{\mu}_{k,j}(v)|$ with finite samples (right panel).}
\label{fig:DFT}
\end{figure}
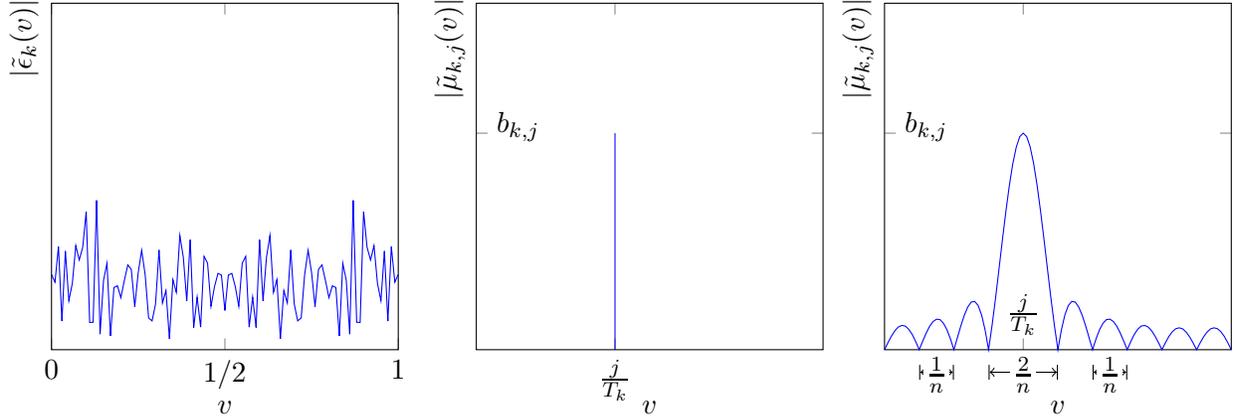

In order to estimate period $T_k$, we investigate the periodogram of $\tilde{y}_k(v)$ and expect to identify frequency components $j/T_k$ with $|b_{k,j}|>0$, which are referred to as \emph{present frequencies} in the following discussion. When $v$ is a present frequency, $|\tilde{y}_k(v)|$ is the aggregation of the spike at the main lobe of $\tilde{\mu}_{k,j}(v)$, the leakage from side lobes of $\tilde{\mu}_{k,j'}(v)$ for $j' \neq j$, and the noise $\tilde{\epsilon}_{k}(v)$. On the other hand, if $v$ is far apart from any $j/T_k$, $|\tilde{y}_k(v)|$ is the aggregation of the leakage and the noise.
Therefore, if a proper \emph{threshold} can be established to be both a lower bound of the spikes and an upper bound of the sum of  leakage and noise, then it will help screen out the spikes of main lobes from the floor of leakage and noise, according to the differences in their scales.
As a result, we are able to identify present frequencies and then to estimate $T_k$.
This is the main idea behind our threshold-based algorithm of frequency identification, with details  provided in the following Section \ref{sec:phaseone-algorithm}.
We also note that $|\tilde{y}_{k}(v)|=|\tilde{y}_{k}(1-v)|$ from the definition \eqref{eq:dft-def}, i.e., the periodogram is symmetric with respect to $v=1/2$ in the frequency domain $v \in [0,1]$. Hence, we only need to inspect the periodogram for the half domain of $v\in [0,1/2]$.

\subsubsection{The Frequency Identification Algorithm}
\label{sec:phaseone-algorithm}
We explain the intuition of Algorithm \ref{alg:frequency-identification}, in particular how to choose a threshold and then how to apply an adaptive neighborhood approach to frequency identification.
In the beginning (Step \ref{step:begin} - \ref{step:exploration end}), we conduct an exploration of $nK$ epochs where  each of the $K$ arms is pulled for $n$ times consecutively. The period of each arm is investigated individually. Given the reward sample sequence of arm $k$, we generate its periodogram in Step \ref{step:compute periodogram}, and initialize the set of candidate present frequencies $\setf$ by considering all possible integer values of $T_k$ in Step \ref{step:define candidate set}.

\begin{algorithm}[]
    %\SingleSpacedXI
    %\OneAndAHalfSpacedXI
    %\singlespacing
    \caption{stage one: period estimation}
    \label{alg:frequency-identification}
    \begin{algorithmic}[1]
        \State Input: $T$, $K$ and $\sigma$ \label{step:begin}
        \State Choose parameters: $n$ (length of exploration for each arm), $g \ge \max\{2, \sqrt{n}\}$ (the width of the neighborhood to be excluded is $\frac{2g}{n}$), and  $H>0$ (a constant in the threshold)\label{step:para-choice}
        \For{$t=1:nK$} \Comment{Explore each arm sequentially}\label{step:arm-sequential}
        \State $k\gets \lfloor \frac{t-1}{n}+1 \rfloor$
        \State Pull arm $k$, observe the reward $Y_{k,t}=\mu_{k,t}+\epsilon_t$
        \EndFor \label{step:exploration end}

        \For{$k=1:K$} \Comment{Estimate period $T_k$ for arm $k$ }\label{step:start-periodogram}

        \State \label{step:compute periodogram}
        Compute the periodogram $\abs{\tilde{y}_k(v)} = \abs{\frac{1}{n}\sum_{s=n(k-1)+1}^{nk}Y_{k,s}\exp(-2\pi i v s)}$ for $v \in [0,1/2]$

        \State \label{step:define candidate set}
        %Initialize the set of candidate frequencies $\setf \gets \left\{\frac{j_1}{j_2}: j_1,j_2 \in \N^+, ~ 1\leq j_1 < j_2< \min \left\{\sqrt{\frac{n}{2}}, \frac{n}{2g }\right\} \right\}$ \hspace{0cm}
        Initialize the set of candidate frequencies $\setf \gets \left\{\frac{j_1}{j_2}: j_1,j_2 \in \N^+, ~ 1\leq j_1 < j_2< \frac{n}{2g} \right\}$ \hspace{0cm}

        \State Compute the threshold $\tau_k$: \label{step:threshold}
        \begin{align*}
            A_j & \gets \sup \left\{ \frac{|\sin(\pi \nu)|}{\pi \nu}: \nu \in [j,j+1] \right\},\ j=1,2,\dots\\
            U_1 &\gets \sum_{j=0}^{\lfloor{\frac{n-2g -1}{4g }\rfloor}} A_{(2j+1)g}, ~\ U_2 \gets \sum_{j=1}^{\lfloor{\frac{n-1}{4g }\rfloor}} A_{2jg-1}\\
            \bar\epsilon_v&\gets\frac{2\sigma H}{1-\pi/24}\sqrt{\frac{\log(n)}{n}} \\
            \tau_k &\gets  \bar\epsilon_v + \frac{\pi U_1}{1-\pi U_2}\left(\bar\epsilon_v + \PB \right)
        \end{align*}

        \State\label{step:exclude0} Initialize the frequency domain of interest $\setd \gets \left\{v: \frac{g }{n} \leq v \leq \frac12, ~ \abs{\tilde{y}_k(v)}>\tau_{k} \right\}$  \hfill \Comment{The neighborhood of $v=0$ is excluded}

        \State $i\gets 0$  \label{step:search begin}
        \While{$\setd$ is not empty}
        \State $i\gets i+1$
        \State Find a global maximum of the periodogram in $\setd$ as $v^\ast_i = \argmax_{v \in \setd}|\tilde{y}_k(v)|$  \label{step:global-maximum}
        \State Find the frequency in $\setf$ closest to $v^\ast_i$ such that $\hat{v}_i = \argmin_{v \in \setf}\abs{v - v_i^\ast}$\label{step:pick-frequency}
        \State Exclude the neighborhood of $\hat{v}_i$ and update $\setd$: $\setd \gets \setd \setminus \left(\hat{v}_i-\frac{g }{n},\hat{v}_i+\frac{g }{n}\right)$ \label{step:exclude-neighborhood}
        \EndWhile \label{step:search end}
        \State Return the estimated period for arm $k$: $\hat T_k=\textsc{LCM}(\hat{v}_1^{-1},\hat{v}_2^{-1},...)$ %\Comment{The least common multiple based on the estimated frequencies}
        \label{step:LCM}
        \EndFor
        \label{step:end-periodogram}
    \end{algorithmic}
\end{algorithm}

\textbf{Choosing a Threshold.}
Now we proceed to Step \ref{step:threshold}, the key step of Algorithm~\ref{alg:frequency-identification} which determines the threshold.
Ideally, the periodogram consists of spikes located at each present frequency.
However, this is not exactly the case due to the noise and spectral leakage discussed in Section \ref{sec:phaseone-dft}.
To this end, we look for a threshold $\tau_{k}$ to filter out the noise and the leakage.
Recall that the magnitude of $\tilde{y}_{k}(v)\sim b_{k,j}$ at $v=j/T_k$ as $n\to\infty$, which is unknown a priori, so $\tau_{k}$ needs to be determined through a data-driven approach using observed rewards, i.e., calculated after inspecting the periodogram.
Otherwise a pre-specified threshold may leave out present frequencies if set too large, or include spurious frequencies if set too small.
To derive a proper $\tau_k$, we develop the following results in Section \ref{sec:analysis-phaseone}:
Lemma \ref{lem:noise-bound} establishes an upper bound $\bar\epsilon_v$ on the noise;
Lemma \ref{lem:bound-leakage} provides upper bounds $U_1$ and $U_2$ on the leakage;
and Lemma \ref{lem:Bound-B-Above} generates a data-driven upper bound for $|b_{k,j}|$.
Putting these pieces together, Lemma \ref{lem:threshold} guides us to the choice of $\tau_k$ that is large enough to exclude the periodogram not close to present frequencies.
%Lemma \ref{lem:noise-bound}, \ref{lem:bound-leakage} and \ref{lem:Bound-B-Above}
We also expect that the periodogram of present frequencies $|\tilde{y}_k(j/T_k)|$ can emerge above $\tau_k$, and so we develop Lemma \ref{lem:no-false-negative} to obtain a data-driven lower bound on $|\tilde{y}_k(j/T_k)|$. Based on Lemma \ref{lem:noise-bound} - \ref{lem:no-false-negative}, we derive a suitable value for $\tau_k$ which ensures that local maxima near each present frequency will be selected.

\textbf{Neighborhood Approach.} In step \ref{step:exclude0}, we are interested in the sub-domains of $\setd$ where the periodogram is above $\tau_k$. One present frequency is not necessarily the local maximum in the periodogram due to leakage and the noise, and more importantly, the threshold may select local maxima created by side lobes near that present frequency. Therefore, we can not simply treat all the local maximum above $\tau_k$ as estimates of present frequencies.
To remedy this issue, we develop Lemma \ref{lem:local-noise-bound} and \ref{lem:narrow-down-neighborhood} in Section \ref{sec:analysis-phaseone} to guarantee that, under certain technical conditions, the present frequency can be recovered through matching the largest local maximum to the nearest candidate frequency in $\setf$. Then, we remove a neighborhood of width $2g/n$ from the selected present frequency (recall that the width of a side lobe is $1/n$).
If the parameter $g$ is well chosen, only one possible present frequency is located inside the neighborhood, and side lobes outside the neighborhood decay sufficiently so that they do not emerge above the threshold. We also note that the constant term $b_{0,k}$ might be large relative to the magnitude of cyclic components $|b_{j,k}|$ in many applications, and thus its leakage  may distort the present frequencies near $v=0$. Hence, we exclude the neighborhood of the end point $v=0$ in Step \ref{step:exclude0}.

From Step \ref{step:search begin} to \ref{step:search end}, the following procedure is repeated: searching for the global maximum, matching it to the corresponding present frequency, and removing the neighborhood adaptively.
The procedure terminates when the periodogram for the remaining frequencies is completely below the threshold.
Eventually we obtain an estimate of period $\hat T_k$ using the least common multiple (LCM) of the reciprocals of identified present frequencies in Step \ref{step:LCM}.

We assume that $T$ is large enough such that sufficient exploration can be conducted for each of these $K$ arms.
Algorithm~\ref{alg:frequency-identification} is also conditional on the requirement that $T_k$ cannot be too large relative to $n$. Otherwise the present frequencies might be too close to each other (the distance between two present frequencies in the periodogram can be as small as $1/T_k$) and the neighborhood approach may exclude other present frequencies. Note that these are fundamental requirements that are independent of the frequency identification approach. Therefore, we propose the following assumption, which guarantees $\hat{T_k}=T_k$ can be successfully estimated with high probability.
\begin{assumption}
\label{asp:period}
%Assume that $T > 4K$ and for all $k \in \setk$, the period satisfies $\displaystyle{T_k < \min\left\{\sqrt{\frac{n}{2}},~\frac{n}{2g}\right\}}$ where $g$ is an integer parameter satisfying $g \geq \max\{2, \sqrt{n}\}$.
Assume that $T > 4K$ and for all $k \in \setk$, the period satisfies $\displaystyle{T_k < \frac{n}{2g}}$ where $g$ is an integer parameter satisfying $g \geq \max\{2, \sqrt{n}\}$.
\end{assumption}

Recommended choices for parameters that satisfy Assumption~\ref{asp:period} are
\begin{align}
\label{eq:parameter}
n= \lfloor \sqrt{T/K}\rfloor,
%~g = \lceil \max \{ 2,\sqrt{n} \} \rceil
~g = \lceil \sqrt{n}~ \rceil
~ \text{and} ~ H=\sqrt{1 +\log(n)}.
\end{align}
We aim at a regret rate $O(\sqrt{TK})$ for stage one.
Specifically, we choose the length of stage one as $\sqrt{TK}$ by pulling each arm $n \sim \sqrt{T/K}$ times.
To estimate $T_k$'s accurately,
we need $T_k^2 \sim n$ due to the resolution of Algorithm \ref{alg:frequency-identification}, and hence $T_k \sim (T/K)^{1/4}$ is required for our approach.
We choose parameters in \eqref{eq:parameter} to impose $\displaystyle{T_k < \frac{T^{1/4}}{2K^{1/4}}}$ for all $k \in \setk$.
The subtlety of choosing these values is further discussed at the end of Appendix \ref{sec:appendix-proof-analysis-phaseone}.

%%%% comment out the following example for saving complie time
%\comment{
To further explain Algorithm~\ref{alg:frequency-identification}, we demonstrate it with an example.
Since frequency identification is conducted for each arm independently, we focus on a representative case of an individual arm $k$. Suppose that the mean reward is $\mu_{k,t}=3+3\sin\left(\frac{1}{2}\pi t\right)+3\cos\left(\pi t\right)$ of period $T_k=4$, and thus the present frequencies  are $v_1=1/2$ and $v_2=1/4$. We also assume that the noise $\epsilon_t$ is normally distributed with mean 0 and standard deviation $\sigma=0.2$. We generate $n=50$ sample rewards, and the periodogram of these observations is shown in Figure~\ref{fig:alg1-example}a.
Using parameters $H=\sqrt{1+\log(n)}=2.21$ and $g =\lceil \sqrt{n} \rceil=8$, a data-driven threshold $\tau_k=0.885$ is computed according to Step~\ref{step:threshold}. We exclude the neighborhood of $v=0$.
In Figure~\ref{fig:alg1-example}b, we find the first maximum above the threshold, match it to the present frequency $\hat v_1 = 1/2$, and remove the neighbor around $\hat v_1$.
In Figure~\ref{fig:alg1-example}c, we repeat the above process to identify $\hat v_2=1/4$.
In Figure~\ref{fig:alg1-example}d,  when no local maximum above the threshold is left, the algorithm terminates with  $(\hat v_1=1/2, \hat v_2=1/4)$ and obtains an estimated period $\hat T_k=\textsc{LCM}(2, 4)=4$.
By the analysis in Section~\ref{sec:analysis-phaseone}, Algorithm~\ref{alg:frequency-identification} ensures that $\hat T_k=T_k$ is correctly estimated with a probability at least $0.983$ in this example.
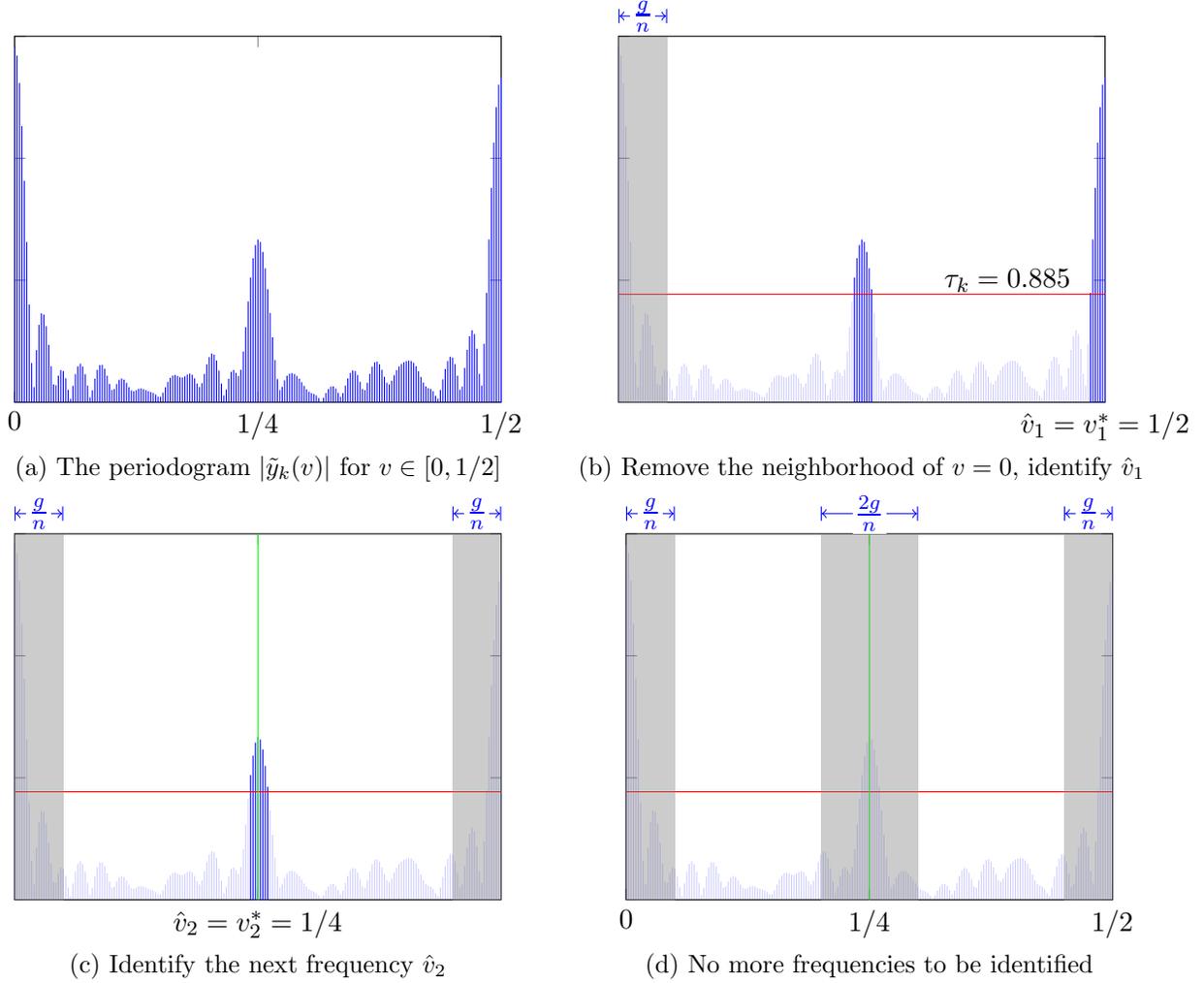
\begin{figure}[t]
\centering
\begin{tikzpicture} [baseline]
    \begin{axis}[width=0.5\textwidth, height=0.4\textwidth,
    xmin=0, xmax=0.5, ymin=0, ymax=3,
    xtick={0,0.25,0.5}, xticklabels={0, 1/4, 1/2},
    ytick={}, yticklabels={},
    xlabel={\small (a) The periodogram $|\tilde y_k(v)|$ for $v\in [0,1/2]$}, x label style={at={(axis description cs:0.5,-0.12)},anchor=north}],
    \addplot+[ycomb, mark=none] table[x=freq,y=dft, col sep=comma] {./data2.csv};
    \end{axis}
\end{tikzpicture} \hfill
\begin{tikzpicture} [baseline]
    \begin{axis}[width=0.5\textwidth, height=0.4\textwidth,
    xmin=0, xmax=0.5, ymin=0, ymax=3,
    xtick={0.5}, xticklabels={$\hat{v}_1={v}^*_1=1/2$}, xticklabel style= {rotate=0,anchor=near xticklabel},
    ytick={}, yticklabels={},
    xlabel={\small (b) Remove the neighborhood of $v=0$, identify $\hat v_1$}, x label style={at={(axis description cs:0.5,-0.12)},anchor=north},
    clip = false]
    \addplot+[blue, ycomb, mark=none]
        table[x=freq,
        y expr={\thisrow{dft} >= 0.8854 && \thisrow{freq} >= 0.05 ? \thisrow{dft} : NaN},
        col sep=comma] {./data2.csv};
    \addplot+[blue, ycomb, mark=none, opacity=0.2]
        table[x=freq,
        y expr={\thisrow{dft} <0.8854 || \thisrow{freq} <= 0.05? \thisrow{dft} : NaN},
        col sep=comma] {./data2.csv};
    \addplot [red, domain=0:0.5,] {0.8854};
    \node[] at (axis cs: 0.4, 1) {$\tau_k=0.885$};
    \fill [gray,opacity=0.4] (axis cs:0,0) rectangle (axis cs:0.05,3);
    \draw[blue, |<->|] ([yshift = +5.3cm] axis description cs:0, 0) --  ([yshift = +5.3cm] axis description cs:0.1, 0) node [midway, rotate=0, fill=white, yshift=0pt,inner sep=0.2ex] {$\frac{g}{n}$};
    \end{axis}
\end{tikzpicture}

\begin{tikzpicture}[baseline]
    \begin{axis}[width=0.5\textwidth, height=0.4\textwidth,
    xmin=0, xmax=0.5, ymin=0, ymax=3,
    xtick={0.25}, xticklabels={$\hat{v}_2=v^\ast_2=1/4$}, xticklabel style= {rotate=0,anchor=near xticklabel},
    ytick={}, yticklabels={},
    xlabel={\small (c) Identify the next frequency $\hat v_2$}, x label style={at={(axis description cs:0.5,-0.12)},anchor=north},
    clip = false]
    \addplot+[blue, ycomb, mark=none]
        table[x=freq,
        y expr={\thisrow{dft} >= 0.8854 && \thisrow{freq} >= 0.05 && \thisrow{freq} <= 0.45? \thisrow{dft} : NaN},
        col sep=comma] {./data2.csv};
    \addplot+[blue, ycomb, mark=none, opacity=0.2]
        table[x=freq,
        y expr={\thisrow{dft} <0.8854 || \thisrow{freq} <= 0.05 || \thisrow{freq} >= 0.45? \thisrow{dft} : NaN},
        col sep=comma] {./data2.csv};
    \addplot [red, domain=0:0.5,] {0.8854};
    %              \node[] at (axis cs: 0.42, 1) {$\tau=0.885$};
    \addplot+[green, ycomb, mark=none] coordinates
        {(0.25, 3)};
    \fill [gray,opacity=0.4] (axis cs:0,0) rectangle (axis cs:0.05,3);
    \draw[blue, |<->|] ([yshift = +5.3cm] axis description cs:0, 0) --  ([yshift = +5.3cm] axis description cs:0.1, 0) node [midway, rotate=0, fill=white, yshift=0pt,inner sep=0.2ex] {$\frac{g}{n}$};
    \fill [gray,opacity=0.4] (axis cs:0.45,0) rectangle (axis cs:0.5,3);
    \draw[blue, |<->|] ([yshift = +5.3cm] axis description cs:0.9, 0) --  ([yshift = +5.3cm] axis description cs:1, 0) node [midway, rotate=0, fill=white, yshift=0pt,inner sep=0.2ex] {$\frac{g}{n}$};
    \end{axis}
\end{tikzpicture} \hfill
\begin{tikzpicture}[baseline]
    \begin{axis}[width=0.5\textwidth, height=0.4\textwidth,
    xmin=0, xmax=0.5, ymin=0, ymax=3,
    xtick={0,0.25,0.5}, xticklabels={0, 1/4, $\quad\quad1/2\quad\quad$},
    ytick={}, yticklabels={},
    xlabel={\small (d) No more frequencies to be identified}, x label style={at={(axis description cs:0.5,-0.12)},anchor=north},
    clip = false]
    \addplot+[blue, ycomb, mark=none]
        table[x=freq,
        y expr={\thisrow{dft} >= 0.8854 && \thisrow{freq} >= 0.05 && \thisrow{freq} <= 0.45 && 0.2 <= \thisrow{freq} <= 0.3? \thisrow{dft} : NaN},
        col sep=comma] {./data2.csv};
    \addplot+[blue, ycomb, mark=none, opacity=0.2]
        table[x=freq,
        y expr={\thisrow{dft} <0.8854 || \thisrow{freq} <= 0.05 || \thisrow{freq} >= 0.45 || !(0.2 <= \thisrow{freq} <= 0.3)? \thisrow{dft} : NaN},
        col sep=comma] {./data2.csv};
    \addplot [red, domain=0:0.5,] {0.8854};
    %              \node[] at (axis cs: 0.42, 1) {$\tau=0.885$};
    \addplot+[green, ycomb, mark=none] coordinates
        {(0.25, 3)};
    \fill [gray,opacity=0.4] (axis cs:0,0) rectangle (axis cs:0.05,3);
    \draw[blue, |<->|] ([yshift = +5.3cm] axis description cs:0, 0) --  ([yshift = +5.3cm] axis description cs:0.1, 0) node [midway, rotate=0, fill=white, yshift=0pt,inner sep=0.2ex] {$\frac{g}{n}$};
    \fill [gray,opacity=0.4] (axis cs:0.45,0) rectangle (axis cs:0.5,3);
    \draw[blue, |<->|] ([yshift = +5.3cm] axis description cs:0.9, 0) --  ([yshift = +5.3cm] axis description cs:1, 0) node [midway, rotate=0, fill=white, yshift=0pt,inner sep=0.2ex] {$\frac{g}{n}$};
    \fill [gray,opacity=0.4] (axis cs:0.20,0) rectangle (axis cs:0.30,3);
    \draw[blue, |<->|] ([yshift = +5.3cm] axis description cs:0.4, 0) --  ([yshift = +5.3cm] axis description cs:0.6, 0) node [midway, rotate=0, fill=white, yshift=0pt,inner sep=0.2ex] {$\frac{2g}{n}$};
    \end{axis}
\end{tikzpicture}

\caption{A demonstration of Algorithm~\ref{alg:frequency-identification}.}
\label{fig:alg1-example}
\end{figure}
%}

\subsection{Stage Two: Nested Confidence-Bound-Based Learning}
\label{sec:phasetwo}
We obtain an estimation $\hat{T}_k$ on the length of period for each arm $k \in \setk$, by spending $nK$ epochs totally in stage one.
In the remaining horizon, which is referred to as stage two, we need to learn the specific values of mean rewards $\mu_{k,t}$ while simultaneously optimizing immediate rewards.
Classic algorithms such as UCB are not immediately applicable due to the asynchronous periods.
As mentioned earlier, a naive approach would consider decomposing stage two into a number of $T_{\text{LCM}} = \text{LCM}(\hat{T}_1,...,\hat{T}_K)$ separate MAB subproblems. However, it works poorly in practice since the number of the subproblems, $T_{\text{LCM}}$, may grow too rapidly to contain the overall regret.
Therefore, the DM demands a better algorithm to solve the learning problem in stage two more efficiently.

Given the periodicity of arm rewards, we say that arm $k$ is at \emph{phase} $p$ when the epoch index $t$ divided by period $T_{k}$ yields a remainder $p$, i.e. $p \equiv t (\mathrm{mod}\, T_k)$.
Whenever an arm is at a particular phase, it has the same mean reward and thus can be regarded as the same ``\emph{effective arm}''.
While for an arm at different phases, the learning of the mean rewards has to be conducted separately, i.e., regarded as different effective arms.
Hence, the DM essentially faces $d \coloneqq \sum_{k=1}^K T_k$ effective arms (unique mean rewards) to learn.
Although $d$ is a much smaller number than $T_{\text{LCM}}$, it is still challenging to analyze the learning process since arms are nested due to their asynchronous periods. To this end, we propose a nested confidence-bound-based learning approach in Algorithm \ref{alg:learning} where the exploration and exploitation are carefully designed to maintain tractability.
In the remainder of this section, we focus on the main ideas driving our algorithm and show that the regret achieved is comparable to the optimal one of the classic stationary MAB problem. Related lemmas and detailed analysis are provided in Section \ref{sec:analysis-phasetwo}. %and proofs can be found in Appendix \ref{sec:appendix-proof-analysis-phasetwo}.

We first introduce some notations. Let $\Psi(t) \subseteq \{1,\dots,t-1\}$ be a generic index set of historical epochs before $t$. The operator $|\cdot|$ returns the cardinality when applied to a set.
Given an estimator $\hat{T}_k$ for $T_k$, we define a function $\banditcount{k}{t}(\Psi)$ to count the number of epochs within an index set $\Psi$ that arm $k$ has been pulled at the same phase as $t$:
\begin{align}
\banditcount{k}{t}\left(\Psi\right) \coloneqq \left| \left\{j\in\Psi(t): ~\pi_j=k, ~j \equiv t (\mathrm{mod}\, \hat{T}_k) \right\} \right|.  \label{eq:def-countC}
\end{align}

Our algorithm follows the general principle of exploration and exploitation, which gradually estimates the mean rewards to the desired accuracy and then takes action by treating these estimates as if they are correct.
%Our algorithm gradually estimates the mean rewards to a desired accuracy and then takes actions by treating these estimates as if they are correct.
It is akin to \citet{auer2002using} in the way of examining arms. Specifically, at each epoch $t$ of stage two, Algorithm \ref{alg:learning} chooses an action $\pi_t$ by screening effective arms through a tournament of at most $S$ rounds. We maintain index sets $\Psi^{(s)}(t)$ for $s \in \sets \coloneqq \{1,...,S\}$ where each $\Psi^{(s)}(t)$ tracks epochs of trials made in round $s$ during stage two.
We let $\bar{\Psi}$ denote the set of epochs in stage one.

\begin{algorithm}
    %\SingleSpacedXI
    %\OneAndAHalfSpacedXI
    %\singlespacing
    \caption{stage two: nested confidence-bound-based exploration and exploitation}
    \label{alg:learning}

    \begin{algorithmic}[1]
        \State Input: $T$, $K$, $\sigma$, $n$, $\{\hat T_k\}_{k=1}^K$ (periods estimated from stage one) and a parameter $\delta \in (0,1)$
        \State Define $\bar\Psi = \{1,\dots,nK\}$, $S = \floor{\log_2 T}$ and $\hat{d} = \sum_{k=1}^K\hat T_k$
        \State Initialize $\Psi^{(s)}(nK+1)=\emptyset$ for $s=1,\dots,S$
        \For{$t=(nK+1):T$}
        \State $s\gets 1$, $\Ascr_1 \gets \left\{1,\dots,K\right\}$ \Comment{Start a screening tournament at each epoch} \label{step:reset-tournament}
        \Repeat
        \State Compute the estimated mean $\banditmean_{k,t}^{(s)}$ and the width of confidence interval $w_{k,t}^{(s)}$ based on sample rewards from $\Psi^{(s)}(t)$ and $\bar\Psi$ for all $k\in \Ascr_s$: \label{step:compute-ci}
        \begin{align}
        \banditmean_{k,t}^{(s)} &~=~ \frac{1}{\banditcount{k}{t}\left(\Psi^{(s)}(t) \cup \bar\Psi\right)}\sum_{\substack{j \in {\Psi^{(s)}(t) \cup \bar\Psi}: \\ \pi_j=k, ~ j \equiv t (\mathrm{mod}\, \hat{T}_k) }} Y_{k,j} \label{eq:alg2-s-mean}\\
        w_{k,t}^{(s)} &~=~ \frac{\banditcount{k}{t}\left(\bar\Psi\right)}{\banditcount{k}{t}\left(\Psi^{(s)}(t) \cup \bar\Psi\right)}\sqrt{\frac{4\sigma^2}{\banditcount{k}{t}\left(\bar\Psi\right)}\log\left( \frac{8\hat{d} \banditcount{k}{t}\left(\bar\Psi\right)}{\delta} \right)} \nonumber \\
                       &\quad +\frac{\banditcount{k}{t}\left(\Psi^{(s)}(t)\right)}{\banditcount{k}{t}\left(\Psi^{(s)}(t) \cup \bar\Psi\right)}\sqrt{\frac{4\sigma^2}{\banditcount{k}{t}\left( \Psi^{(s)}(t)\right)}\log\left( \frac{8\hat{d} \banditcount{k}{t}\left(\Psi^{(s)}(t)\right)}{\delta} \right)} \label{eq:alg2-s-ci}
        \end{align}

        \If{$w_{k,t}^{(s)} > 2^{-s}\sigma$ for some $k\in \Ascr_s$} \label{step:wide-ci-begin}
        \State Pull arm $\pi_t=k$ \label{step:wide-ci}
        \State Update $\Psi^{(s)}(t+1)\gets\Psi^{(s)}(t)\cup \{t\}$ and $\Psi^{(s')}(t+1)\gets\Psi^{(s')}(t)$ for $s' \neq s$ \label{step:wide-ci-end}
        \ElsIf{$w_{k,t}^{(s)} \le \frac{\sigma}{\sqrt{T}}$ for all $k\in \Ascr_s$} \label{step:narrow-ci-begin}
        \State Pull arm $\pi_t=\argmax_{k \in \Ascr_s} \banditmean_{k,t}^{(s)}$ \label{step:narrow-ci}
        \State %Update $\Psi_{0}\gets\Psi_{0}\cup \{t\}$
        Update $\Psi^{(s')}(t+1)\gets\Psi^{(s')}(t)$ for $s' = 1,...,S$  \label{step:narrow-ci-end}
        \ElsIf{$w_{k,t}^{(s)} \le 2^{-s}\sigma$ for all $k\in \Ascr_s$} \label{step:eliminate-arms-begin}
        \State $\Ascr_{s+1}\gets \left\{k \in \Ascr_s: \banditmean_{k,t}^{(s)} \geq \max_{k' \in \Ascr_s} \banditmean_{k',t}^{(s)} - 2^{1-s}\sigma \right\}$ \label{step:eliminate-arms}
        \State $s \gets s+1$ \label{step:eliminate-arms-end}
        \EndIf
        \Until{$\pi_t$ is chosen.}
        \EndFor
    \end{algorithmic}
\end{algorithm}

\textbf{Learning the Best Effective Arms.}
Our exploration relies on the idea of using increasingly accurate \emph{confidence bounds} to eliminate clearly suboptimal effective arms.
At each epoch $t$, in Step \ref{step:reset-tournament},
we start an independent tournament ($s\gets1$ and $\Ascr_1\gets \setk$) to screen effective arms active at this particular moment.
In round $s$ of the tournament, we compute the estimated mean $\banditmean_{k,t}^{(s)}$ and the width of confidence interval $w_{k,t}^{(s)}$ based on samples from $\Psi^{(s)}(t)$ and $\bar\Psi$ for each candidate effective arm in Step \ref{step:compute-ci}. Given a target confidence level $2^{-s}\sigma$ of round $s$, Algorithm \ref{alg:learning} proceeds with one of the following three outcomes.
If there is a candidate effective arm with a confidence bound that is still too wide, i.e., $w_{k,t}^{(s)} > 2^{-s}\sigma$ for some $k \in \Ascr_s$, then more exploration is needed for that arm (Step \ref{step:wide-ci-begin} - \ref{step:wide-ci-end}); or if an accurate estimation with a narrow confidence bound is achieved for all candidates, i.e., $w_{k,t}^{(s)} \le \frac{\sigma}{\sqrt{T}}$ for all $k\in \Ascr_s$, then the arm with the highest estimated mean reward is pulled (Step \ref{step:narrow-ci-begin} - \ref{step:narrow-ci-end}); otherwise it enters the next round after eliminating effective arms with unfavorable estimated mean rewards, i.e., the difference is larger than $2^{1-s}\sigma$ or the confidence bound does not overlap with that of the current optimal one (Step \ref{step:eliminate-arms-begin} - \ref{step:eliminate-arms-end}). The tournament at $t$ terminates when an arm is chosen to be pulled.
Note that we directly choose the highest mean, instead of the upper confidence bound used by \citet{auer2002using}, in Step \ref{step:narrow-ci}. This modification leads to a cleaner regret analysis and is also applied in \citet{li2017provably}.

\textbf{Reducing Sample Dependence.} \label{pg:Reducing Sample Dependence}
An important complication in analyzing the regret is the \emph{nested} inter-temporal dependence over the horizon caused by asynchronous periods of the arms.
For example, to compare two effective arms at a given epoch, the DM needs to backtrack the historical epochs at which they were pulled. However, due to the different lengths of periods, these two effective arms may have never appeared in the same epoch simultaneously in the past. The tracking of individual phases of all arms introduces a great deal of statistical dependence, which significantly complicates the regret analysis when applying standard MAB methodologies.
To handle this issue, we maintain mutually exclusive index sets $\Psi^{(s)}(t)$ with the techniques developed by \cite{auer2002using} such that the observed rewards in $\Psi^{(s)}(t)$ are not used to eliminate arms in the same index set and thus the samples from the same $\Psi^{(s)}(t)$ are independent.
%\blue{To handle this issue, we maintain mutually exclusive sets $\Psi^{(s)}(t)$ over screening rounds with the techniques developed by \cite{auer2002using} such that the exploration during round $s$ (Step 8-10) does not depend on the values of rewards already observed in the same round and thus the samples in $\Psi^{(s)}(t)$ remain independent.}
We also make efforts to efficiently utilize data.
In particular, rewards observed for frequency identification during stage one are reused in stage two for estimation purposes, i.e., means and confidence bounds are computed based on samples selected from $\Psi^{(s)}(t)$ and $\bar\Psi$.
This data reusing benefits the exploration by making it possible to reach the desired estimation accuracy more quickly.
However, a direct combination of $\Psi^{(s)}(t)$ and $\bar\Psi$ would contaminate the sample independence again because $\hat{T}_k$ estimated from observed samples in $\bar\Psi$ is used to identify the phases, i.e., samples from $\Psi^{(s)}(t)$ and $\bar\Psi$ are implicitly correlated.
To this end, we carefully design the computation scheme of the estimation to control the dependence structure, and eventually we still manage to obtain a valid regret bound.

In the next Sections \ref{sec:analysis} and \ref{sec:lower-bound}, we evaluate the performance of our policy in terms of regret.

\section{Upper Bound for the Regret}
\label{sec:analysis}
Our main result is the following upper bound on the expected regret of our two-stage policy.
\begin{theorem} \label{thm:upper_bound}
Given Assumptions \ref{asp:mean-reward} - \ref{asp:bBratio} and parameters chosen in \eqref{eq:parameter}, the expected regret of the two-staged policy $\pi$ proposed in Section \ref{sec:policy} is bounded as,
\begin{align}
\E[R_{T}^{\pi}] \le \textit{Constant} \cdot \RegretBound, \nonumber
\end{align}
where $d=\sum_{k=1}^{K}T_k$ and the Constant does not depend on $T$, $K$ or any $T_k$ for $k \in \setk$.
\end{theorem}

Theorem \ref{thm:upper_bound} shows that our algorithm achieves regret $\tilde{O}(\sqrt{T\sum_{k=1}^K T_k})$.
Recall that there are $\sum_{k=1}^K T_k$ effective arms in our setting if the phases of each arm are regarded as separate arms.
That is, there are $\sum_{k=1}^K T_k$ parameters to learn as opposed to $K$ parameters in the classic stationary MAB problem.
In this sense, the obtained regret rate matches the optimal regret rate $O(\sqrt{TK})$ of the classic stationary MAB problem.
In Section~\ref{sec:lower-bound}, we establish a lower bound $\Omega(\sqrt{T\max_k\{ T_k\}})$, which shows that our algorithm is optimal up to a factor of $\sqrt{K}$.
Moreover, in Appendix~\ref{sec:appendix-extended-results-SamePeriodDifferentSeasonality}, when $T_1=\dots=T_K$,
we prove a lower bound $\Omega(\sqrt{KTT_1})$ which matches the regret rate $\tilde O(\sqrt{T\sum_{k=1}^K T_k})$,
and hence our algorithm achieves optimal in that case.
It is encouraging to see that our algorithm performs remarkably well in terms of regret, although the non-stationarity introduces a significant complication.

Theorem \ref{thm:upper_bound} is based on certain technical conditions. Specifically, Assumption \ref{asp:mean-reward}, \ref{asp:sub-Gaussian} and \ref{asp:period} are imposed on the mean reward, the random noise and the period of each arm respectively. We also need a condition on the magnitude of present frequencies, which is formally introduced as Assumption \ref{asp:bBratio} in Section \ref{sec:analysis-phaseone} for the simplicity of notation.
These assumptions are quite mild.
We will provide justifications that imposing these conditions is not a limitation of our policy and generally does not constrain the practical application.

We summarize the outline of the regret analysis, which is composed of two parts.
In Section \ref{sec:analysis-phaseone}, we show that Algorithm \ref{alg:frequency-identification} is able to correctly identify periods of all arms with high probability in stage one. In Section \ref{sec:analysis-phasetwo}, we show that Algorithm \ref{alg:learning} achieves a regret as stated in Theorem~\ref{thm:upper_bound}.

\subsection{Estimate Periods Correctly in Stage One}
\label{sec:analysis-phaseone}
As introduced in Section \ref{sec:phaseone}, Algorithm~\ref{alg:frequency-identification} relies on a proper threshold for the periodogram. In this section, we elaborate on how the threshold works by developing a series of lemmas. The detailed proofs are provided in Appendix \ref{sec:appendix-proof-analysis-phaseone}.
These lemmas together with related assumptions guarantee that the lengths of periods of all arms can be correctly estimated with high probability, and this conclusion is applied in the regret analysis of Theorem \ref{thm:upper_bound}.

In Algorithm \ref{alg:frequency-identification}, we observe $n$ reward samples from arm $k$ to produce the DFT $\tilde{y}_k(v)$  which can be decomposed as $\tilde{y}_k(v)=\tilde{\mu}_k(v) + \tilde{\epsilon}_k(v)$ in \eqref{eq:dft-decomp}, and then we investigate the periodogram $|\tilde{y}_k(v)|$ in the frequency domain $v \in [0, 1/2]$.
To develop the threshold, we first provide an upper bound on the DFT of the noise $\tilde{\epsilon}_k(v)$ in Lemma \ref{lem:noise-bound}.

\begin{lemma}\label{lem:noise-bound}
For each arm $k \in \setk$ and any $\delta > 0$, $\displaystyle{ \PR\bigg( \NB \ge \delta \bigg) \leq 48n \exp\left(- \frac{n\delta^2}{4\sigma^2}\left( 1-\frac{\pi }{24} \right)^2 \right)}$.
\end{lemma}

Specifically, if choosing the upper bound $\delta$ in Lemma~\ref{lem:noise-bound} as $\bar\epsilon_v = \frac{2\sigma H}{1-\pi/24}\sqrt{\frac{\log(n)}{n}}$ defined in Step \ref{step:threshold} of Algorithm \ref{alg:frequency-identification}, then we have ${\PR \bigg(\NB \ge \bar\epsilon_v \bigg) \leq \frac{48}{n^{H^2-1}}}$. Note that Lemma \ref{lem:noise-bound} deviates from the standard concentration bounds in that we bound $\tilde{\epsilon}_k(v)$ for all $v \in [0,1/2]$ as a whole.
Although we can control $\tilde{\epsilon}_k(v)$ individually for each $v$, applying the union bound directly for an uncountable number of random variables does not work. To this end, we leverage the analytical structure of $\tilde{\epsilon}_k(v)$ to discretize the domain of $v$ first, and then control the bound of each sub-interval before applying the union bound.

Next we study the leakage caused by $\tilde{\mu}_k(v)$. Recall that $\tilde{\mu}_k(v)$ can be decomposed into frequency components in \eqref{eq:dft-deterministic}, and the present frequencies are components $j/T_k$ with $|b_{k,j}|>0$. Let $B_k$ and $b_k$ denote the magnitude of the strongest and the weakest present frequency components of $\tilde{\mu}_{k}(v)$ respectively as following
\begin{align}
B_{k} ~\coloneqq~ &\sup \left\{\abs{b_{k,j}}: |b_{k,j}|>0, j=0,\dots,T_k-1 \right\},  \label{eq:def-Bk} \\
b_{k} ~\coloneqq~ &\inf \left\{\abs{b_{k,j}}: \abs{b_{k,j}}>0,j=0,\dots,T_k-1 \right\}.   \label{eq:def-bk}
\end{align}
Algorithm \ref{alg:frequency-identification} considers a neighborhood of width $\frac{2g}{n}$ around each present frequency. Now, we let $\Vscr_k \coloneqq [0,1/2] \cap \{\cup_{j:|b_{k,j}|>0} ~[\frac{j}{T_k}-\frac{g}{n}, \frac{j}{T_k}+\frac{g}{n}]\}$ represent the union of neighborhoods of all present frequencies, and let $\overline{\Vscr}_k \coloneqq [0,1/2]\setminus \Vscr_k$.
When inspecting the periodogram in $\overline{\Vscr}_k$, we expect to bound the leakage caused by all present frequencies. On the other hand, when inspecting the periodogram in the neighborhood of one present frequency, we expect to bound the leakage produced by other present frequencies. We develop Lemma \ref{lem:bound-leakage} for these two purposes.
Recall that ${A_j = \sup \left\{ \frac{|\sin(\pi \nu)|}{\pi \nu}: \nu \in [j,j+1] \right\} }$ for $j\in \N^+$, $\displaystyle{U_1 = \sum_{j=0}^{\lfloor{\frac{n-2g -1}{4g }\rfloor}} A_{(2j+1)g}}$ and $\displaystyle{U_2 = \sum_{j=1}^{\lfloor{\frac{n-1}{4g }\rfloor}} A_{2jg-1}}$ are defined in Algorithm \ref{alg:frequency-identification}.

\begin{lemma}\label{lem:bound-leakage}
The following two bounds hold for each arm $k \in \setk$,
\begin{align}
\sup_{v\in\bar\Vscr_k}\Bigg|\sum_{j=0}^{T_{k}-1}\tilde{\mu}_{k,j}(v)\Bigg| & \le \pi B_k U_1 \text{ and } \sup_{v\in{\Vscr}_k}\Bigg|\sum_{j:|\frac{j}{T_k}-v|\geq \frac{g}{n}}\tilde{\mu}_{k,j}(v)\Bigg| \le \pi B_k U_2\le\pi B_k U_1. \nonumber
\end{align}
\end{lemma}

Note that the upper bounds provided in Lemma \ref{lem:bound-leakage} depend on $B_k$ which is unknown a priori. Hence, we develop a data-driven upper bound for $B_k$ in Lemma \ref{lem:Bound-B-Above}.

\begin{lemma}\label{lem:Bound-B-Above}
For each arm $k \in \setk$, $\displaystyle{B_k \leq  \frac{1}{1-\pi U_2}\left(\NB + \PB \right)}$.
\end{lemma}

Putting all the above pieces together, we establish Lemma~\ref{lem:threshold} which shows that the threshold $\tau_k$ used in Algorithm \ref{alg:frequency-identification} is sufficiently large to filter out the noise and leakage with high probability. %of $v \in \overline{\Vscr}_k$.
\begin{lemma}\label{lem:threshold}
For each arm $k \in \setk$, $\displaystyle{\sup_{v\in {\overline{\Vscr}_k}}|\tilde{y}_k(v)|} \le \tau_{k}$ holds with probability at least $\ProbNB$.
\end{lemma}
We also expect that the threshold $\tau_k$ is not chosen too large, otherwise it would suppress the periodogram of present frequencies as well.
Next, we investigate the lower bound of the periodogram at present frequencies
in Lemma \ref{lem:no-false-negative}.

\begin{lemma}\label{lem:no-false-negative}
For each arm $k \in \setk$, given the condition
\begin{align}\label{eq:guarantee}
b_k\geq \left(\frac{2\pi U_1}{1 - \pi U_2}+2\right)\bar\epsilon_v + \left(\frac{\pi U_1}{1-\pi U_2}\max\left\{\pi U_1, \pi U_2+1\right\}+\pi U_2\right)B_k,
\end{align}
then $|\tilde{y}_k(j/T_k)| > \tau_k$ holds for all present frequencies $j/T_k$ such that $|b_{k,j}|>0$ with probability no less than $\ProbNB$.
\end{lemma}

Lemma \ref{lem:no-false-negative} states that the periodogram at each present frequency $j/T_k$ is above $\tau_k$ with high probability, so Step \ref{step:global-maximum} of Algorithm \ref{alg:frequency-identification} will select a local maximum $v_{j}^{*}$ from the neighborhood $[\frac{j}{T_k}-\frac{g}{n}, \frac{j}{T_k}+\frac{g}{n}]$ because $|\tilde{y}_k(v_{j}^{*})| \ge |\tilde{y}_k(j/T_k)| > \tau_k$.
However, $v_{j}^{*}$ is not necessarily at $j/T_k$, and thus we still need to identify the correct $j/T_k$ after locating $v_{j}^{*}$. We do so by searching the candidate set $\setf$ and matching $v_{j}^{*}$ to the nearest possible frequency.
The difficulty is that $v_{j}^{*}$ we found is possibly closer to some $j'/T_{k}{'}$ where $T_{k}{'} \neq T_{k}$ other than the correct $j/T_k$, and thus it will lead to a wrong estimation of the present frequency. To resolve this issue, we continue to work on improving the resolution of frequency identification. We revisit the upper bound on the noise in Lemma \ref{lem:local-noise-bound}.

\begin{lemma}\label{lem:local-noise-bound}
For  each arm $k \in \setk$ and any $\delta > 0$, let ${\Uscr_k \coloneqq [0,1/2] \cap \{\cup_{j:}[\frac{j}{T_k}-\frac{g}{n}, \frac{j}{T_k}+\frac{g}{n}]\}}$,
\begin{align}
\PR\bigg( \sup_{v\in\Uscr_k} |\tilde{\epsilon}_k{(v)}| > \delta \bigg) \leq 200n \exp\left(  -\frac{0.233n\delta^2}{\sigma^2}  \right) + 200n\exp\left( -\frac{0.291n\delta^2}{\sigma^2}  \right). \nonumber
\end{align}
\end{lemma}

Note that $\Uscr_k$ represents the union of neighborhoods around all possible locations of present frequencies. Comparing to that Lemma~\ref{lem:noise-bound} bounds the noise in the frequency interval of $[0,1/2]$, Lemma~\ref{lem:local-noise-bound} focuses on a smaller area $\Uscr_k$ and thus we can control the union bound of the noise into a tighter one.
In particular, we have ${\PR\left( \sup_{v\in\Uscr_k} \lvert \tilde{\epsilon}{(v)} \rvert >\frac34\bar\epsilon_v  \right) \leq \frac{200}{n^{0.867 H^2-1}}+\frac{200}{n^{0.694  H^2 -1}}}$ if choosing the upper bound $\delta$ in Lemma \ref{lem:local-noise-bound} as $\frac{3}{4}\bar\epsilon_v$.
To make a comparison, we consider the example discussed in the end of Section \ref{sec:phaseone-algorithm}. In that case, Lemma~\ref{lem:local-noise-bound} provides $ \PR\left( \sup_{v\in\Uscr_k} \lvert \tilde{\epsilon}_k{(v)} \rvert \geq\frac34\bar\epsilon_v \right) \leq 0.0169$ in contrast to that Lemma~\ref{lem:noise-bound} provides $\PR\left(   \sup_{v\in[0,1/2]} \lvert \tilde{\epsilon}_k{(v)} \rvert  \geq \frac34\bar\epsilon_v \right) \leq 0.0430$. With the better controlled noise bound given in Lemma~\ref{lem:local-noise-bound}, we can further develop the following.

\begin{lemma}\label{lem:narrow-down-neighborhood}
For each arm $k \in \setk$, if $b_k \geq 2 \bar\epsilon_v + \frac{8\pi}{3}U_2B_k$, then all maxima $\displaystyle{v_{j}^{*} \in \mathop{\arg\sup}_{v \in [\frac{j}{T_k}-\frac{g}{n}, \frac{j}{T_k}+\frac{g}{n}]} |\tilde{y}_k(v)|}$ are attained in $[\frac{j}{T_k}-\frac{1}{n},\frac{j}{T_k}+\frac{1}{n}]$ for all present frequencies $j/T_k$ such that $|b_{k,j}|>0$ with probability no less than $\ProbNBtight$.
\end{lemma}

Lemma \ref{lem:narrow-down-neighborhood} shows that the distance between one present frequency $j/T_k$ and its nearby local maximum $v_{j}^{*}$ is no more than $\frac{1}{n}$ with high probability.
Given $T_k < \frac{n}{2g} < \frac{\sqrt{n}}{2}$ imposed by Assumption \ref{asp:period}, the distance between any two candidate frequencies in set $\setf$ is no less than $\left|\frac{j}{i}- \frac{j'}{i'}\right| \ge \frac{1}{i'i}> \frac{2}{n}$. Hence, the present frequency $j/T_k$ is indeed closest to $v_{j}^{*}$ among all candidate frequencies in $\setf$, and it will be matched correctly in Step \ref{step:pick-frequency} of Algorithm \ref{alg:frequency-identification} after $|\tilde{y}_k(v_{j}^{*})| > \tau_k$ being located.

We combine the conditions required by Lemma \ref{lem:no-false-negative} and Lemma \ref{lem:narrow-down-neighborhood} together, in particular inequality \eqref{eq:guarantee} and $b_k \geq 2 \bar\epsilon_v + \frac{8\pi}{3}U_2B_k$, to impose the following technical assumption, which requires that the magnitude of the weakest present frequency $b_k$ cannot be too small compared to that of the noise and the strongest one $B_k$.
Finally, we conclude with Theorem \ref{thm:freq-identification}.

\begin{assumption}\label{asp:bBratio}
For each arm $k \in \setk$,
\begin{align}
b_k& \geq \left(\frac{2\pi U_1}{1 - \pi U_2}+2\right)\bar\epsilon_v +\left(\frac{8\pi}{3}U_2\right) \lor \left(\frac{\pi U_1}{1-\pi U_2}\max\left\{\pi U_1, \pi U_2+1\right\}+\pi U_2\right)B_k. \label{eq:assumptionbBratio}
\end{align}
\end{assumption}

\begin{theorem}\label{thm:freq-identification}
Let $\Lambda \coloneqq \{ \hat{T}_k = T_k, ~ \forall k \in \setk \}$ denote the event that periods of all arms are correctly estimated. Under Assumption \ref{asp:sub-Gaussian} - \ref{asp:bBratio}, Algorithm \ref{alg:frequency-identification} ensures $\PR(\Lambda)\ge 1-\frac{48K}{n^{H^2-1}}-\frac{200K}{n^{0.867H^2-1}}-\frac{200K}{n^{0.694 H^2 -1}}$.
\end{theorem}

Theorem \ref{thm:freq-identification} relies on Assumption \ref{asp:sub-Gaussian} - \ref{asp:bBratio}. Next we argue that they are not restrictive in practice.
As explained earlier, Assumption \ref{asp:mean-reward} and \ref{asp:sub-Gaussian} are commonly used in the MAB literature, and Assumption \ref{asp:period} is a fundamental requirement in spectral analysis to identify frequencies.
It is also worth noting that Assumption~\ref{asp:bBratio} holds automatically for a sufficiently large $n$ and it may be easily satisfied even for small $n$ when using parameters chosen in \eqref{eq:parameter}.
We demonstrate in Table~\ref{tab:asp4} that constraint \eqref{eq:assumptionbBratio} imposed by Assumption~\ref{asp:bBratio} gradually relaxes as $n$ increases and the probability of failure in correctly estimating periods of all $K=5$ arms reduces rapidly as well.
\begin{table}[!h]
%\SingleSpacedXI
\centering
\begin{tabular}{@{\extracolsep{7pt}} rcccc}
\toprule
 $n$ & $U_1$ & $U_2$ & Assumption \ref{asp:bBratio}  & $1- \PR(\Lambda)$ in Theorem \ref{thm:freq-identification}\\
\midrule
$ 50 $  & 0.05047  & 0.02054 & $b_k \geq 3.337 \sigma + 0.2450 B_k$ & \( 0.08365 \)   \\
$ 100 $  & 0.04077 & 0.02438 & $b_k \geq 2.663 \sigma + 0.2259 B_k$ & \( 1.679 \cdot 10^{ - 3} \)   \\
$ 200 $  & 0.03175 & 0.01970 & $b_k \geq 2.080 \sigma + 0.1748 B_k$ & \( 1.756 \cdot 10^{ - 5} \)   \\
$ 500 $  & 0.02439 & 0.01591 & $b_k \geq 1.489 \sigma + 0.1347 B_k$ & \( 1.533 \cdot 10^{ - 8} \)  \\
%$ 1000 $  & 0.01993 & 0.01295 & $b_k \geq 1.146 \sigma + 0.1086 B_k$ & \( 3.436 \cdot 10^{ - 11} \)   \\
\bottomrule
\end{tabular}
\vspace{3mm}
\caption{A demonstration of Assumption \ref{asp:bBratio} and other constants when $K=5$.}
\label{tab:asp4}
\end{table}

\subsection{Bound the Regret in Stage Two}
\label{sec:analysis-phasetwo}
In this section, we evaluate the regret incurred by Algorithm \ref{alg:learning}. As discussed in Section \ref{sec:phasetwo}, the main difficulty in applying regret analysis techniques is the dependence caused by reusing data of stage one.
In particular, we use $\hat{T}_k$ to identify phase when selecting samples of a particular effective arm from $\Psi^{(s)}(t) \cup \bar\Psi$ in Step \ref{step:compute-ci} of Algorithm \ref{alg:learning}.
Because $\hat{T}_k$ is calculated from rewards observed in stage one, namely $\bar\Psi$, these samples selected in Step \ref{step:compute-ci} are implicitly correlated.

To resolve this issue, we consider a scenario that the bandit problem is played by a weak oracle who is aware of the lengths of periods of all arms but not any value of the mean reward, i.e., $T_k$ for all $k \in \setk$ are known but any $\mu_{k,t}$ is not.
Note that we deviate from the standard terminology as we do not use the oracle to refer to the policy that knows both the periods and the mean rewards.
We assume that this oracle follows exactly the same two-stage policy where she pulls each arm $n$ times in stage one and implements Algorithm \ref{alg:learning} in stage two with the only exception that she directly uses the true values of $T_k$ in Algorithm \ref{alg:learning}.
Since the oracle does not estimate $T_k$ from rewards observed in stage one, we are able to disentangle the aforementioned dependence which occurred in sample selection from $\Psi^{(s)}(t) \cup \bar\Psi$. It tremendously simplifies the analysis and thus we can develop a regret bound for the oracle policy.
Conditional on the event when our policy correctly estimates periods of all arms in stage one, we then show that it performs the same as the oracle policy.  Drawing on this observation, we eventually prove the regret bound stated in Theorem~\ref{thm:upper_bound} by carefully examining the connections between these two policies.

We start the analysis for the oracle policy by introducing some notations.
Since the oracle policy is identical to our policy in many aspects, we use $\tilde{\cdot}$ to differentiate a term of the oracle policy only when necessary. For example, we let $\tpi_k$ denote the action taken by the oracle policy $\tpi$, and in particular, we need to pay attention to Step \ref{step:compute-ci} of Algorithm \ref{alg:learning} where the oracle policy directly applies the true period information $T_k$ comparing to the estimation $\hat{T}_k$ used by our policy. Hence, the following  quantities originally defined in \eqref{eq:def-countC}, \eqref{eq:alg2-s-mean} and \eqref{eq:alg2-s-ci} need to be revised accordingly as
\begin{align}
\banditcountoracle{k}{t}\left(\Psi\right) &~=~ \left| \left\{j\in\Psi: ~\tpi_j=k, ~j \equiv t (\mathrm{mod}\, T_k) \right\} \right|, \label{eq:def-countC-oracle} \\
\banditmeanoracle_{k,t}^{(s)} &~=~ \frac{1}{\banditcountoracle{k}{t}\left(\Psi^{(s)}(t) \cup \bar\Psi\right)}\sum_{\substack{j \in {\Psi^{(s)}(t) \cup \bar\Psi}: \\ \tpi_j=k, ~ j \equiv t (\mathrm{mod}\, T_k) }} Y_{k,j}, \label{eq:alg2-s-mean-oracle} \\
\banditwidthoracle_{k,t}^{(s)} &~=~ \frac{\banditcountoracle{k}{t}\left(\bar\Psi\right)}{\banditcountoracle{k}{t}\left(\Psi^{(s)}(t) \cup \bar\Psi\right)}\sqrt{\frac{4\sigma^2}{\banditcountoracle{k}{t}\left(\bar\Psi\right)}\log\left( \frac{8{d} \banditcountoracle{k}{t}\left(\bar\Psi\right)}{\delta} \right)} \nonumber \\
  &\quad +\frac{\banditcountoracle{k}{t}\left(\Psi^{(s)}(t)\right)}{\banditcountoracle{k}{t}\left(\Psi^{(s)}(t) \cup \bar\Psi\right)}\sqrt{\frac{4\sigma^2}{\banditcountoracle{k}{t}\left( \Psi^{(s)}(t)\right)}\log\left( \frac{8{d} \banditcountoracle{k}{t}\left(\Psi^{(s)}(t)\right)}{\delta} \right)},\label{eq:alg2-s-ci-oracle}
\end{align}
where $\hat{T}_k$ and $\hat{d}$ are replaced by $T_k$ and $d= \sum_{k=1}^K T_k$ respectively.
The derivation of the regret of the oracle policy is done by a series of lemmas, and we summarize the road map as following. Lemma \ref{lem:independent} formally states the conditional independence of the samples for the oracle policy; Lemma \ref{lem:confidence_set} shows that $\banditmeanoracle_{k,t}^{(s)}$ estimated by the oracle policy is close to the true mean $\mu_{k,t}$ with high probability; Lemma \ref{lem:single-step-regret} examines the screening process of Algorithm \ref{alg:learning} conducted by the oracle policy; and these results lead to Lemma \ref{lem:oracle-bound} which bounds the expected regret of the oracle policy. The detailed proofs of these lemmas are deferred to Appendix \ref{sec:appendix-proof-analysis-phasetwo}, and the techniques used are adapted from \citet{auer2002using} and \citet{li2017provably}.

\begin{lemma}\label{lem:independent}
For all $t \in \{nK+1,...,T\}$ and $s \in \sets$, conditional on the stage-one rewards $\{Y_{\tpi_\tau,\tau}: \tau \in \bar\Psi\}$, the set $\Psi^{(s)}(t)$ and the arms being pulled by the oracle policy $\{ \tpi_{\tau}: \tau \in \Psi^{(s)}(t) \cup \bar\Psi \}$, the rewards $\{ Y_{\tpi_\tau,\tau}: \tau \in \Psi^{(s)}(t) \cup \bar\Psi \}$ are independent random variables with mean $\mu_{ \tpi_\tau,\tau}$.
\end{lemma}

\begin{lemma}\label{lem:confidence_set}
Define event $\mathcal{E} \coloneqq \left\{ \left| \banditmeanoracle_{k,t}^{(s)} -\mu_{k,t} \right| \leq \banditwidthoracle_{k,t}^{(s)}, ~ \forall k \in \setk, t \in \{nK+1,...,T\}, s \in \sets \right\}$.
Then, $\PR(\mathcal{E}) \ge 1- \delta S$ holds for the oracle policy.
\end{lemma}

\begin{lemma} \label{lem:single-step-regret}
Let $\pi_t^* \coloneqq \argmax_{k \in \setk} \mu_{k,t}$ denote the optimal arm. Suppose that event $\mathcal{E}$ holds, and that the oracle policy chooses arm $\tpi_{t}$ in round $s$ by Algorithm \ref{alg:learning} at epoch $t$, then
\begin{enumerate}
\item The optimal arm is never excluded during screening: $\pi_t^* \in \Ascr_{s'}, \forall \,s'\leq s$;
\item If $\tpi_{t}$ is chosen in Step~\ref{step:wide-ci} when $s = 1$, then $\mu_{\pi_t^*,t}-\mu_{\tpi_t,t} \leq 1$;
\item If $\tpi_{t}$ is chosen in Step~\ref{step:wide-ci} when $s \geq 2$, then $\mu_{\pi_t^*,t}-\mu_{\tpi_t,t} \leq \frac{8\sigma}{2^{s}}$;
\item If $\tpi_{t}$ is chosen in Step~\ref{step:narrow-ci}, then $\mu_{\pi_t^*,t}-\mu_{\tpi_t,t} \leq \frac{2\sigma}{\sqrt{T}}$.
\end{enumerate}
\end{lemma}

\begin{lemma}\label{lem:oracle-bound}
Let $R_{T}^{\tilde{\pi}}$ denote the pseudo-regret of the oracle policy $\tilde{\pi}$. Then, under Assumption \ref{asp:mean-reward} and \ref{asp:sub-Gaussian}, the expected regret of the oracle policy is bounded as,
\begin{align}
\E[R_{T}^{\tpi}] \leq \textit{Constant} \cdot \RegretBound ~. \nonumber
\end{align}
where the \textit{Constant} is not related to $T$, $K$ or any $T_k$ for $k \in \setk$.
\end{lemma}

Finally, we proceed to the regret analysis of our two-staged policy.
\proof{{Proof of Theorem \ref{thm:upper_bound}}:}
By Theorem \ref{thm:freq-identification} we showed that periods of all arms can be correctly estimated through Algorithm \ref{alg:frequency-identification}, i.e., event $\Lambda$ holds, with high probability. Conditional on $\Lambda$, our policy and the oracle policy can be coupled to have exactly the same sample path such that they pull the same arms in the same order, i.e., $\pi_t = \tpi_t$ for $t=1,...,T$ and observe identical rewards sequentially as well.
Therefore, we have $\E[R_{T}^{\pi} \mathbf{1}_{\{\Lambda\}}] = \E[R_{T}^{\tpi} \mathbf{1}_{\{\Lambda\}}] \le \E[R_{T}^{\tpi}]$ where $\mathbf{1}$ denotes the indicator function.
On the other hand, on event $\Lambda^{c}$ where not all periods are correctly estimated,  Assumption~\ref{asp:mean-reward} implies a regret bound $T$ for our policy.
Given the probability $\PR(\Lambda)$ derived in Theorem \ref{thm:freq-identification}, we have
\begin{align}
\E[R_{T}^{\pi}] &~=~ \E[R_{T}^{\pi} \mathbf{1}_{\{\Lambda\}}] + \E[R_{T}^{\pi} \mathbf{1}_{\{\Lambda^{c}\}}] \nonumber \\
&~\leq~ \E[R_{T}^{\tpi}] + T \left( \frac{48K}{n^{H^2-1}}+\frac{200K}{n^{0.867 H^2-1}}+\frac{200K}{n^{0.694 n H^2 -1}} \right) \nonumber \\
&~\leq~ \textit{Constant} \cdot \RegretBound ~. \nonumber
\end{align}
where the \textit{Constant} is not related to $T$, $K$ or any $T_k$ for $k \in \setk$.
The above result applies because that if parameters $n= \lfloor \sqrt{T/K}\rfloor$ and $H=\sqrt{1 +\log(n)}$ are chosen as \eqref{eq:parameter}, then the second term in the first inequality goes to 0 as $T$ becomes large. Finally, Theorem \ref{thm:upper_bound} is established.
\endproof

\section{Lower Bound for the Regret}
\label{sec:lower-bound}
Given the fact that the classic stationary MAB problem has a regret lower bound $\Omega(\sqrt{TK})$ and the DM essentially faces $\sum_{k=1}^K T_k$ effective arms to learn in our setting, it is reasonable to conjecture that our MAB model incurs a regret lower bound $\Omega(\sqrt{T\sum_{k=1}^K T_k})$, which ideally matches the regret upper bound of our learning policy $\tilde{O}(\sqrt{T\sum_{k=1}^K T_k})$ proved in Theorem \ref{thm:upper_bound}.
It is easy to verify this conjecture in certain special cases.
In the aforementioned example where all arms share a common period, i.e., $T_1=...=T_K$, the learning problem can be decomposed into $T_1$ independent subproblems, each with a horizon $T/T_1$ and $K$ arms,
and hence the optimal rate of regret is $O(T_1\sqrt{KT/T_1 })=O(\sqrt{KTT_1})$, matching the upper bound.
However, after careful investigation, we find that the regret depends on the number-theoretical relationship among $T_k$'s in the general case, and hence the exact lower bound is not straightforward to establish.
In this section, we present a slightly weaker lower bound that holds generally.
The detailed proofs are provided in Appendix~\ref{sec:proof-lower-bound}.

We first consider a simple case.
Let $\banditset_1$ be the class of $K$-armed unit-variance Gaussian bandits, where the arms are stationary with mean $\mu_k$, i.e., $\mu_{k,t}\equiv \mu_k \in [0,1]$ for $k\in \setk$.
Given an instance $\nu$ from $\banditset_1$, we use $\E_{\nu}[R_{T}^{\pi}]$ to denote the expected regret incurred by a policy $\pi$ in this MAB instance.
Note that $\banditset_1$ belongs to the classic stationary MAB model.
Using common techniques, we have the following regret lower bound of $\banditset_1$ for any policy that knows $\mu_k$ for $k \ge 2$ but doesn't know $\mu_1$.
\begin{lemma}
\label{lem:lowerbound-no-arm-periodic}
For any such policy $\pi$, $\displaystyle{\sup_{\nu\in\banditset_1}\E_{\nu}[R_T^\pi] \geq \frac{1}{8 \sqrt{2} e}\sqrt{T}}$.
\end{lemma}

We next consider an advanced case.
Let $\banditset_2(T_1)$ be the class of MAB instances similar to $\banditset_1$ except that the mean reward of its first arm $\mu_{1,t}$ is periodic with minimum period $T_1$.
Note that $\banditset_2(T_1)$ is different from $\banditset_1$ when $T_1\ge 2$.
Using the idea of decomposing the learning problem into $T_1$ independent subproblems and then applying Lemma \ref{lem:lowerbound-no-arm-periodic}, we derive the following regret lower bound of $\banditset_2(T_1)$ for policies that know $T_1$ and $\mu_k$ for $k\ge2$ but not $\mu_{1,t}$ for any $t$.
\begin{lemma}
\label{lem:lowerbound-one-arm-periodic}
For any such policy $\pi$, $\displaystyle{\sup_{\nu\in\banditset_2(T_1)}\E_{\nu}[R_T^\pi] \geq \frac{1}{16e}\sqrt{TT_1}}$.
\end{lemma}
Obviously, for policies that don't have the knowledge of $T_1$, the regret can only be larger.

Now we consider the most sophisticated case.
Let $\banditset_3(T_1,...,T_K)$ be the class of MAB instances where all arms are periodic with minimum periods $T_1,...,T_K$.
When $T_k\geq2$ for some $k\ge2$, $\banditset_3(T_1,...,T_K)$ does not degenerate to $\banditset_2(T_1)$, and hence the lower bound in Lemma~\ref{lem:lowerbound-one-arm-periodic} does not readily apply to the general cases in our periodic MAB model.
To tackle this issue, we introduce a metric in the union $\banditset_{2}(T_1) \cup \banditset_{3}(T_1,...,T_K)$.
We show that $\banditset_{3}$ is \emph{dense} in $\banditset_{2} \cup \banditset_{3}$ under the metric.
Moreover, as a functional of $\nu\in \banditset_{2} \cup \banditset_{3}$, the regret $\E_{\nu}[R_T^{\pi}]$ is \emph{continuous} with respect to the metric.
Therefore, by slightly perturbing the mean reward $\mu_{k,t}$ for certain arms $k\ge2$ in any instance $\nu_1 \in \banditset_{2}(T_1)$,
we can obtain a new instance $\nu_2 \in \banditset_{3}(T_1,...,T_K)$ which has the required periods,
and this perturbation does not affect the regret significantly.
Finally, due to symmetry, we may derive the following result by rotating over $K$ arms.
\begin{theorem}
\label{thm:lower-bound}
For any policy $\pi$ that knows $T_k$ for $k \in \setk$, $\displaystyle{\sup_{\nu\in\banditset_3(T_1,...,T_K)}\E_{\nu}[R_T^\pi] \geq \frac{1}{32e} \sqrt{T\max_{k \in \setk}\{T_k\}}}$.
\end{theorem}

Theorem \ref{thm:lower-bound} shows that our MAB model incurs a regret lower bound $\Omega(\sqrt{T\max_{k}\{T_k\}})$.
Note that
$\sqrt{T\max_{k}\left\{T_k\right\}}\ge \frac{1}{\sqrt{K}}\sqrt{T\sum_{k=1}^{K}T_k}$.
Therefore, when $K$ does not scale with $T$ and is treated as a constant,
it matches the regret rate in Theorem~\ref{thm:upper_bound} and implies that our policy is near-optimal.

\section{Conclusion}
\label{sec:conclusion}

In this paper, we study a non-stationary MAB problem with periodic rewards.
The rewards of the arms may have different lengths of periods, which are unknown initially.
We design an algorithm that first identifies the lengths of the periods and then learns the best arm at each epoch.
The regret upper bound of the algorithm is $\tilde O(\sqrt{T\sum_{k=1}^K T_k})$, matching the lower bound
$\Omega(\sqrt{T\max_{k}\{T_k\}})$ up to a factor of $\sqrt{K}$.
Our study opens a wide range of interesting directions in online learning with non-stationary dynamics.

\begin{itemize}
\item The length of the period of an arm is always an integer in our setup.
However, the complex representation \eqref{eq:exp-rep} of the rewards is valid for any periodic functions.
It provides a more flexible framework as the period may not be a multiple of the sampling rate ($t=1,2,\dots$ in MAB problems) in practice.
We may need to construct confidence bounds for the parameters in the complex representation directly. This is left for future study.

\item
Our policy estimates periods and learns rewards separately in stages one and two respectively.
Ideally, an algorithm that integrates the learning of periods and rewards as well as exploitation may perform better in practice.
The main challenge of the integrated approach is the analysis of the Fourier transform, which doesn't seem to be tractable for dependent and nested observations,
and hence it inspires a separate stage one in our algorithm.
Nevertheless, our analysis shows that the separated design already nearly matches the regret lower bound, and hence the benefit of integration is not substantial.
Moreover, because the data collected in stage one is reused in stage two, the major concern of a separated design, i.e., the inefficient use of data, has been addressed.
We leave the design of an integrated algorithm for future research.

\end{itemize}

\bibliography{ref}
\clearpage
\appendix
\renewcommand{\theequation}{\thesection-\arabic{equation}}
\setcounter{equation}{0}

\section{Table of Notations}
\label{sec:appendix-notation}
\begin{table}[ht]
\singlespacing
\small
\begin{tabular}{ll}
\toprule
Notation & Definition\\
\midrule
\\[-2.5ex] \multicolumn{2}{c}{\textbf{Problem Formulation}} \\
$K,~\setk$ &  the number of arms, and the set $\{1,...,K\}$ \\
$T,~\sett$ & the decision horizon, and the set $\{1,...,T\}$ \\
$T_k$ & the period of arm $k$\\
$\hat T_k$ & the estimated period of arm $k$\\
$Y_{k,t}$ & the random reward of arm $k$ at epoch $t$ \\
$\mu_{k,t}$ & the mean reward of arm $k$ at epoch $t$\\
$\epsilon_{t}$ & the independent sub-Gaussian noise at epoch $t$ \\
\( \sigma \) & an upper bound on the sub-Gaussian norm of the noise $\epsilon_{t}$ for $t\in\Tscr$ \\
\( \pi_t \) & the arm chosen at epoch $t$ \\
$\pi_t^*$ & \(\argmax_{k \in \Kscr} \{\mu_{k,t}\}\), the optimal arm at epoch $t$\\
$R_T^\pi$ & $R_{T}^{\pi} \coloneqq \sum_{t=1}^T \left(\max_{k \in \setk}\mu_{k,t}-\mu_{\pi_t,t}\right)$, the pseudo-regret of policy \( \pi \) \\
\midrule
\\[-2.5ex] \multicolumn{2}{c}{\textbf{Algorithm~\ref{alg:frequency-identification} in Stage One}} \\
$n$ & the number of epochs that each arm is pulled during stage one \\
$\tilde{a}(v)$ & the discrete Fourier transform of the sequence $a_1$, \dots, $a_n$\\
\( \tau_k \) & the threshold applied to the periodogram of arm $k$ \\
$\bar \epsilon_v$ & a high-probability upper bound on the noise of the entire periodogram\\
%$B_k,\, b_k$ & \multirow{2}{\textwidth}{\( B_k\defeq \max\{ \abs{b_{k,j}}:k\in\Kscr,0 \leq j \leq  T_k - 1 \}  \) and \( b_k \defeq \min\{ \abs{b_{k,j}}:\abs{b_{k,j}} > 0 \}  \) are the greatest and smallest modulus of the coefficients of arm \( k \)'s  mean reward  \( \mu_{k,t} = \sum_{j = 0}^{T_k-1} b_{k,j} \exp( 2\pi ij t / T_k) \)}  \\
%\\[-2.2ex] & \\ % left blank for the multirow from the above line
$g$  & a parameter defining the width of the neighborhood \\
$A_j,~U_{1},~U_2$ & auxiliary quantities used to control the leakage of the periodogram \\
$\Lambda$ & the event that the lengths of periods of all arms are correctly estimated \\
% $[n]$ & $ \{1, 2, \ldots, n\}$ where $n\in \mathbb{N}_+$ \\
$\textsc{LCM}$ & the least common multiple of a set of positive integers\\
\midrule
\\[-2.5ex] \multicolumn{2}{c}{\textbf{Algorithm~\ref{alg:learning} in Stage Two}} \\
$\hat{d}, ~ d$ &  $\hat{d} \coloneqq \sum_{k = 1}^{K}\hat{T}_k$, $d \coloneqq \sum_{k = 1}^{K}T_k$ \\
$S,~\sets$ & the number of screening rounds, and the set $\{1,...,S\}$ \\
$\bar\Psi$ & the index set of epochs in stage one, $\bar\Psi\coloneqq \{1,...,nK\}$ \\
$\Psi^{(s)}(t)$ & the index sets of trials made in round $s$ up to epoch $t$ \\
$\banditcount{k}{t}(\Psi)$ & the number of epochs within $\Psi$ that arm $k$ has been pulled at the same phase as $t$ \\
$\banditmean_{k,p}^{(s)}$ &  the estimated mean reward of arm $k$ at phase $p$ in round $s$ \\
$w_{k,p}^{(s)}$ & the width of the confidence interval of arm $k$ at phase $p$ in round $s$\\
$\Ascr_s$ & the set of active arms in round $s$ \\
$\Escr$ & the event that the estimated mean achieves a desired accuracy  \\
$\tpi,~\tilde{\cdot}$ & the oracle policy and the associated analogous quantities \\
\bottomrule
\end{tabular}
\vspace{3mm}
\caption{A summary table of notations used in the paper.}
\label{tab:TableOfNotationForMyResearch}
\end{table}

\clearpage

\section{Proofs of Main Results}
\label{sec:appendix-proof}

\subsection{Proofs in Section \ref{sec:phaseone-dft}}
\label{sec:appendix-proof-phaseone}
\begin{proof}[\textbf{Derivation of Equation \eqref{eq:dft-sinform}:}]
To simplify $\tilde{\mu}_{k,j}(v)$ defined in \eqref{eq:dft-deterministic}, we note the following
\begin{align}
\tilde{\mu}_{k,j}(v) ~=~ & \frac{1}{n} \sum_{t=1}^n b_{k,j}\exp\left(2\pi i \left( \frac{j}{T_k}-v \right)t\right)  \nonumber \\
=~ & \frac{b_{k,j}}{n} \exp\left(2\pi i \left( \frac{j}{T_k}-v \right)\right) \sum_{t=0}^{n-1}\exp\left(2\pi i \left( \frac{j}{T_k}-v \right)t\right) \nonumber \\
=~ & \frac{b_{k,j}}{n} \exp\left(2\pi i \left( \frac{j}{T_k}-v \right)\right)\frac{1-\exp\left(2\pi i \left( \frac{j}{T_k}-v \right)n\right)}{1-\exp\left(2\pi i \left( \frac{j}{T_k}-v \right)\right)}. \label{eq:appendix-eq6proof-1}
\end{align}
The numerator of the last term in \eqref{eq:appendix-eq6proof-1} can be rewritten as
\begin{align}
& 1-\exp\left(2\pi i \left( \frac{j}{T_k}-v \right)n\right) \nonumber \\
=~& \exp\left(\pi i \left( \frac{j}{T_k}-v \right)n\right) \left(\exp\left(-\pi i \left( \frac{j}{T_k}-v \right)n\right) -\exp\left(\pi i \left( \frac{j}{T_k}-v \right)n\right) \right) \nonumber \\
=~& - 2i\exp\left(\pi i \left( \frac{j}{T_k}-v \right)n\right) \sin\left( \pi  \left( \frac{j}{T_k}-v \right)  n\right). \nonumber
\end{align}
Similarly, the denominator of the last term in \eqref{eq:appendix-eq6proof-1} can be rewritten as
\begin{align}
1-\exp\left(2\pi i \left( \frac{j}{T_k}-v \right)\right) = - 2i\exp\left(\pi i \left( \frac{j}{T_k}-v \right)\right) \sin\left( \pi  \left( \frac{j}{T_k}-v \right) \right). \nonumber
\end{align}

Therefore, we obtain Equation \eqref{eq:dft-sinform} by simplifying \eqref{eq:appendix-eq6proof-1}
\begin{align}
\tilde{\mu}_{k,j}(v) ~=~ & \frac{b_{k,j}}{n} \exp\left(2\pi i \left( \frac{j}{T_k}-v \right)\right)  \frac{\exp\left(\pi i \left( \frac{j}{T_k}-v \right)n\right) \cdotp  \sin\left( \pi  \left( \frac{j}{T_k}-v \right)  n\right)}{ \exp\left(\pi i \left( \frac{j}{T_k}-v \right)\right) \cdotp \sin\left( \pi  \left( \frac{j}{T_k}-v \right)  \right)} \nonumber \\
~=~ & \frac{b_{k,j}}{n} \exp\left(2\pi i \left( \frac{j}{T_k}-v \right)  \frac{n+1}{2}  \right)  \frac{  \sin\left( \pi  \left( \frac{j}{T_k}-v \right)  n\right)}{ \sin\left( \pi  \left( \frac{j}{T_k}-v \right)  \right)},\nonumber
\end{align}
and we can immediately get $\displaystyle{\lim_{n\to\infty} |\tilde{\mu}_{k,j}(v)| = 0}$.
\end{proof}

\subsection{Proofs in Section~\ref{sec:analysis-phaseone}}
\label{sec:appendix-proof-analysis-phaseone}
The main techniques used in the following proofs are concentration inequalities for random variables. We start with a technical lemma (Theorem 1 in \citet{nagy2012}) which will be used in the proof of Lemma \ref{lem:noise-bound}.
\begin{lemma} \label{lem:berstein_generalized}
Suppose $\omega \in (0,\pi]$ and $p_n$ is a polynomial of complex number with degree no more than $n$, then the following holds for all $\theta \in (-\omega,+\omega)$,
\begin{align}
\left|p_n^\prime(e^{i\theta})\right| \le \frac{n}{2}\left(1+\frac{\sqrt{2}\cos(\theta/2)}{\sqrt{ \cos(\theta)-\cos(\omega) }}\right)  \sup_{\phi \in [-\omega, \omega]} \left|p_n(e^{i\phi})\right|. \nonumber
\end{align}
In particular, when $\omega=\pi$, then ${\sup_{|z| \le 1} |p_n^\prime(z)| \le n\sup_{|z| \le 1} |p_n(z)|}$.
\end{lemma}

\begin{proof}[\textbf{Proof of Lemma \ref{lem:noise-bound}}:]
The DFT of the noise can be expressed as
\begin{align}
\tilde{\epsilon}_k(v)=\frac{1}{n}\sum_{t=1}^{n}\epsilon_t\exp\left( -2\pi i v t\right) = \frac{1}{n}\sum_{t=1}^{n}\epsilon_t\cos\left( 2\pi v t\right)-\frac{i}{n}\sum_{t=1}^{n}\epsilon_t\sin\left( 2\pi v t\right). \nonumber
\end{align}
For any $\delta>0$, by the union bound, we have
\begin{align}\label{eq:sin-cos-decomp}
\PR\left( |\tilde{\epsilon}_k(v)| > \delta \right) & ~=~ \PR\left( |\tilde{\epsilon}_k(v)|^2 > \delta^2 \right)\notag\\
& ~\leq~  \PR\left( \left|  \frac{1}{n}\sum_{t=1}^{n}\epsilon_t\cos\left( 2\pi v t\right) \right|^2> \frac12 \delta^2 \right) + \PR\left( \left|  \frac{1}{n}\sum_{t=1}^{n}\epsilon_t\sin\left( 2\pi v t\right) \right|^2> \frac12 \delta^2 \right).
\end{align}
We analyze the first term of \eqref{eq:sin-cos-decomp} in the following and the second term can be done analogously. For a given $v$, $\sum_{t=1}^{n}\epsilon_t\cos\left( 2\pi v t\right)$ is the sum of independent sub-Gaussian random variables with parameter $\sigma^2$, and thus $\sum_{t=1}^{n}\epsilon_t\cos\left( 2\pi v t\right)$ is a sub-Gaussian random variable with parameter $\sum_{t=1}^{n}\sigma^2\cos^2\left( 2\pi v t\right)$. Hence, the property of sub-Gaussian random variable in Assumption \ref{asp:sub-Gaussian} gives
\begin{align}
\PR\left( \left|  \frac{1}{n}\sum_{t=1}^{n}\epsilon_t\cos\left( 2\pi v t\right) \right|^2> \frac12 \delta^2 \right) \le 2\exp\left(  -\frac{n^2\delta^2}{4\sigma^2 \sum_{t=1}^n\cos^2\left( 2\pi v t\right)}  \right) \leq 2\exp\left( -\frac{n\delta^2}{4\sigma^2 }  \right). \nonumber
\end{align}
By combining the bounds for the two terms in \eqref{eq:sin-cos-decomp}, for a given $v \in [0,1/2]$, we have
\begin{align}\label{eq:bound-noise-v}
\PR\left( |\tilde{\epsilon}_k(v)| > \delta \right) \leq 4\exp\left( -\frac{n\delta^2}{4\sigma^2 } \right).
\end{align}

Note that $\tilde{\epsilon}_k(v)$ can be viewed as an $n$-degree polynomial of $\exp(-2\pi i v)$, and let $p_n$ denote this polynomial. For any $v_1, v_2\in[0,1/2]$, we have
\begin{align}
\abs{\tilde{\epsilon}_k(v_1)-\tilde{\epsilon}_k(v_2)} &~=~ \abs{p_n(\exp(-2\pi i v_1))-p_n(\exp(-2\pi i v_2))} \nonumber\\
&~=~ \left|\int_{\exp(-2\pi i v_2)}^{\exp(-2\pi i v_1)} p^\prime_n(z) dz\right| \nonumber\\
& ~\le~ \abs{\exp\left(-2\pi i v_1\right)-\exp(-2\pi i v_2)} \sup_{\abs{z}\leq1}\abs{p_n^\prime(z)}  \nonumber\\
& ~\le~ \abs{\exp\left(-2\pi i v_1\right)-\exp(-2\pi i v_2)}n \sup_{\abs{z}\leq1}\abs{p_n(z)}, \nonumber
\end{align}
where the last step follows from Lemma~\ref{lem:berstein_generalized}.
The maximum modulus principle of analytic functions implies that $\sup_{\abs{z}\leq1}\abs{p_n(z)}$ must be attained at the boundary $\{z:\abs{z}=1\}$ as $\sup_{\abs{z}\leq1}\abs{p_n(z)}=\sup_{\abs{z}=1}\abs{p_n(z)}=\sup_{v\in[0,1]}\abs{\tilde{\epsilon}_k(v)}$. Moreover, since $|\tilde{\epsilon}_k(v)| = |\tilde{\epsilon}_k(1-v)|$, we have
\begin{align}
\abs{\tilde{\epsilon}_k(v_1)-\tilde{\epsilon}_k(v_2)} & ~\leq~ \abs{\exp\left(-2\pi i v_1\right)-\exp(-2\pi i v_2)} n \NB \nonumber \\
& ~\leq~ 2\abs{\sin(\pi (v_1-v_2))}n \NB \nonumber \\
& ~\leq~ 2n\pi \abs{v_1-v_2} \NB .\label{eq:derivative-noise}
\end{align}

Suppose that we divide the frequency domain $[0,1/2]$ into $L\in \N^+$ equal intervals, and we let $v_{l}^{\text{mid}}$ denote the middle point of the interval $[\frac{l-1}{2L}, \frac{l}{2L}]$ for $l=1,...,L$. We also suppose that $\NB$ is attained at $v_{\epsilon}^*$ which locates in the interval $[\frac{l^*-1}{2L}, \frac{l^*}{2L}]$, and thus we have $\left|v_{\epsilon}^*-v_{l^*}^{\text{mid}}\right| < \frac{1}{4L}$. By applying the inequality \eqref{eq:derivative-noise}, we find the following
\begin{align}
& |\tilde{\epsilon}_k(v_{\epsilon}^*)| - |\tilde{\epsilon}_k(v_{l^*}^{\text{mid}})| \le \left| \tilde{\epsilon}_k(v_{\epsilon}^*-v_{l^*}^{\text{mid}}) \right| \le \frac{n\pi}{2L} \NB \nonumber \\
\implies & \NB = |\tilde{\epsilon}_k(v_{\epsilon}^*)| \le \max_{l=1,...,L}|\tilde{\epsilon}_k(v_{l}^{\text{mid}})| + \frac{n\pi}{2L} \NB \nonumber \\
\implies & \NB \leq \left(1-\frac{n\pi}{2L}\right)^{-1} \max_{l=1,...,L}|\tilde{\epsilon}_k(v_{l}^{\text{mid}})|.  \nonumber
\end{align}
Then, we apply the result from \eqref{eq:bound-noise-v} and the union bound to obtain
\begin{align}
\PR\left(\NB \ge \delta \right) &\leq \PR\left(\max_{l=1,...,L}|\tilde{\epsilon}_k(v_{l}^{\text{mid}})| \ge \left( 1-\frac{n\pi}{2L} \right)\delta \right) \nonumber \\
&\leq 4L \exp\left(- \frac{n\delta^2}{4\sigma^2}\left( 1-\frac{\pi n}{2L} \right)^2 \right) \label{eq:bound-noise-L}
\end{align}

The lemma is proved by choosing $L=12n$ as following
\begin{align}
\PR\left(   \sup_{v\in[0,1/2]} |\tilde{\epsilon}_k(v)|  \geq \delta \right) ~\leq~ 48n \exp\left(- \frac{n\delta^2}{4\sigma^2}\left( 1-\frac{\pi }{24} \right)^2 \right). \nonumber
\end{align}
Specifically, if we choose the upper bound $\delta$ as $\bar\epsilon_v = \frac{2\sigma H}{1-\pi/24}\sqrt{\frac{\log(n)}{n}}$ defined in Step \ref{step:threshold} of Algorithm \ref{alg:frequency-identification}, then we have ${\PR \bigg(\NB  \geq \bar\epsilon_v  \bigg) \leq \frac{48}{n^{H^2-1}}}$.
\end{proof}

\begin{proof}[\textbf{Proof of Lemma \ref{lem:bound-leakage}}:]
Recall that we defined $\Vscr_k = [0,1/2] \cap \{\cup_{j:|b_{k,j}|>0}[\frac{j}{T_k}-\frac{g}{n}, \frac{j}{T_k}+\frac{g}{n}]\}$ and $\overline{\Vscr}_k = [0,1/2]\setminus \Vscr_k$, where $\Vscr_k$ represents the union of neighborhoods of all present frequencies.
To prove Lemma \ref{lem:bound-leakage}, it is more natural to start from investigating the leakage in sets $\Uscr_k \coloneqq [0,1/2] \cap \{\cup_{j:}[\frac{j}{T_k}-\frac{g}{n}, \frac{j}{T_k}+\frac{g}{n}]\}$ and $\overline{\Uscr}_k = [0,1/2]\setminus \Uscr_k$. Note that $\Uscr_k$ differs from $\Vscr_k$ by removing the requirement $|b_{k,j}|>0$, so $\Uscr_k$ represents the union of neighborhoods around \emph{all} possible locations of present frequencies and $\Vscr_k$ is a subset of $\Uscr_k$, i.e., $\Vscr_k \subseteq \Uscr_k$.

Given the expression of $\tilde{\mu}_{k,j}(v)$ in \eqref{eq:dft-sinform} and the definition of $B_{k}$ in \eqref{eq:def-Bk}, we can show
\begin{align}\label{eq:Rtheta-rep}
\abs{\tilde{\mu}_{k,j}(v)} \le |b_{k,j}| \left| \frac{\sin\left(\pi (j/T_k-v)n\right)}{n\sin\left(\pi(j/T_k-v)\right)}\right| \le B_k \left| \frac{\sin\left(\pi (j/T_k-v)n\right)}{n\sin\left(\pi(j/T_k-v)\right)}\right|.
\end{align}
Define a function $\displaystyle{R(\theta) \coloneqq \left|\frac{\sin\left(\pi n\theta\right)}{n\sin\left(\pi\theta\right)}\right|}$ and then we have \eqref{eq:reduce-to-R}. Note that $R(\theta)$ will play a critical role in this proof.
\begin{equation}\label{eq:reduce-to-R}
|\tilde{\mu}_k(v)| \le \sum_{j=0}^{T_k-1}|\tilde{\mu}_{k,j}(v)|\le  B_k \sum_{j=0}^{T_k-1}R(j/T_k-v).
\end{equation}

We first examine the leakage for $v\in\bar\Uscr_k$. Suppose that $v \in [\frac{j'}{T_k}+\frac{g}{n}, \frac{j'+1}{T_k}-\frac{g}{n}]$ for some $j'\in\left\{0,\dots,\lfloor \frac{T_k-1}{2}\rfloor\right\}$. We decompose the sum $\sum_{j=0}^{T_k-1}R(j/T_k-v)$ in \eqref{eq:reduce-to-R} into three terms as following, and we bound each term using properties $R(\theta)=R(-\theta)$ and $R(\theta)=R(1-\theta)$.
\begin{align}
\sum_{j=0}^{j'}R(j/T_k-v) = \sum_{j=0}^{j'}R(v-j/T_k) \le \sum_{j=0}^{j'}\sup\left\{ R(\theta):\theta\in\left[\frac{j}{T_k}+\frac{g}{n}, \frac{j+1}{T_k}-\frac{g}{n}\right] \right\}. \label{eq:leakage-bound-p1}
\end{align}
\begin{align}
\sum_{j=j'+1}^{j'+\lfloor\frac{T_k-1}{2}\rfloor+1}R(j/T_k-v) \le \sum_{j=0}^{\lfloor\frac{T_k-1}{2}\rfloor}\sup\left\{ R(\theta):\theta\in\left[\frac{j}{T_k}+\frac{g}{n}, \frac{j+1}{T_k}-\frac{g}{n}\right] \right\}. \label{eq:leakage-bound-p2}
\end{align}
\begin{align}
\sum_{j=j'+\lfloor\frac{T_k-1}{2}\rfloor+2}^{T_k-1}R(j/T_k-v) &~\le~ \sum_{\lfloor\frac{T_k-1}{2}\rfloor+1}^{T_k-j'-2}\sup\left\{ R(\theta):\theta\in\left[\frac{j}{T_k}+\frac{g}{n}, \frac{j+1}{T_k}-\frac{g}{n}\right] \right\} \nonumber \\
&~=~ \sum_{\lfloor\frac{T_k-1}{2}\rfloor+1}^{T_k-j'-2}\sup\left\{ R(1-\theta):\theta\in\left[\frac{j}{T_k}+\frac{g}{n}, \frac{j+1}{T_k}-\frac{g}{n}\right] \right\} \nonumber \\
&~=~ \sum_{j=j'+1}^{T_k-\lfloor\frac{T_k-1}{2}\rfloor-2} \sup\left\{ R(\theta):\theta\in\left[\frac{j}{T_k}+\frac{g}{n}, \frac{j+1}{T_k}-\frac{g}{n}\right] \right\}. \label{eq:leakage-bound-p3}
\end{align}
Combining \eqref{eq:leakage-bound-p1}, \eqref{eq:leakage-bound-p2} and \eqref{eq:leakage-bound-p3} together, we can further develop \eqref{eq:reduce-to-R} as
\begin{align}
|\tilde{\mu}_k(v)| \le B_k \sum_{j=0}^{T_k-1}R(j/T_k-v) \le 2B_k \sum_{j=0}^{\lfloor\frac{T_k-1}{2}\rfloor}\sup\left\{ R(\theta):\theta\in\left[\frac{j}{T_k}+\frac{g}{n}, \frac{j+1}{T_k}-\frac{g}{n}\right] \right\}. \label{eq:reduce-to-R-version1}
\end{align}

To analyze the expression \eqref{eq:reduce-to-R-version1}, we need to examine $R(\theta)$ closely. Note that $R(\theta)$ is bounded. Furthermore, the numerator $|\sin\left(\pi n\theta\right)|$ has a period $\frac{1}{n}$, and the denominator $|n\sin(\pi\theta)|$ is monotonically  increasing for $\theta \in [0, 1/2]$.
Therefore, we can make the following two remarks. (i) Given $0 \le \theta_1 < \theta_2$ and $\theta_2 + \frac{1}{n} \le \frac{1}{2}$, then $\sup\left\{R(\theta): \theta \in [\theta_1, \theta_1+\frac{1}{n}]\right\} > \sup\left\{R(\theta): \theta \in [\theta_2, \theta_2+\frac{1}{n}]\right\}$, i.e, for two intervals of the same width $\frac{1}{n}$ in the domain $[0,1/2]$, the maximum of $R(\theta)$ in the left interval is larger than that in the right interval. (ii) Given $0 \le \theta_1 < \theta_2 \le \frac{1}{2}$, then $\sup\left\{R(\theta): \theta \in [\theta_1, \theta_1+\frac{1}{n}]\right\} \ge \sup\left\{R(\theta): \theta \in [\theta_1, \theta_2]\right\}$.
We continue the analysis by applying these two properties of $R(\theta)$.

When $T_k$ is even, for $j = 0,...,\frac{T_k}{2}-1$, we have
\begin{align}
\label{eq:rtheta-bound-1}
\begin{split}
\sup\left\{ R(\theta):\theta\in\left[\frac{j}{T_k}+\frac{g}{n}, \frac{j+1}{T_k}-\frac{g}{n}\right] \right\} &~\leq~ \sup\left\{ R(\theta):\theta\in\left[\frac{j}{T_k}+\frac{g}{n}, \frac{j}{T_k}+\frac{g+1}n\right] \right\} \\
&~\leq~ \sup\left\{ R(\theta):\theta\in \left[\frac{(2j+1)g}{n}, \frac{(2j+1)g+1}{n}\right] \right\}
\end{split}
\end{align}
where the second step follows from that $\frac{j}{T_k}+\frac{g}{n} \ge \frac{(2j+1)g}{n}$ since $T_k < \frac{n}{2g}$ given by Assumption \ref{asp:period}. We also have $\frac{T_k}{2}-1 \le \lfloor \frac{n-1}{4g}-1\rfloor = \lfloor\frac{n-4g-1}{4g}\rfloor$.

When $T_k$ is odd, inequality \eqref{eq:rtheta-bound-1} holds directly for $j=0,...,\frac{T_k-1}{2}-1$. For $j=\frac{T_k-1}{2}$, although the corresponding interval $\left[\frac{T_k-1}{2T_k}+\frac{g}{n}, \frac{T_k+1}{2T_k}-\frac{g}{n}\right]$ is not fully contained in $[0,1/2]$, inequality \eqref{eq:rtheta-bound-1} still holds since $R(\theta)$ is symmetric around $\theta=1/2$. We also have  $\frac{T_k-1}{2} \le \lfloor\frac{1}{2} \big(\frac{n-1}{2g} -1\big)\rfloor=\lfloor\frac{n-2g-1}{4g}\rfloor$.

Combing these two cases, we can further derive \eqref{eq:reduce-to-R-version1} as
\begin{align}
|\tilde{\mu}_k(v)| \le 2B_k\sum_{j=0}^{\lfloor\frac{(n-2g-1)}{4g}\rfloor}\sup\left\{ R(\theta):\theta\in\left[\frac{(2j+1)g}{n}, \frac{(2j+1)g+1}{n}\right] \right\}. \nonumber
\end{align}
Recall that we defined ${A_j = \sup \left\{ \frac{|\sin(\pi \nu)|}{\pi \nu}: \nu \in [j,j+1] \right\} }$ and  ${U_1 = \sum_{j=0}^{\lfloor{\frac{n-2g -1}{4g }\rfloor}} A_{(2j+1)g}}$, and note the fact that $\sin(x)\ge 2x/\pi$ for $x\in [0,\pi/2]$. Then, we can bound the leakage for $v\in\bar\Uscr_k$ as
\begin{align}
|\tilde{\mu}_k(v)| &\le \pi B_k\sum_{j=0}^{\lfloor\frac{(n-2g-1)}{4g}\rfloor}\sup\left\{  \frac{|\sin(\pi n\theta)|}{\pi n\theta}:\theta\in\left[\frac{(2j+1)g}{n}, \frac{(2j+1)g+1}{n}\right] \right\} \nonumber \\
& = \pi B_k\sum_{j=0}^{\lfloor{\frac{n-2g-1}{4g}\rfloor}} A_{(2j+1)g} = \pi B_k U_1. \nonumber
\end{align}

The analysis of the leakage in $\Uscr_k$ can be conducted in an analogous way as above. In addition to further exploiting the properties of $R(\theta)$, we also utilizing the symmetry of the neighborhood to achieve a finer result.
Suppose that $v \in [\frac{j'}{T_k}-\frac{g}{n}, \frac{j'}{T_k}+\frac{g}{n}]$ for some $j'\in\left\{0,\dots,\lfloor \frac{T_k}{2}\rfloor\right\}$. Let $v' = v - \frac{j'}{T_k}$ and then $v' \in [-\frac{g}{n}, \frac{g}{n}]$. We also assume that $v'$ falls in the interval $[\frac{j_0}{n}, \frac{j_0+1}{n}]$ for some $j_0 \in \{-g,...,g-1\}$. The leakage in the neighborhood of $j'/T_k$ is contributed by frequency components $j/T_k$ with $j\neq j'$. Hence, we decompose $\sum_{j \neq j'}R(j/T_k-v)$ into three terms as following,
\begin{align}
\sum_{j=0}^{j'-1}R(j/T_k-v) = \sum_{j=0}^{j'-1}R(v-j/T_k) \le \sum_{j=1}^{j'} \sup\left\{ R(\theta):\theta\in\left[\frac{j}{T_k}+\frac{j_0}{n}, \frac{j}{T_k}+\frac{j_0+1}{n}\right] \right\}. \label{eq:leakage-bound-p4}
\end{align}
\begin{align}
\sum_{j=j'+1}^{j'+\lfloor\frac{T_k}{2}\rfloor}R(j/T_k-v) ~\le~& \sum_{j=1}^{\lfloor\frac{T_k}{2}\rfloor} \sup\left\{ R(\theta):\theta\in\left[\frac{j}{T_k}-\frac{j_0+1}{n}, \frac{j}{T_k}-\frac{j_0}{n}\right] \right\} \nonumber  \\
~\le~& \sum_{j=1}^{\lfloor\frac{T_k}{2}\rfloor} \sup\left\{ R(\theta):\theta\in\left[\frac{2jg-j_0-1}{n}, \frac{2jg-j_0}{n}\right] \right\} \le \frac{\pi}{2} \sum_{j=1}^{\lfloor \frac{n-1}{4g} \rfloor}A_{2jg-j_0-1}. \label{eq:leakage-bound-p5}
\end{align}
\begin{align}
\sum_{j=j'+\lfloor\frac{T_k}{2}\rfloor+1}^{T_k-1}R(j/T_k-v) &~\le~ \sum_{\lfloor\frac{T_k}{2}\rfloor+1}^{T_k-j'-1} \sup\left\{ R(\theta):\theta\in\left[\frac{j}{T_k}-\frac{j_0+1}{n}, \frac{j}{T_k}-\frac{j_0}{n}\right] \right\} \nonumber \\
&~=~ \sum_{\lfloor\frac{T_k}{2}\rfloor+1}^{T_k-j'-1} \sup\left\{ R(1-\theta):\theta\in\left[\frac{j}{T_k}-\frac{j_0+1}{n}, \frac{j}{T_k}-\frac{j_0}{n}\right] \right\} \nonumber \\
&~=~ \sum_{j=j'+1}^{T_k-\lfloor\frac{T_k}{2}\rfloor-1} \sup\left\{ R(\theta):\theta\in\left[\frac{j}{T_k}+\frac{j_0}{n}, \frac{j}{T_k}+\frac{j_0+1}{n}\right] \right\}. \label{eq:leakage-bound-p6}
\end{align}
Combining \eqref{eq:leakage-bound-p4} and \eqref{eq:leakage-bound-p6} together, we have
\begin{align}
& \sum_{j=0}^{j'-1}R(j/T_k-v) + \sum_{j=j'+\lfloor\frac{T_k}{2}\rfloor+1}^{T_k-1}R(j/T_k-v) \nonumber \\
~\le~& \sum_{j=1}^{\lfloor\frac{T_k}{2}\rfloor} \sup\left\{ R(\theta):\theta\in\left[\frac{j}{T_k}+\frac{j_0}{n}, \frac{j}{T_k}+\frac{j_0+1}{n}\right] \right\} \nonumber \\
~\le~& \sum_{j=1}^{\lfloor\frac{T_k}{2}\rfloor} \sup\left\{ R(\theta):\theta\in\left[\frac{2jg+j_0}{n}, \frac{2jg+j_0+1}{n}\right] \right\} \le \frac{\pi}{2} \sum_{j=1}^{\lfloor \frac{n-1}{4g} \rfloor}A_{2jg+j_0}. \label{eq:leakage-bound-p4andp6}
\end{align}

Note that $A_j$ is monotonically decreasing in $j$ such that  $\frac{1}{(j+\frac{1}{2})\pi} \le A_j \le \frac{1}{j\pi}$.
Moreover, $A_{j-1}+A_{j+1} \ge 2A_j$ for all $j \ge 2$. Recall that we defined ${U_2 = \sum_{j=1}^{\lfloor{\frac{n-1}{4g }\rfloor}} A_{2jg-1}}$.
We merge \eqref{eq:leakage-bound-p5} and \eqref{eq:leakage-bound-p4andp6} together, apply the monotonicity and the convexity of $A_j$, and then we obtain
\begin{align}
\sum_{j \neq j'}R(j/T_k-v) \le \frac{\pi}{2} \sum_{j=1}^{\lfloor \frac{n-1}{4g} \rfloor}A_{2jg-j_0-1} + \frac{\pi}{2} \sum_{j=1}^{\lfloor \frac{n-1}{4g} \rfloor}A_{2jg+j_0} \le \pi \sum_{j=1}^{\lfloor \frac{n-1}{4g} \rfloor}A_{2jg-1} = \pi U_2, \nonumber
\end{align}
which leads to  $|\sum_{j\not=j'}\tilde{\mu}_{k,j}(v)| \le \sum_{j\not=j'}|\tilde{\mu}_{k,j}(v)| \le B_k\sum_{j \neq j'} R(v-j/T_k) \le \pi B_k U_2$.

We show that $U_1 \ge U_2$ by revisiting their definitions as following
\begin{align*}
U_1 &= \sum_{j=0}^{\lfloor{\frac{n-2g-1}{4g}\rfloor}} A_{(2j+1)g} \ge \sum_{j=0}^{\lfloor{\frac{n-4g-1}{4g}\rfloor}} A_{(2j+1)g}  = \sum_{j=1}^{\lfloor{\frac{n-1}{4g}\rfloor}} A_{(2j-1)g} \ge \sum_{j=1}^{\lfloor{\frac{n-1}{4g}\rfloor}} A_{2jg-1} = U_2.
\end{align*}
The last inequality is due to $(2j-1)g<2jg-1$ given $g \ge 2$ in Assumption \ref{asp:period}.

So far, we have proven
\begin{align}
\label{eq:bound-leakage-Uk}
\sup_{v\in\bar\Uscr_k}\left|\sum_{j=0}^{T_{k}-1}\tilde{\mu}_{k,j}(v)\right|  \le \pi B_k U_1 \text{ and }
\sup_{v\in{\Uscr}_k}\left|\sum_{j:|\frac{j}{T_k}-v|\geq \frac{g}{n}}\tilde{\mu}_{k,j}(v)\right|  \le \pi B_k U_2 \le \pi B_k U_1.
\end{align}
Since $\Vscr_k\subseteq\Uscr_k$, we immediately have $\sup_{v\in{\Vscr}_k}\left|\sum_{j:|\frac{j}{T_k}-v|\geq \frac{g(n)}{n}}\tilde{\mu}_{k,j}(v)\right|  \le \pi B_k U_2$.
For $v\in\bar\Vscr_k$, we can bound the leakage in $\bar\Vscr_k \bigcap \bar\Uscr_k$ and $\bar\Vscr_k \bigcap \Uscr_k$ separately, i.e.,  $\sup_{v\in \bar\Vscr_k \bigcap \bar\Uscr_k}\left|\sum_{j=0}^{T_{k}-1}\tilde{\mu}_{k,j}(v)\right|  \le \pi B_k U_1$ and $\sup_{v\in \bar\Vscr_k \bigcap  \Uscr_k}\left|\sum_{j:|\frac{j}{T_k}-v|\geq \frac{g}{n}}\tilde{\mu}_{k,j}(v)\right|  \le \pi B_k U_2 \le \pi B_k U_1$. Therefore, $\sup_{v\in\bar\Vscr_k}\left|\sum_{j=0}^{T_{k}-1}\tilde{\mu}_{k,j}(v)\right|  \le \pi B_k U_1$. The lemma follows.
\end{proof}

\begin{proof}[\textbf{Proof of Lemma \ref{lem:Bound-B-Above}}:]
Suppose that $B_k$ is attained at $j'/T_k$. We have
\begin{align}
\tilde{y}_{k}(j'/T_k) &= \tilde{\mu}_{k,j'}(j'/T_k) + \sum_{j\not= j'} \tilde{\mu}_{k,j}(j'/T_k) + \tilde{\epsilon}_k(j'/T_k) \nonumber \\
\implies \quad  B_k = |\tilde{\mu}_{k,j'}(j'/T_k)| &\le |\tilde{y}_{k}(j'/T_k)| + \bigg|\sum_{j\not= j'} \tilde{\mu}_{k,j}(j'/T_k)\bigg| + |\tilde{\epsilon}_k(j'/T_k)| \nonumber \\
&\leq  \PB + \pi B_k U_2 + \NB, \nonumber
\end{align}
where the last inequality follows from Lemma \ref{lem:bound-leakage}. Therefore, the lemma follows.
\end{proof}

\begin{proof}[\textbf{Proof of Lemma \ref{lem:threshold}}:]
We can show the following holds with probability at least $\ProbNB$,
\begin{align}
\sup_{v\in {\overline{\Vscr}_k}}|\tilde{y}_k(v)| &~\le~ \sup_{v\in {\overline{\Vscr}_k}}|\tilde{\epsilon}_k(v)|+\sup_{v\in {\overline{\Vscr}_k}}|\tilde{\mu}_k(v)| \nonumber\\
&~\le~ \sup_{v\in {\overline{\Vscr}_k}}|\tilde{\epsilon}_k(v)|+\pi B_k U_1 \nonumber \\
&~\le~ \sup_{v\in {\overline{\Vscr}_k}}|\tilde{\epsilon}_k(v)|+ \frac{\pi U_1}{1-\pi U_2}\left(\NB + \PB \right) \nonumber\\
&~\le~ \bar\epsilon_v  +\frac{\pi U_1}{1-\pi U_2} \left( \bar\epsilon_v + \PB \right) =: \tau_{k}.  \label{eq:derive the threshold}
\end{align}
The second inequality follows from Lemma~\ref{lem:bound-leakage},
the third inequality follows from Lemma~\ref{lem:Bound-B-Above},
and the last inequality holds with probability at least $\ProbNB$ by Lemma~\ref{lem:noise-bound} and the definition of $\bar\epsilon_v$.
%Note that \eqref{eq:derive the threshold} is exactly the threshold $\tau_k$ used in Algorithm \ref{alg:frequency-identification}, and this inequality shows that $\tau_k$ is able to filter out $v \in \overline{\Vscr}_k$ with high probability.
\end{proof}

\begin{proof}[\textbf{Proof of Lemma \ref{lem:no-false-negative}}:]
The periodogram at a present frequency $j'/T_k$ can be lower bounded by
\begin{align}
\left|\tilde{y}_{k}(j'/T_k)\right| &= \bigg|\tilde{\mu}_{k,j'}(j'/T_k) + \sum_{j\not= j'} \tilde{\mu}_{k,j}(j'/T_k) + \tilde{\epsilon}_k(j'/T_k)\bigg| \nonumber \\
&\ge |\tilde{\mu}_{k,j'}(j'/T_k)| - \bigg|\sum_{j\not= j'} \tilde{\mu}_{k,j}(j'/T_k)\bigg| - \NB \nonumber \\
&\ge b_{k} -\pi B_k U_2 - \bar\epsilon_v, \label{eq:decomp-y}
\end{align}
with probability no less than $\ProbNB$ according to definition of $b_k$ in \eqref{eq:def-bk} and Lemma \ref{lem:noise-bound} and \ref{lem:bound-leakage}.

Next we examine the periodogram in set $\Vscr_k$ and $\overline{\Vscr}_k$ respectively.
We find the following
\begin{align}
& \sup_{v\in \Vscr_k}|\tilde{y}_k(v)| \leq |\tilde{\mu}_{k,j'}(v)| + \bigg|\sum_{j\neq j'}\tilde{\mu}_{k,j}(v)\bigg| +  \NB \leq B_k + \pi B_k U_2 + \bar\epsilon_v, \nonumber \\
& \sup_{v\in \overline\Vscr_k}|\tilde{y}_k(v)| \leq \bigg|\sum_{j=0}^{T_k-1}\tilde{\mu}_{k,j}(v)\bigg| + \NB \leq \pi B_k U_1 + \bar\epsilon_v, \nonumber
\end{align}
hold with probability no less than $\ProbNB$ by using the definition of $B_k$ in \eqref{eq:def-Bk} and applying Lemma \ref{lem:noise-bound} and \ref{lem:bound-leakage} again.
Hence, with that probability as well, we have
\begin{align}
\PB \le \max\bigg\{\sup_{v\in \Vscr_k}|\tilde{y}_k(v)|, \sup_{v\in \overline\Vscr_k}|\tilde{y}_k(v)|\bigg\} \le \max\left\{\pi U_1,\pi U_2+1\right\}B_k + \bar\epsilon_v . \label{eq:upper-bound-yk}
\end{align}
By plugging \eqref{eq:upper-bound-yk} into the definition of $\tau_k$ in \eqref{eq:derive the threshold}, we can derive an upper bound on $\tau_k$ as
\begin{align}\label{eq:upper-bound-tau}
\tau_k \le \left(\frac{2\pi U_1}{1 - \pi U_2}+1\right)\bar\epsilon_v + \frac{\pi U_1}{1-\pi U_2}\max\left\{\pi U_1,\pi U_2+1\right\}B_k.
\end{align}

Note that condition \eqref{eq:guarantee} specified in the lemma leads to the following
\begin{align}
b_k -\pi B_k U_2 - \bar\epsilon_v \geq \left(\frac{2\pi U_1}{1 - \pi U_2}+1\right)\bar\epsilon_v + \frac{\pi U_1}{1-\pi U_2}\max\left\{\pi U_1, \pi U_2+1\right\}B_k. \nonumber
\end{align}
Therefore, by comparing \eqref{eq:decomp-y} and \eqref{eq:upper-bound-tau}, we can show that $|\tilde{y}_k(j'/T_k)| > \tau_k$ holds for all present frequencies with  probability no less than $\ProbNB$.
\end{proof}

\begin{proof}[\textbf{Proof of Lemma \ref{lem:local-noise-bound}}:]
The following analysis relies on certain results developed in the proof of Lemma \ref{lem:noise-bound}.
Suppose that we divide each neighborhood $[\frac{j}{T_k}-\frac{g}{n},\frac{j}{T_k}+\frac{g}{n}]$ for $j=1,...,\lfloor\frac{T_k}{2}\rfloor$ into $L_1 \in \N^+$ equal intervals, and thus there are totally $\lfloor\frac{T_k}{2}\rfloor L_1$ intervals. We let $v_1$ denote where $\NBtight$ is attained, let $v_2$ denote the midpoint of the interval where $v_1$ falls in, and let $\Gscr_k$ denote the set of the mid-points of all the intervals.
Since $v_1$ is at most $\frac{g}{nL_1}$ away from $v_2$, we can show the following according to the inequality \eqref{eq:derivative-noise},
\begin{align}
\NBtight = |\tilde{\epsilon}_k(v_1)| &~\le~ |\tilde{\epsilon}_k(v_2)| +  2n\pi|v_1-v_2| \NB  \nonumber \\
&~\le~ \max_{v \in \Gscr_k}|\tilde{\epsilon}_k(v)| +  \frac{2\pi g}{L_1} \NB. \nonumber
\end{align}

For any $w_1, w_2\in(0,1)$ satisfying $w_1+w_2=1$, by the union bound, we have
\begin{align}
\PR\bigg( \NBtight >\delta \bigg) &~\leq~ \PR\bigg(\max_{v \in \Gscr_k}|\tilde{\epsilon}_k(v)| > w_1\delta \bigg) + \PR\bigg(\frac{2\pi g}{L_1} \NB > w_2\delta  \bigg). \label{eq:lemma5-proof1}
\end{align}
By using the union bound again together with the inequality \eqref{eq:bound-noise-v}, we can show
\begin{align}
\PR\bigg(\max_{v \in \Gscr_k}|\tilde{\epsilon}_k(v)| > w_1\delta \bigg) \leq 4\lfloor\frac{T_k}{2}\rfloor L_1 \exp\left(  -\frac{n\delta^2w_1^2}{4\sigma^2}  \right) \leq \frac{nL_1}{g} \exp\left(  -\frac{n\delta^2w_1^2}{4\sigma^2}  \right)\label{eq:lemma5-proof2}
\end{align}
given $T_k<\frac{n}{2g}$ from Assumption \ref{asp:period}. By applying inequality \eqref{eq:bound-noise-L} for any $L_2 \in \N^+$, we have
\begin{align}
\PR\bigg(\frac{2\pi g}{L_1} \NB > w_2\delta  \bigg) \le 4L_2\exp\left( -\frac{n\delta^2w_2^2L_1^2}{16\pi^2g^2\sigma^2}\left( 1-\frac{\pi n}{2L_2}  \right)^2  \right). \label{eq:lemma5-proof3}
\end{align}

We calibrate parameters and choose $L_1=200g$, $L_2=50n$, $w_1 = 0.965$, and $w_2 = 0.035$. By substituting \eqref{eq:lemma5-proof2} and \eqref{eq:lemma5-proof3} into \eqref{eq:lemma5-proof1}, we achieve
\begin{align}
\PR\bigg( \NBtight >\delta \bigg) \leq 200n \exp\left(  -\frac{0.233n\delta^2}{\sigma^2}  \right) + 200n\exp\left( -\frac{0.291n\delta^2}{\sigma^2}  \right) \nonumber
\end{align}
which completes the proof.
\end{proof}

\begin{proof}[\textbf{Proof of Lemma \ref{lem:narrow-down-neighborhood}}:]
The idea applied in this proof is similar to that used in Lemma \ref{lem:no-false-negative}.
By applying Lemma \ref{lem:local-noise-bound} on the noise term in \eqref{eq:decomp-y}, the periodogram at a present frequency $j'/T_k$ can be lower bounded by
\begin{align}
\left|\tilde{y}_{k}(j'/T_k)\right| \ge |\tilde{\mu}_{k,j'}(j'/T_k)| - \bigg|\sum_{j\not= j'} \tilde{\mu}_{k,j}(j'/T_k)\bigg| - \NBtight \ge |b_{k,j'}| -\pi B_k U_2 - \frac{3}{4}\bar\epsilon_v \label{eq:lemma6-lowerbound}
\end{align}
with probability no less than $\ProbNBtight$.

Next we examine the periodogram in the region $\Hscr_{k,j'}\coloneqq[\frac{j'}{T_k}-\frac{g}{n},\frac{j'}{T_k}+\frac{g}{n}] \setminus [\frac{j'}{T_k}-\frac{1}{n},\frac{j'}{T_k}+\frac{1}{n}]$. Recall the definition of function $R(\theta)$ in \eqref{eq:Rtheta-rep}, and we note that for $v \in \Hscr_{k,j'}$,
\begin{align}
R(j'/T_k-v)\le \sup_{\theta\in [\frac{1}{n},\frac{2}{n}]}R(\theta) = \sup_{\theta\in [\frac{1}{n},\frac{2}{n}]} \left|\frac{\sin(\pi n\theta)}{n\sin(\pi\theta)}\right| \le  \frac{1}{n\sin(2\pi /n)}\le \frac{1}{4}, \nonumber
\end{align}
where the last step follows from that $n > 2gT_k \geq 4$ given in Assumption \ref{asp:period} and $\sin(x)\ge 2x/\pi$ for $x\in[0,\pi/2]$. By applying Lemma \ref{lem:local-noise-bound} again, we find that
\begin{align}
\sup_{v\in\Hscr_{k,j'}} |\tilde{y}_k(v)| &\leq |\tilde{\mu}_{k,j'}(v)| + \bigg|\sum_{j\neq j'}\tilde{\mu}_{k,j}(v)\bigg| +  \NBtight \nonumber \\
& \leq |b_{k,j'}| R(j'/T_k-v) + \pi B_k U_2 + \frac{3}{4}\bar\epsilon_v  \le \frac{1}{4}|b_{k,j'}| + \pi B_k U_2 + \frac{3}{4}\bar\epsilon_v \label{eq:lemma6-upperbound}
\end{align}
holds with the same probability no less than $\ProbNBtight$.

Note that the condition specified in the lemma leads to the following
\begin{align}
b_k \geq 2 \bar\epsilon_v + \frac{8\pi}{3}U_2B_k &~\implies~ |b_{k,j'}| \geq 2 \bar\epsilon_v + \frac{8\pi}{3}U_2B_k \nonumber \\
&~\implies~  |b_{k,j'}| -\pi B_k U_2 - \frac{3}{4}\bar\epsilon_v  \ge \frac{1}{4}|b_{k,j'}| + \pi B_k U_2 + \frac{3}{4}\bar\epsilon_v. \nonumber
\end{align}
By comparing \eqref{eq:lemma6-lowerbound} and \eqref{eq:lemma6-upperbound}, we conclude that $|\tilde{y}_k(v_{j'}^{*})| \ge \left|\tilde{y}_{k}(j'/T_k)\right| \ge \sup_{v\in\Hscr_{k,j'}} |\tilde{y}_k(v)|$. Therefore, $v_{j'}^{*} = \mathop{\arg\sup}_{v \in [\frac{j'}{T_k}-\frac{g}{n}, \frac{j'}{T_k}+\frac{g}{n}]} |\tilde{y}_k(v)|$ is not attained in $\Hscr_{k,j'}$ but in $[\frac{j'}{T_k}-\frac{1}{n},\frac{j'}{T_k}+\frac{1}{n}]$.
\end{proof}

\begin{proof}[\textbf{Proof of Theorem \ref{thm:freq-identification}}:]
Given Assumption \ref{asp:sub-Gaussian} - \ref{asp:bBratio}, by applying union bound we have Lemma \ref{lem:threshold}, Lemma \ref{lem:no-false-negative} and Lemma \ref{lem:narrow-down-neighborhood} hold together with probability no less than $\ProbNBall$.
In particular, we have $\sup_{v\in{\overline{\Vscr}_k}}|\tilde{y}_k(v)| \le \tau_k$, $|\tilde{y}_k(v_{j}^{*})| \ge |\tilde{y}_k(j/T_k)| > \tau_k$ and $v_{j}^{*} \in [\frac{j}{T_k}-\frac{1}{n},\frac{j}{T_k}+\frac{1}{n}]$, which guarantee that all present frequencies of $j/T_k$ with $|b_{k,j}|>0$ can be exactly found out by Algorithm \ref{alg:frequency-identification} with probability no less than $\ProbNBall$.
Therefore, the period $T_k$ can be correctly estimated accordingly.
Theorem \ref{thm:freq-identification} implies immediately by applying the above result of a single arm with union bound for $K$ arms.

We also remark on the values of parameters $n$, $g$ and $H$ used in Algorithm \ref{alg:frequency-identification}.
Choosing a large sample size $n$ obviously benefits the period estimation in stage one, but it leaves less time for stage two to exploit rewards.
Parameter $g$ is used in determining $U_1$, $U_2$ and the width of the neighborhood. A large $g$ relaxes the sample size requirement for condition \eqref{eq:assumptionbBratio}, but it limits the ability to estimate a long period as $T_k < \frac{n}{2g}$ is assumed in Assumption \ref{asp:period}.
Parameter $H$ is used in setting $\bar\epsilon_v$ and thus is also related to condition \eqref{eq:assumptionbBratio}.
A smaller $H$ makes \eqref{eq:assumptionbBratio} easier to satisfy, but it lowers the probability that all periods are correctly estimated.
By considering these trade-offs as a whole, we carefully derive a set of values in \eqref{eq:parameter} which can significantly ease the restriction in \eqref{eq:assumptionbBratio}, and we also note that Algorithm \ref{alg:frequency-identification} performs well empirically with this setting.
\end{proof}

\subsection{Proofs in Section~\ref{sec:analysis-phasetwo}}
\label{sec:appendix-proof-analysis-phasetwo}

The proofs of following lemmas lead to an upper bound on the regret of the oracle policy $\tpi$, and the proof outline is related to \cite{auer2002using} and \citet{li2017provably}. Recall that the oracle knows the periods of all arms in advance and applies the true values $T_k$ instead of the estimations $\hat{T}_k$ in Algorithm \ref{alg:learning}. As a consequence, we keep in mind that $\banditcount{k}{t}$, $\banditmean_{k,t}^{(s)}$ and $\banditwidth_{k,t}^{(s)}$ originally defined in \eqref{eq:def-countC}, \eqref{eq:alg2-s-mean} and \eqref{eq:alg2-s-ci} are modified as $\banditcountoracle{k}{t}$, $\banditmeanoracle_{k,t}^{(s)}$ and $\banditwidthoracle_{k,t}^{(s)}$ given in \eqref{eq:def-countC-oracle}, \eqref{eq:alg2-s-mean-oracle} and \eqref{eq:alg2-s-ci-oracle} respectively.

\begin{proof}[\textbf{Proof of Lemma \ref{lem:independent}}:]
An epoch $t$ can only be added to $\Psi^{(s)}(t)$ in Step \ref{step:wide-ci-end} of Algorithm \ref{alg:learning}, and this action of set expansion only depends on rewards observed in $\bar\Psi$ and at epochs $j \in \cup_{s'<s}\Psi^{(s')}(t)$ as well as the confidence intervals $\banditwidthoracle_{k,t}^{(s)}$ for $k \in \setk$. The definition \eqref{eq:alg2-s-ci-oracle} shows that $\banditwidthoracle_{k,t}^{(s)}$ does not depend on the values of the rewards observed from $\Psi^{(s)}(t)$. Hence, we conclude the proof.
\end{proof}

The conditional independence property established in Lemma \ref{lem:independent} allows us to apply concentration inequalities in the regret analysis. Next, we provide an additional technical lemma to facilitate the proof of Lemma \ref{lem:confidence_set}.

\begin{lemma}\label{lem:subGaussian-inequality}
Suppose that $\{X_t:t \in \N^{+}\}$ are independent sub-Gaussian random variables with parameter $\sigma$. Let $q_t = \frac{1}{t}\sum_{j=1}^{t} X_j$.  For any $\delta_0 \in (0,1)$, then $\PR\left(\exists\, t \geq 1, \abs{q_t} \geq \sqrt{\frac{4\sigma^2}{t} \log\left(\frac{4t}{\delta_0}\right)}\right) \leq \delta_0$.
\end{lemma}
\begin{proof}
Define $Q_n = \sum_{j=1}^{n} X_j$.  Since $Q_n$ is the sum of independent sub-Gaussian random variables, $Q_n$ is sub-Gaussian with parameter $\sqrt{n}\sigma$, and $\E[\exp(\lambda Q_n)] \leq \exp\left(\frac{1}{2}n \sigma^2 \lambda^2 \right)$ holds for all $\lambda\in \R$ according to the sub-Gaussian property described in Assumption \ref{asp:sub-Gaussian}. For any $\eta > 0$, we have the following by choosing $\lambda =\frac{\eta}{n\sigma^2}$,
\begin{align}
\PR\left(\exists t \leq n, Q_t \geq \eta \right) &= \PR\left(\max_{t \leq n}\exp(\lambda Q_t) \geq \exp(\lambda\eta)\right) \nonumber \\
&\leq \frac{\E\left[\exp(\lambda Q_n)\right]}{\exp(\lambda\eta)} \leq \exp\left(\frac{1}{2}n \sigma^{2}\lambda^{2} - \lambda \eta \right) = \exp\left(-\frac{\eta^{2}}{2n \sigma^{2}}\right). \nonumber
\end{align}
The first inequality holds from Doob's submartingale inequality and the fact that $Q_t$ is a submartingale with respect to the filtration generated by $X_1,...,X_n$.

By symmetry, we also have $\PR\left(\exists t \leq n, Q_t \leq -\eta\right) \leq \exp\left(-\frac{\eta^{2}}{2n \sigma^{2}}\right)$, and therefore
\begin{align}
\PR\left(\exists t \leq n, \abs{Q_t} \geq \eta \right) \leq 2\exp\left(-\frac{\eta^{2}}{2n \sigma^{2}}\right). \label{eq:lem-subGaussian-inequality-1}
\end{align}
Then, we can show
\begin{align}
& \PR\left(\exists\, t \geq 1, \abs{Q_t} \geq \sqrt{4\sigma^2 t \log\left(\frac{4t}{\delta_0}\right)}\right) \nonumber \\
\leq~ & \sum_{j = 0}^{\infty} \PR\left(2^{j} \leq t < 2^{j + 1},\abs{Q_t} \geq \sqrt{4\sigma^2 t \log\left(\frac{4t}{\delta_0}\right)}\right) \nonumber \\
\leq~ & \sum_{j = 0}^{\infty} \PR\left(1 \leq t \leq 2^{j + 1},\abs{Q_t} \geq \sqrt{4\sigma^2 \cdot 2^{j} \cdot \log\left(\frac{4 \cdot 2^j}{\delta_0}\right)}\right) \nonumber \\
\leq~ & \sum_{j = 0}^{\infty} 2\exp\left( -\frac{1}{2 \cdot 2^{j+1}\sigma^2} \cdot 4\sigma^2 \cdot 2^j \cdot \log\left(\frac{4 \cdot 2^j}{\delta_0}\right) \right) = \sum_{j = 0}^{\infty} \frac{\delta_0}{2^{j+1}} = \delta_0, \nonumber
\end{align}
where the first inequality follows the union bound and the last inequality applies the knowledge of \eqref{eq:lem-subGaussian-inequality-1}. The lemma follows immediately by dividing $Q_t$ by $t$.
\end{proof}

\begin{proof}[\textbf{Proof of Lemma \ref{lem:confidence_set}}:]
Recall that $\banditmeanoracle_{k,t}^{(s)}$ given in \eqref{eq:alg2-s-mean-oracle} computes the average of sample rewards selected from $\Psi^{(s)}(t) \cup \bar\Psi$. Let $\banditmeanoracle_{k,t}^{(s)}\big(\bar\Psi\big)$ and $\banditmeanoracle_{k,t}^{(s)}\big(\Psi^{(s)}(t)\big)$ denote the average of sample rewards selected from $\bar\Psi$ and $\Psi^{(s)}(t)$ separately, i.e.,
\begin{align}
\banditmeanoracle_{k,t}^{(s)}\big(\bar\Psi\big) &~=~ \frac{1}{\banditcountoracle{k}{t}\left(\bar\Psi\right)} \sum_{\substack{j \in \bar\Psi: \\ \tpi_j=k, ~ j \equiv t (\mathrm{mod}\, T_k) }} Y_{k,j}, \nonumber \\
\banditmeanoracle_{k,t}^{(s)}\big(\Psi^{(s)}(t)\big) &~=~ \frac{1}{\banditcountoracle{k}{t}\left(\Psi^{(s)}(t)\right)} \sum_{\substack{j \in {\Psi^{(s)}(t)}: \\ \tpi_j=k, ~ j \equiv t (\mathrm{mod}\, T_k) }} Y_{k,j}. \nonumber
\end{align}
Then, $\tilde{m}_{k,t}^{(s)}$ can be expressed as a linear combination
\begin{align}
\banditmeanoracle_{k,t}^{(s)} = \frac{\banditcountoracle{k}{t}\left(\bar\Psi\right)}{\banditcountoracle{k}{t}\left(\Psi^{(s)}(t) \cup \bar\Psi\right)} \banditmeanoracle_{k,t}^{(s)}\big(\bar\Psi\big)
 + \frac{\banditcountoracle{k}{t}\left(\Psi^{(s)}(t)\right)}{\banditcountoracle{k}{t}\left(\Psi^{(s)}(t) \cup \bar\Psi\right)} \banditmeanoracle_{k,t}^{(s)}\big(\Psi^{(s)}(t)\big). \label{eq:banditmeanoracle-linearexpression}
\end{align}
To simplify notations, we also introduce

\begin{align}
\banditwidthoracle_{k,t}^{(s)}\big(\bar\Psi\big) &~=~ \sqrt{\frac{4\sigma^2}{\banditcountoracle{k}{t}\left(\bar\Psi\right)}\log\left( \frac{8{d} \banditcountoracle{k}{t}\left(\bar\Psi\right)}{\delta} \right)}, \nonumber \\
\banditwidthoracle_{k,t}^{(s)}\big(\Psi^{(s)}(t)\big) &~=~ \sqrt{\frac{4\sigma^2}{\banditcountoracle{k}{t}\left( \Psi^{(s)}(t)\right)}\log\left( \frac{8{d} \banditcountoracle{k}{t}\left(\Psi^{(s)}(t)\right)}{\delta} \right)},  \nonumber
\end{align}
and then the confidence width $\banditwidthoracle_{k,t}^{(s)}$ given in \eqref{eq:alg2-s-ci-oracle} can also be expressed as a linear combination
\begin{align}
\banditwidthoracle_{k,t}^{(s)} &~=~ \frac{\banditcountoracle{k}{t}\left(\bar\Psi\right)}{\banditcountoracle{k}{t}\left(\Psi^{(s)}(t) \cup \bar\Psi\right)} \banditwidthoracle_{k,t}^{(s)}\big(\bar\Psi\big)
 + \frac{\banditcountoracle{k}{t}\left(\Psi^{(s)}(t)\right)}{\banditcountoracle{k}{t}\left(\Psi^{(s)}(t) \cup \bar\Psi\right)} \banditwidthoracle_{k,t}^{(s)}\big(\Psi^{(s)}(t)\big). \label{eq:banditwidthoracle-linearexpression}
\end{align}

We first examine samples selected from $\bar\Psi$. Since the rewards $\{Y_{k,j}: j\in \bar\Psi\}$ observed in stage one are independent sub-Gaussian random variables with mean $\mu_{k,t}$, by applying Lemma \ref{lem:subGaussian-inequality} with $\delta_0$ chosen as $\frac{\delta}{2d}$, then for all arm $k \in \setk$, round $s \in \sets$ and phase $p=1,...,T_k$, we can show
\begin{align}
\label{eq:sug-gaussian-bound-stage-one}
\PR\bigg(\bigcup_{\substack{t \in {\Psi^{(s)}(T)}: \\ \tpi_t=k,~ t \equiv p (\mathrm{mod}\, T_k)}} \left\{\left|\banditmeanoracle_{k,t}^{(s)}\big(\bar\Psi\big) -\mu_{k,t} \right| \ge \banditwidthoracle_{k,t}^{(s)}\big(\bar\Psi\big) \right\} \bigg) \le \frac{\sigma}{2d}.
\end{align}
Next we examine samples selected from $\Psi^{(s)}(t)$. Lemma \ref{lem:independent} states that the rewards $\{Y_{k,j}: j \in {\Psi^{(s)}(t)}, \tpi_j=k, j \equiv t (\mathrm{mod}\, T_k) \}$ are conditionally independent sub-Gaussian random variables with mean $\mu_{k,t}$. Again by applying Lemma \ref{lem:subGaussian-inequality} with $\delta_0$ chosen as $\frac{\delta}{2d}$, we have the following conditional probability bound
\begin{align}
\label{eq:sug-gaussian-bound-stage-two}
\PR\bigg(\bigcup_{\substack{t \in {\Psi^{(s)}(T)}: \\ \tpi_t=k,~ t \equiv p (\mathrm{mod}\, T_k)}} \bigg\{&\left|\banditmeanoracle_{k,t}^{(s)}\big(\Psi^{(s)}(t)\big) -\mu_{k,t} \right| \ge \banditwidthoracle_{k,t}^{(s)}\big(\Psi^{(s)}(t)\big) \nonumber \\
 &~\bigg|~ \Psi^{(s)}(t), \tpi_j \text{ for } j \in \Psi^{(s)}(t), Y_{\tpi_\tau,\tau} \text{ for } \tau \in \bar\Psi \bigg\} \bigg) \le \frac{\sigma}{2d}.
\end{align}
Taking expectation of both sides, the above bound holds for the unconditional probability as well.

Given the linear combination expressions of $\banditmeanoracle_{k,t}^{(s)}$ and $\banditwidthoracle_{k,t}^{(s)}$ in \eqref{eq:banditmeanoracle-linearexpression} and \eqref{eq:banditwidthoracle-linearexpression}, by applying union bound and using \eqref{eq:sug-gaussian-bound-stage-one} and \eqref{eq:sug-gaussian-bound-stage-two}, we have
\begin{align}
& \PR \bigg(\bigcup_{\substack{t \in {\Psi^{(s)}(T)}: \\ \tpi_t=k,~ t \equiv p (\mathrm{mod}\, T_k)}} \left\{ \left| \banditmeanoracle_{k,t}^{(s)} -\mu_{k,t} \right| \geq \banditwidthoracle_{k,t}^{(s)} \right\}  \bigg)\label{eq:lem-confidence_set-1}\\
~\le~ &\PR\bigg(\bigcup_{\substack{t \in {\Psi^{(s)}(T)}: \\ \tpi_t=k,~ t \equiv p (\mathrm{mod}\, T_k)}} \left\{\left|\banditmeanoracle_{k,t}^{(s)}\big(\bar\Psi\big) -\mu_{k,t} \right| \ge \banditwidthoracle_{k,t}^{(s)}\big(\bar\Psi\big) \right\} \bigg) \nonumber \\
& + \PR\bigg(\bigcup_{\substack{t \in {\Psi^{(s)}(T)}: \\ \tpi_t=k,~ t \equiv p (\mathrm{mod}\, T_k)}} \bigg\{\left|\banditmeanoracle_{k,t}^{(s)}\big(\Psi^{(s)}(t)\big) -\mu_{k,t} \right| \ge \banditwidthoracle_{k,t}^{(s)}\big(\Psi^{(s)}(t)\big) \bigg\} \bigg) \le \frac{\delta}{d}. \nonumber
\end{align}
Taking the union bound of \eqref{eq:lem-confidence_set-1} over $p$, $k$ and $s$, we have:
\begin{align}
\PR \bigg(\bigcup_{\substack{s \in \sets, ~k \in \setk, \\ p = 1,...,T_k}} \bigcup_{\substack{t \in {\Psi^{(s)}(T) }: \\ \tpi_t=k,~ t \equiv p (\mathrm{mod}\, T_k)}} \left\{ \left| \banditmeanoracle_{k,t}^{(s)} -\mu_{k,t} \right| \geq \banditwidthoracle_{k,t}^{(s)} \right\}  \bigg) \leq \frac{\delta}{d} S\sum_{k=1}^{K}T_k = \delta S, \nonumber
\end{align}
which leads to the lemma.
\end{proof}

\begin{proof}[\textbf{Proof of Lemma \ref{lem:single-step-regret}}:]
We prove part 1 by induction. The lemma holds for $s'=1$ and suppose that we have $\pi_t^* \in \Ascr_{s'}$ as well. When Algorithm \ref{alg:learning} proceeds to round $s'+1$, we know from Step \eqref{step:eliminate-arms-begin} that a narrow confidence bound less than $2^{-s'}\sigma$ is obtained for arms of round $s'$. Given event $\mathcal{E}$, we have $\left|\banditmeanoracle_{k,t}^{(s')} - \mu_{k,t} \right| \leq \banditwidthoracle_{k,t}^{(s')} \leq 2^{-s'}\sigma$ for all $k \in \Ascr_{s'}$. Then, the optimality of $\pi_t^* \in \Ascr_{s'}$ implies
 \begin{align}
\banditmeanoracle_{\pi_t^*,t}^{(s')} \geq \mu_{\pi_t^*,t} - 2^{-s'}\sigma \geq \mu_{k,t} - 2^{-s'}\sigma \geq \banditmeanoracle_{k,t}^{(s')} - 2^{1-s'}\sigma \nonumber
\end{align}
for all $k \in \Ascr_{s'}$, which guarantees that $\pi_t^*$ is selected to next round $s'+1$ by Step \ref{step:eliminate-arms}. Therefore, the lemma holds for $s'+1$ with $\pi_t^* \in \Ascr_{s'+1}$ and the induction follows.

Suppose $\tpi_{t}$ is chosen at Step \ref{step:wide-ci} in round $s$. If $s=1$, part 2 of the lemma holds obviously according to Assumption \ref{asp:mean-reward}. If $s\ge2$, since part 1 showed $\pi_t^* \in \Ascr_{s}$, the condition of Step \ref{step:eliminate-arms-begin} in round $s-1$ implies $\left|\banditmeanoracle_{k,t}^{(s-1)} - \mu_{k,t} \right| \leq 2^{1-s}\sigma$ for both $k=\tpi_t$ and $k=\pi_t^*$, and Step \ref{step:eliminate-arms} in stage $s-1$ implies $\banditmeanoracle_{\tpi_t,t}^{(s-1)} \ge \banditmeanoracle_{\pi_t^*,t}^{(s-1)} - 2^{2-s}\sigma$. By combining these inequalities together, we can prove part 3 as
\begin{align}
\mu_{\tpi_t,t} \ge \banditmeanoracle_{\tpi_t,t}^{(s-1)} - 2^{1-s}\sigma \ge \banditmeanoracle_{\pi_t^*,t}^{(s-1)} - 3\cdot2^{1-s}\sigma \ge \mu_{\pi_t^*,t} - 4\cdot2^{1-s}\sigma. \nonumber
\end{align}

If $\tpi_t$ is chosen in Step~\ref{step:narrow-ci}, then we have $\banditmeanoracle_{\tpi_t,t}^{(s)} \ge \banditmeanoracle_{\pi_t^*,t}^{(s)}$ and $\left|\banditmeanoracle_{k,t}^{(s)} - \mu_{k,t} \right| \leq \frac{\sigma}{\sqrt{T}}$ for both $k=\tpi_t$ and $k=\pi_t^*$. Therefore, part 4 follows through a similar argument as that used in the proof above,
\begin{align}
\mu_{\tpi_t,t} \ge \banditmeanoracle_{\tpi_t,t}^{(s)} - \frac{\sigma}{\sqrt{T}} \ge \banditmeanoracle_{\pi_t^*,t}^{(s-1)} - \frac{\sigma}{\sqrt{T}} \ge \mu_{\pi_t^*,t} - \frac{2\sigma}{\sqrt{T}}. \nonumber
\end{align}
We complete the proof.
\end{proof}

Lemma \ref{lem:intermediate-bound} a technical result that is used in the proof of Lemma~\ref{lem:oracle-bound}.
\begin{lemma}\label{lem:intermediate-bound}
For all $s \in \sets$, then $\displaystyle{\sum_{t\in\Psi^{(s)}(T)} \banditwidthoracle_{\tpi_t,t}^{(s)} \leq  4\sigma\sqrt{|\Psi^{(s)}(T)| d\log\left( \frac{8T}{\delta} \right) \log\left(\frac{T}{d} \right)} }$.
\end{lemma}
\begin{proof}
Recall the definition of $\banditwidthoracle_{\tpi_t,t}^{(s)}$  in \eqref{eq:alg2-s-ci-oracle}.
To bound the first term, we have
\begin{align}
& \sum_{t\in \Psi^{(s)}(T)}\frac{\banditcountoracle{\tpi_t}{t}\left(\bar\Psi\right)}{\banditcountoracle{\tpi_t}{t}\left(\Psi^{(s)}(t) \cup \bar\Psi\right)}\sqrt{\frac{4\sigma^2}{\banditcountoracle{\tpi_t}{t}\left(\bar\Psi\right)}\log\left( \frac{8{d} \banditcountoracle{\tpi_t}{t}\left(\bar\Psi\right)}{\delta} \right)}\nonumber\\
~=~& \sum_{t\in \Psi^{(s)}(T)}2\sigma\sqrt{\frac{\banditcountoracle{\tpi_t}{t}\left(\bar\Psi\right)}{\banditcountoracle{\tpi_t}{t}^2\left(\Psi^{(s)}(t) \cup \bar\Psi\right)}\log\left( \frac{8{d} \banditcountoracle{\tpi_t}{t}\left(\bar\Psi\right)}{\delta} \right)}\nonumber\\
~\le~& 2\sigma\sqrt{|\Psi^{(s)}(T)| \sum_{t\in \Psi^{(s)}(T)}\frac{\banditcountoracle{\tpi_t}{t}\left(\bar\Psi\right)}{\banditcountoracle{\tpi_t}{t}^2\left(\Psi^{(s)}(t) \cup \bar\Psi\right)}\log\left( \frac{8{d} \banditcountoracle{\tpi_t}{t}\left(\bar\Psi\right)}{\delta} \right)} \label{eq:w-first-term-step1} \\
~\le~& 2\sigma\sqrt{|\Psi^{(s)}(T)| \sum_{t\in \Psi^{(s)}(T)}\frac{1}{\banditcountoracle{\tpi_t}{t}\left(\Psi^{(s)}(t) \cup \bar\Psi\right)}\log\left( \frac{8{d} \banditcountoracle{\tpi_t}{t}\left(\Psi^{(s)}(t) \cup \bar\Psi\right) }{\delta} \right)} ~. \label{eq:w-first-term-step2}
\end{align}
To obtain \eqref{eq:w-first-term-step1}, we apply Jensen's inequality: $\sum_{j=1}^{J}\sqrt{x_j} \le \sqrt{J\sum_{j=1}^J x_j}$.
To get \eqref{eq:w-first-term-step2}, we use the fact that  $\banditcountoracle{\tpi_t}{t}\left(\bar\Psi\right)\le \banditcountoracle{\tpi_t}{t}\left(\Psi^{(s)}(t) \cup \bar\Psi\right)$ since the counting function is non-decreasing when the argument  set expands.

We reorganize the sum term in the square root of \eqref{eq:w-first-term-step2} by grouping over $k$ and $p$ as following,
\begin{align}
& \sum_{t\in \Psi^{(s)}(T)}\frac{1}{\banditcountoracle{\tpi_t}{t}\left(\Psi^{(s)}(t) \cup \bar\Psi\right)}\log\left( \frac{8{d} \banditcountoracle{\tpi_t}{t}\left(\Psi^{(s)}(t) \cup \bar\Psi\right) }{\delta} \right) \nonumber \\
~=~& \sum_{k=1}^K\sum_{p=1}^{T_k} \sum_{\substack{t\in\Psi^{(s)}(T):  \\ \tpi_t=k,~ t\equiv p(\mathrm{mod}\, T_k)}}\frac{1}{\banditcountoracle{k}{p}\left(\Psi^{(s)}(t) \cup \bar\Psi\right)}\log\left( \frac{8{d} \banditcountoracle{k}{p}\left(\Psi^{(s)}(t) \cup \bar\Psi\right) }{\delta} \right) \nonumber \\
~\le~& \sum_{k=1}^K\sum_{p=1}^{T_k} \log\left( \frac{8{d} \banditcountoracle{k}{p}\left(\Psi^{(s)}(T) \cup \bar\Psi\right) }{\delta} \right) \sum_{\substack{t\in\Psi^{(s)}(T):  \\ \tpi_t=k,~ t\equiv p(\mathrm{mod}\, T_k)}}\frac{1}{\banditcountoracle{k}{p}\left(\Psi^{(s)}(t) \cup \bar\Psi\right)} \label{eq:w-first-term-step3} \\
~\le~& \sum_{k=1}^K\sum_{p=1}^{T_k} \log\left( \frac{8{d} \banditcountoracle{k}{p}\left(\Psi^{(s)}(T) \cup \bar\Psi\right) }{\delta} \right) \sum_{j=2}^{\banditcountoracle{k}{p}\left(\Psi^{(s)}(T) \cup \bar\Psi\right)}\frac{1}{j} \label{eq:w-first-term-step4} \\
~\le~& \sum_{k=1}^K\sum_{p=1}^{T_k} \log\left( \frac{8{d} \banditcountoracle{k}{p}\left(\Psi^{(s)}(T) \cup \bar\Psi\right) }{\delta} \right) \log\left(\banditcountoracle{k}{p}\left(\Psi^{(s)}(T) \cup \bar\Psi\right)\right) \label{eq:w-first-term-step5} \\
~\le~& d \log\left( \frac{8{d}}{\delta} \cdot \frac{1}{d}\sum_{k=1}^K\sum_{p=1}^{T_k}\banditcountoracle{k}{p}\left(\Psi^{(s)}(T) \cup \bar\Psi\right)\right) \log\left(\frac{1}{d}\sum_{k=1}^K\sum_{p=1}^{T_k}\banditcountoracle{k}{p}\left(\Psi^{(s)}(T) \cup \bar\Psi\right)\right) \label{eq:w-first-term-step6} \\
~\le~& d\log\left( \frac{8T}{\delta} \right) \log\left(\frac{T}{d} \right). \label{eq:w-first-term-step7}
\end{align}
To derive \eqref{eq:w-first-term-step3}, we use the fact that $\banditcountoracle{k}{p}\left(\Psi^{(s)}(t) \cup \bar\Psi\right) \le \banditcountoracle{k}{p}\left(\Psi^{(s)}(T) \cup \bar\Psi\right)$.
Noticing the counting nature of $\banditcountoracle{k}{p}$, we rewrite the last sum term in \eqref{eq:w-first-term-step3} to obtain \eqref{eq:w-first-term-step4}, where the corresponding summation index starts from $j=2$ because an arm was pulled at each phase for at least twice during stage one given $T_k < \frac{n}{2g} < \frac{n}{2}$ in Assumption \ref{asp:period}.
Hence, we use the upper bound on the harmonic series $\sum_{j=2}^{J}j^{-1} \le \log(J)$ to establish \eqref{eq:w-first-term-step5}.
We apply Jensen's inequality to obtain \eqref{eq:w-first-term-step6} because $\log(\frac{8dx}{\delta})\log(x)$ is a concave function on $x \ge 2$ given $d\ge2$ and $\delta \in (0,1)$ and $\sum_{k=1}^K\sum_{p=1}^{T_k}1=d$. Inequality \eqref{eq:w-first-term-step7} holds due to the fact that $\sum_{k=1}^K\sum_{p=1}^{T_k}\banditcountoracle{k}{p}\left(\Psi^{(s)}(T) \cup \bar\Psi\right)\le T$.

Finally, by plugging \eqref{eq:w-first-term-step7} back into \eqref{eq:w-first-term-step2}, we can show
\begin{align}
\sum_{t\in \Psi^{(s)}(T)}\frac{\banditcountoracle{\tpi_t}{t}\left(\bar\Psi\right)}{\banditcountoracle{\tpi_t}{t}\left(\Psi^{(s)}(t) \cup \bar\Psi\right)}\sqrt{\frac{4\sigma^2}{\banditcountoracle{\tpi_t}{t}\left(\bar\Psi\right)}\log\left( \frac{8{d} \banditcountoracle{\tpi_t}{t}\left(\bar\Psi\right)}{\delta} \right)} \le 2\sigma\sqrt{|\Psi^{(s)}(T)| d\log\left( \frac{8T}{\delta} \right) \log\left(\frac{T}{d} \right)} . \nonumber
\end{align}
The above derivation works through on the second term in \eqref{eq:alg2-s-ci-oracle} as well, and thus we also have
\begin{align}
&\sum_{t\in\Psi^{(s)}(T)}\frac{\banditcountoracle{\tpi_t}{t}\left(\Psi^{(s)}(t)\right)}{\banditcountoracle{\tpi_t}{t}\left(\Psi^{(s)}(t) \cup \bar\Psi\right)}\sqrt{\frac{4\sigma^2}{\banditcountoracle{\tpi_t}{t}\left( \Psi^{(s)}(t)\right)}\log\left( \frac{8{d} \banditcountoracle{\tpi_t}{t}\left(\Psi^{(s)}(t)\right)}{\delta} \right)} \le 2\sigma\sqrt{|\Psi^{(s)}(T)| d\log\left( \frac{8T}{\delta} \right) \log\left(\frac{T}{d} \right)} . \nonumber
\end{align}
Summing up the these two terms completes the proof.
\end{proof}

\begin{proof}[\textbf{Proof of Lemma \ref{lem:oracle-bound}}:]
Suppose that event $\mathcal{E}$ holds on a sample path, and thus the corresponding pseudo regret incurred by the oracle policy can be bounded as following
\begin{align}\label{eq:orcal-bound_regret_decomposition}
R_{T}^{\tpi} \leq  nK + |\Psi^{(1)}(T)| + \sum_{s = 2}^{S}\frac{8\sigma}{2^{s}} |\Psi^{(s)}(T)| + \frac{2\sigma}{\sqrt{T}}\left(T-nK-\sum_{s=1}^{S}|\Psi^{(s)}(T)|\right).
\end{align}
Note that the first term in \eqref{eq:orcal-bound_regret_decomposition} bounds the regret accumulated in stage one, and the other three terms bounds the regret incurred in stage two by applying parts 2 - 4 of Lemma \ref{lem:single-step-regret}.

We examine the second and third terms in \eqref{eq:orcal-bound_regret_decomposition} and combine them together as
\begin{align}
|\Psi^{(1)}(T)| + \sum_{s = 2}^{S}\frac{8\sigma}{2^{s}} |\Psi^{(s)}(T)| \le \max\left\{\frac{1}{4\sigma},1\right\} \sum_{s = 1}^{S}\frac{8\sigma}{2^{s}} |\Psi^{(s)}(T)| \le \bigg(\frac{2}{\sigma}+8\bigg)\sum_{s = 1}^{S}\frac{\sigma}{2^{s}} |\Psi^{(s)}(T)|. \label{eq:orcal-bound_regret_1}
\end{align}
Note that Step \ref{step:wide-ci-begin} of Algorithm \ref{alg:learning} ensures that $\banditwidthoracle_{\tpi_t,t}^{(s)} \ge 2^{-s}\sigma$ holds for any round $s \in \sets$. Hence, we continue deriving the inequality \eqref{eq:orcal-bound_regret_1} as following
\begin{align}
\eqref{eq:orcal-bound_regret_1} &~\leq~ \bigg(\frac{2}{\sigma}+8\bigg)\sum_{s = 1}^{S} \sum_{t\in\Psi^{(s)}(T)} \banditwidthoracle_{\tpi_t,t}^{(s)} \nonumber \\
&~\leq~ (8+32\sigma) \sum_{s = 1}^{S} \sqrt{|\Psi^{(s)}(T)| d\log\left( \frac{8T}{\delta} \right) \log\left(\frac{T}{d} \right)}  \label{eq:orcal-bound_regret_step2} \\
&~\leq~ (8+32\sigma) \sqrt{S\sum_{s = 1}^{S}|\Psi^{(s)}(T)| d\log\left( \frac{8T}{\delta} \right) \log\left(\frac{T}{d} \right)}  \label{eq:orcal-bound_regret_step3} \\
&~\leq~ (10+40\sigma) \sqrt{T d \log(T) \log\left( \frac{8T}{\delta} \right) \log\left(\frac{T}{d} \right)} ~.  \label{eq:orcal-bound_regret_step4}
\end{align}
We apply Lemma \ref{lem:intermediate-bound} to establish  \eqref{eq:orcal-bound_regret_step2} and apply Jensen's Inequality to develop \eqref{eq:orcal-bound_regret_step3}. The last step \eqref{eq:orcal-bound_regret_step4} is due to $\sum_{s = 1}^{S}|\Psi^{(s)}(T)| \le T$ and $\sqrt{S} \leq \sqrt{\frac{\log(T)}{\log 2}} \leq \frac{5}{4} \sqrt{\log(T)}$ as $S=\floor{\log_2 T}$.

We also have $nK \le \sqrt{TK}$ as $n = \lfloor\sqrt{T/K}\rfloor$, so the bound conditional on event $\mathcal{E}$ given in \eqref{eq:orcal-bound_regret_decomposition} can be further derived as
\begin{align}
R_{T}^{\tpi} &~\leq~ \sqrt{TK} + (10+40\sigma) \sqrt{T d \log(T) \log\left( \frac{8T}{\delta} \right) \log\left(\frac{T}{d} \right)} + 2\sigma \sqrt{T}. \nonumber
\end{align}
Given $\PR(\mathcal{E}) \ge 1- \delta S$ by Lemma \ref{lem:confidence_set}, if choosing $\delta = 8T^{-1}$, we have
\begin{align}
\E[R_{T}^{\tpi} \mathbf{1}_{\{\mathcal{E}\}}] &~\leq~ \sqrt{TK} + (10+40\sigma) \sqrt{T d \log(T) \log\left(T^2\right) \log\left(\frac{T}{d} \right)} + 2\sigma \sqrt{T} \nonumber \\
&~\leq~ \left(11\sqrt{2}+42\sqrt{2}\sigma\right) \sqrt{T d \log^2(T) \log\left(\frac{T}{d} \right)} \nonumber \\
\E[R_{T}^{\tpi} \mathbf{1}_{\{\mathcal{E}^{c}\}}] &~\leq~ T(1-\PR(\mathcal{E})) \le T\cdot\frac{8S}{T} = 8S = 8\floor{\log_2 T} \le 9\sqrt{2}\log(T). \nonumber
\end{align}
Therefore, the expected regret of the oracle policy can be bounded as
\begin{align}
\E[R_{T}^{\tpi}] = \E[R_{T}^{\tpi} \mathbf{1}_{\{\mathcal{E}\}}] + \E[R_{T}^{\tpi} \mathbf{1}_{\{\mathcal{E}^{c}\}}] \leq \left(20\sqrt{2}+42\sqrt{2}\sigma\right) \sqrt{T d \log^2(T) \log\left(\frac{T}{d} \right)}, \nonumber
\end{align}
which leads to the lemma $\E[R_{T}^{\tpi}] \leq \textit{Constant} \cdot \RegretBound$ where the \textit{Constant} is not related to $T$, $K$ or any $T_k$ for $k \in \setk$.
\end{proof}

\subsection{Proofs in Section~\ref{sec:lower-bound}}
\label{sec:proof-lower-bound}

The following proofs provide the regret lower bounds for three classes of bandits, namely $\banditset_1$, $\banditset_2(T_1)$ and $\banditset_3(T_1,...,T_K)$, discussed in Section \ref{sec:lower-bound}.
In particular, $\banditset_1$ is defined as the class of $K$-armed unit-variance Gaussian bandits where all arms are stationary,
$\banditset_2(T_1)$ is similar to $\banditset_1$ except that its first arm has a (minimum) period $T_1$,
and $\banditset_3(T_1,...,T_K)$ is the class where all arms are periodic with (minimum) periods $T_1,...,T_K$.
By these definitions, bandits in these classes satisfy Assumption \ref{asp:mean-reward} and \ref{asp:sub-Gaussian}, and we assume the periods of arms satisfy Assumption \ref{asp:period} as well.
For simplicity, we abbreviate $\banditset_2(T_1)$ as $\banditset_2$ and $\banditset_3(T_1,...,T_K)$ as $\banditset_3$ when it does not effect reading.

\begin{proof}[\textbf{Proof of Lemma~\ref{lem:lowerbound-no-arm-periodic}:}]
This Lemma is an easy implication of the regret lower bound $\Omega(\sqrt{TK})$ of the classic stationary MAB problem.
For completeness, we provide an independent proof here.
The techniques can be referred to Chapter 15 of \cite{Lattimore2020banditalgo} (hereafter LS2020).

We first study a simple case where $\banditset_{1}$ is the class of two-armed bandits.
Let $\Delta \in (0,0.5)$ be a parameter to be chosen later.
Consider two instances $\nu_1$ and $\nu_2$ in $\banditset_1$,
specifically $\nu_1$ with mean rewards $\{ \mu_1^{(\nu_1)}=0.5+\Delta,  ~ \mu_2^{(\nu_1)}=0.5 \}$
and $\nu_2$ with mean rewards $\{ \mu_1^{(\nu_2)}=0.5-\Delta,  ~ \mu_2^{(\nu_2)}=0.5 \}$,
where the superscript on $\mu$ denotes the MAB instance to which the arm belongs.

Let $C_1(T)$ count the number of epochs when the first arm is pulled during the horizon $T$, and $C_2(T)$ count that of the second arm.
For an arbitrary policy $\pi$, we have
\begin{align*}
\E_{\nu_1}[R^\pi_T] \geq \PR_{\nu_1}\left(C_1(T) \leq T/2 \right) \frac{T\Delta}{2}
\quad \text{and} \quad
\E_{\nu_2}[R^\pi_T] \geq \PR_{\nu_2}\left(C_1(T) > T/2 \right) \frac{T\Delta}{2} .
\end{align*}
Since $2\max\{a,b\} \ge a + b$, we have
\begin{align*}
\sup_{\nu\in\banditset_1}\E_{\nu}[R_T^\pi] \geq \frac{1}{2}\big(\E_{\nu_1}[R^\pi_T] + \E_{\nu_2}[R^\pi_T]\big)
\geq \big(\PR_{\nu_1}\left(C_1(T) \leq T/2\right) + \PR_{\nu_2}\left(C_1(T) > T/2 \right) \big) \frac{T\Delta}{4}.
\end{align*}
Then we apply the Bretagnolle-Huber inequality (Theorem 14.2 in LS2020) and get
\begin{align*}
\sup_{\nu\in\banditset_1}\E_{\nu}[R_T^\pi] \geq \exp\left( - D(\PR_{\nu_1},\PR_{\nu_2})\right) \frac{T\Delta}{8},
\end{align*}
where $D(\PR_{\nu_1},\PR_{\nu_2})$ is the Kullback-Leibler divergence (also known as the relative entropy) from $\PR_{\nu_1}$ to $\PR_{\nu_2}$.
We can further simplify $D(\PR_{\nu_1},\PR_{\nu_2})$ using Lemma 15.1 in LS2020 to obtain
\begin{align*}
D(\PR_{\nu_1},\PR_{\nu_2}) = & \E_{\nu_1}\left[C_{1}(T)\right] D\left( \Nscr(0.5 + \Delta,1), \Nscr(0.5 - \Delta,1) \right) + \E_{\nu_1}\left[C_{2}(T)\right] D\left( \Nscr(0.5,1), \Nscr(0.5,1) \right) \\
= & 2\E_{\nu_1}\left[C_{1}(T)\right]\Delta^2 \leq 2T\Delta^2.
\end{align*}
If we set $\Delta = \sqrt{\frac{1}{2T}}$, then we have
\begin{align}
\label{eq:lowerbound-2arm-stationary}
\sup_{\nu\in\banditset_1}\E_{\nu}[R_T^\pi] \geq \exp\left( -2T\Delta^2 \right)\frac{T\Delta}{8} = \frac{1}{8\sqrt{2}e}\sqrt{T}.
\end{align}

For $\banditset_{1}$ with bandits of $K>2$ arms, since $\mu_2,...,\mu_K$ are already known, we can always choose the best arm among $K\ge2$ and discard the suboptimal ones at the beginning of the decision horizon, i.e., it actually reduces to the case of $K=2$.
Hence, the lower bound in \eqref{eq:lowerbound-2arm-stationary} applies and this proves the result.
\end{proof}

\begin{proof}[\textbf{Proof of Lemma~\ref{lem:lowerbound-one-arm-periodic}:}]
{By keeping the best arm and discarding suboptimal ones among arms $k\ge2$, we only need to consider a reduced case of $\banditset_{2}$ with two arms.}

Given $T_1$ is known, we can treat the decision scenarios at the same phase of each cycle as an independent stationary MAB,
and hence each instance in $\banditset_2(T_1)$
can be decomposed into $T_1$ independent sub-instances, each of which belongs to $\banditset_1$ and takes a decision horizon at least of $\lfloor T/T_1 \rfloor$.
Applying Lemma~\ref{lem:lowerbound-no-arm-periodic}, we have
\begin{align}
\label{eq:lowerbound-1arm-periodic}
\sup_{\nu\in\banditset_2}\E_{\nu}[R_T^\pi] \geq T_1 \sup_{\nu'\in\banditset_1}\E_{\nu'}[R_{\lfloor T/T_1 \rfloor}^\pi] \ge \frac{T_1}{8 \sqrt{2}e} \sqrt{\lfloor T/T_1 \rfloor} \ge \frac{1}{16e} \sqrt{T T_1}.
\end{align}
The last inequality holds since Assumption \ref{asp:period} implies $\lfloor \frac{T}{T_1} \rfloor \ge \frac{T}{\sqrt{2}T_1}$.
This completes the proof.
\end{proof}

\begin{proof}[\textbf{Proof of Theorem~\ref{thm:lower-bound}:}]
Since all arms are symmetric, we only need to prove
\begin{align}
\label{eq:lowerbound-allarm-periodic-T1}
\sup_{\nu\in\banditset_3}\E_{\nu}[R_T^\pi] \geq \frac{1}{32e } \sqrt{T T_1}.
\end{align}

When $T_2=...=T_K=1$, $\banditset_3(T_1,...,T_K)$ degenerates to $\banditset_2(T_1)$, and hence we can  invoke Lemma \ref{lem:lowerbound-one-arm-periodic} to obtain \eqref{eq:lowerbound-allarm-periodic-T1}.
However, if we consider $\banditset_3$ with $T_k\geq2$ for some $k\ge2$, i.e., certain arms are indeed non-stationary, the above argument does not proceed.
Note that the expected regret $f_\pi(\nu) \coloneqq \E_{\nu}[R_T^\pi]$ can be viewed as a functional of $\nu  \in \banditset_{2} \cup \banditset_{3}$ with parameter $\pi$.
To establish a general lower bound, we introduce a metric for any pair of MAB instances $(\nu_1,\nu_2)$ in $\banditset_{2} \cup \banditset_{3}$ %$\banditset_2(T_1)\cup\banditset_3(T_1,...,T_K)$
as
\begin{align}
\label{eq:bandit-metric}
\Mscr(\nu_1,\nu_2) =\left(\sum_{k=1}^{K}\sum_{p=1}^{T_k}\left(\mu_{k,p}^{(\nu_1)} - \mu_{k,p}^{(\nu_2)}\right)^2\right)^{\frac{1}{2}}.
\end{align}

We have the following two facts with the metric $\Mscr$:
\begin{enumerate}
\item $\banditset_{3}$ is \emph{dense} in $\banditset_{2} \cup \banditset_{3}$.
For any instance $\nu_1 \in \banditset_{2}(T_1)$,
by perturbing its mean rewards $\mu_{k,t}^{(\nu_1)}$ for certain arms $k\ge2$,
we can make these arms periodic and hence obtain a new instance $\nu_2 \in \banditset_{3}(T_1,...,T_K)$ which has the required periods.
For any $\delta>0$, we can make such perturbation slight enough to achieve $\Mscr(\nu_1,\nu_2) < \delta$.

\item $f_\pi(\nu)$ is \emph{continuous} in $\nu$ and this continuity is \emph{uniform} with respect to $\pi$.
Specifically, given $\nu_1 \in \banditset_{2} \cup \banditset_{3}$, for any $\epsilon>0$, there exists $\delta>0$ such that $\abs{f_\pi(\nu_1) - f_\pi(\nu_2)} < \epsilon$ holds for any $\nu_2 \in \banditset_{2} \cup \banditset_{3}$ with $\Mscr(\nu_1,\nu_2) <\delta$.
\end{enumerate}

For any policy $\pi$ that knows $T_k$'s, Lemma \ref{lem:lowerbound-one-arm-periodic} implies that some $\nu_1 \in \banditset_2$ yields $f_\pi(\nu_1) \geq \frac{1}{16e} \sqrt{T T_1}$.
As a result of above two facts, when choosing $\epsilon = \frac{1}{32e} \sqrt{T T_1}$ and slightly perturbing $\nu_1$ accordingly, we can always find a suitable $\nu_2 \in \banditset_{3}$ satisfying $\abs{f_\pi(\nu_1) - f_\pi(\nu_2)} < \epsilon$. Therefore, we have
\begin{align}
\label{eq:lowerbound-allarm-periodic}
    f_\pi(\nu_2) \geq f_\pi(\nu_1) - \abs{f_\pi(\nu_1) - f_\pi(\nu_2)} \geq \frac{1}{32e} \sqrt{T T_1},
\end{align}
which leads to Theorem~\ref{thm:lower-bound}.

The first fact is quite obvious.
For example, given a set of periods $\{T_1,...,T_K\}$ and any $\nu_1 \in \banditset_{2}(T_1)$,
we may create $\nu_2$ by setting its mean rewards as the following perturbation,
\begin{equation}
\mu_{k,t}^{(\nu_2)} =
\left\{
    \begin{array}{ll}
      \mu_{k,t}^{(\nu_1)} + \delta/\sqrt{2K}, & \quad  \text{for any arm $k\ge2$ with $T_k\ge2$ and any $t \equiv 1 (\mathrm{mod}\, T_k)$}, \\
      \mu_{k,t}^{(\nu_1)}, & \quad  \text{otherwise}. \\
    \end{array} \right. \nonumber
\end{equation}
Therefore, $\nu_2$ has the required periods, i.e., $\nu_2 \in \banditset_{3}(T_1,...,T_K)$, and satisfies $\Mscr(\nu_1,\nu_2) < \delta$.

Next we prove the second fact.
Given a bandit instance, we denote the random actions and rewards generated by a policy $\pi$ as $(\pi_1,Y_{\pi_1, 1},\dots ,\pi_T,Y_{\pi_T, T})$, and we denote the realized actions and rewards as $(a_1,y_1,\dots ,a_T,y_T)$.
The pseudo-regret can be expressed as
$R_T^\pi(\pi_1,Y_{\pi_1, 1},\dots ,\pi_T,Y_{\pi_T, T}) = \sum_{t = 1}^{T}\Delta_{a_t,t}$, where $\Delta_{a_t,t}\le 1$ is the gap between the optimal mean reward and the mean reward of arm $a_t$ pulled at epoch $t$.
We integrate $R_T^\pi $ over the distribution of \( (\pi_1,Y_{\pi_1, 1},\dots ,\pi_T,Y_{\pi_T, T}) \), which can be represented by
\begin{align*}
q(a_1,y_1,\dots ,a_T,y_T) =\prod_{t = 1}^{T}\pi(a_t|a_1,y_1,\dots ,a_{t - 1},y_{t - 1}) p(y_t -\mu_{a_t,t}),
\end{align*}
where $p(\cdot)$ is the PDF of the standard normal distribution.
Then we have
\begin{align*}
f_\pi(\nu) =&~\int R_T^\pi(a_1,y_1,\dots ,a_T,y_T) q(a_1,y_1,\dots ,a_T,y_T) d\psi(a_1,y_1,\dots ,a_T,y_T)\\
=&~\int \sum_{t = 1}^{T}\Delta_{a_t,t}\prod_{t = 1}^{T}\pi(a_t|a_1,y_1,\dots ,a_{t - 1},y_{t - 1}) p(y_t -\mu_{a_t,t}) \mathrm{d} \psi(a_1,y_1,\dots ,a_T,y_T),
\end{align*}
where $\psi$ is the product of the counting measures and Lebesgue measures on $(a_1,y_1,\dots ,a_T,y_T)$.

For two instances $\nu_1$ and $\nu_2$ in $\banditset_{2} \cup \banditset_{3}$, we have
\begin{align*}
\abs{f_\pi(\nu_1) - f_\pi(\nu_2)} \leq T \int \left|\prod_{t = 1}^{T} p\left(y_t -\mu^{(\nu_1)}_{a_t,t}\right) -\prod_{t = 1}^{T} p\left(y_t -\mu^{(\nu_2)}_{a_t,t}\right)\right| \mathrm{d} \psi.
\end{align*}
We then invoke Lebesgue's dominated convergence theorem.
For fixed $(a_1,y_1,\dots ,a_T,y_T)$, when $\Mscr(\nu_1,\nu_2) \to 0$, we must have
\begin{equation*}
\left|\prod_{t = 1}^{T} p\left(y_t -\mu^{(\nu_1)}_{a_t,t}\right) -\prod_{t = 1}^{T} p\left(y_t -\mu^{(\nu_2)}_{a_t,t}\right)\right| \to 0.
\end{equation*}
Moreover, it is easy to see that it is dominated by an integrable function because of the light tail of Gaussian distributions and the fact that $\mu_{a_t,t}^{(\nu_1)}, \mu_{a_t,t}^{(\nu_2)}\in[0,1]$ .
As a result, \( \abs{f_\pi(\nu_1) - f_\pi(\nu_2)}\to0 \) and the convergence is uniform with respect to $\pi$.
This proves the second fact.
\end{proof}

\section{Periodic MAB with the Same Period}
\label{sec:appendix-extended-results}
In the main text of this paper, we consider a very general periodic MAB model where we do not impose any assumption on the length of period or the structure of variation for the mean reward of any arm.
We develop a two-stage policy with a regret upper bound $\tilde{O}(\sqrt{T\sum_{k=1}^K T_k})$ and we show a general lower bound $\Omega(\sqrt{T\max_{k}\{ T_k\}})$ for any policy.
In practice, the DM may face special (or simplified) settings where the periodic evolution of mean rewards enjoy certain possess.
For example, all arms may have the same period.
It is a common case that the demands of various products tend to have the same seasonal pattern in the learning and earning problem of dynamic pricing.
Regardless of the pricing policy, the seasonality of the demand, in particular, its period, is typically exogenous and largely unaltered.

In this appendix, we investigate a specific setting of the periodic MAB model where the mean rewards of all arms vary over time with the same period, i.e., $T_1=...=T_K$.
We assume that Assumptions \ref{asp:mean-reward}, \ref{asp:sub-Gaussian} and \ref{asp:period} still hold.
Comparing to Algorithm \ref{alg:learning} designed for the general periodic MAB problem, it is possible to develop more efficient learning algorithms with smaller regret upper bound by exploiting the property of the same period, and it is also possible to establish a matching regret lower bound.
We may still use Algorithm \ref{alg:frequency-identification} to estimate the (common) length of periods.
Recall that it incurs a regret of $O(\sqrt{TK})$ in stage one and the regret caused by the wrong estimation is negligible in stage two.
Because estimating periods does not incur significant regret, to simplify the discussion, we may assume $T_k$'s are known in advance and focus on the learning of rewards.

\subsection{Same Period without Other Structures}
\label{sec:appendix-extended-results-SamePeriodDifferentSeasonality}
When $T_1=...=T_K$, the learning problem can be decomposed into $T_1$ subproblems, each of which is a classic stationary MAB problem with $K$ arms and learning horizon $T/T_1$.
Without additional seasonality structures, these subproblems are independent of each other, and hence need to be learned separately.
As a result, the optimal rate of regret $O(T_1\sqrt{KT/T_1 })=O(\sqrt{KTT_1})$ can be achieved when applying classic MAB algorithms, e.g., UCB, to each subproblem.
Note that in this case, the regret bound in Theorem~\ref{thm:upper_bound} matches the lower bound.

\subsection{Same Period with a Fixed Best Arm}
\label{sec:appendix-extended-results-SamePeriodSameSeasonality}
In addition to sharing a common period, suppose that the best arm remains as the same over time.
More precisely, there exists $k\in \Kscr$ such that $\mu_{k,t}=\max_{k'\in \Kscr}\mu_{k',t}$ for all $t=1,\dots,T_1$.
This is a frequently encountered situation in practice, for example, arms (selling of products) have the same seasonality and their mean rewards (demands) evolve in parallel.
Given the knowledge that the best arm is fixed, we propose the following sequential-elimination-based learning policy.

\begin{algorithm}
    %\SingleSpacedXI
    %\OneAndAHalfSpacedXI
    \caption{sequential-elimination-based learning}
    \label{alg:learning-sequential-elimination}

    \begin{algorithmic}[1]
        \State Input: $T$, $K$, $\sigma$ and $T_1$
        \State Initialize $\Ascr_1 \gets \left\{1,\dots,K\right\}$
        \For{$s=1,2,...$}
            \State Choose $n_s = \ceil{2^{2+2s}\log(KTs^2)/T_1}T_1$ \label{alg:learning-sequential-elimination-numberpulled}
            \State Pull each arm $k \in \Ascr_s$ exactly $n_s$ times and let $\banditmean_{k}^{(s)}$ be the average of these rewards
            \State Update active set $\Ascr_{s+1} \gets \left\{k \in \Ascr_s: \banditmean_{k}^{(s)} \geq \max_{k' \in \Ascr_s} \banditmean_{k'}^{(s)} - 2^{-s}\sigma \right\}$
        \EndFor
    \end{algorithmic}
\end{algorithm}

Comparing to Algorithm \ref{alg:learning}, Algorithm \ref{alg:learning-sequential-elimination} relies on the similar idea of gradually eliminating clearly suboptimal arms. However, the implementation is simpler.
Since one arm always has the highest mean reward at each phase of the common period, we pool the information together to infer the optimal arm instead of tracking each individual phase separately.
Specifically, the number of each arm pulled in round $s$, i.e., $n_s$ chosen in Step \ref{alg:learning-sequential-elimination-numberpulled}, is a multiple of $T_1$, hence the estimator $\banditmean_{k}^{(s)}$ is for the average of mean rewards $\sum_{t=1}^{T_1}\mu_{k,t}/T_1$ in a period.
In the following, we show that Algorithm \ref{alg:learning-sequential-elimination} achieves a regret upper bound $\tilde{O}(\sqrt{TK})$, which matches the regret lower bound $\Omega(\sqrt{TK})$ for any policy in this setting.

\begin{proof}
The regret lower bound $\Omega(\sqrt{TK})$ is provided by the result of the classic stationary MAB model as it is a special case here.
For the regret upper bound, Algorithm \ref{alg:learning-sequential-elimination} is extended from Exercise 6.8 in LS2020 which applies to the classic stationary MAB problem,
and our regret analysis tackles the difficulty introduced by the periodic mean rewards.

Without loss of generality, we assume that arm 1 is the best arm.
Let $\bar{\mu}_k := \frac{1}{T_1}\sum_{t=1}^{T_1}\mu_{k,t}$ denote the average of mean rewards in a period for arm $k \in \setk$.
Let $\bar{\Delta}_k := \bar{\mu}_1 - \bar{\mu}_k$ denote the average suboptimal gap of arm $k$.
Given $\mu_{k,t} \in [0,1]$ imposed by Assumption \ref{asp:mean-reward}, we have $\bar{\Delta}_k \in [0,1]$.

%We first study the possibility of eliminating arm 1. Using the definition of Algorithm \ref{alg:learning-sequential-elimination}, we have
We first study arm 1.
Since it is the best arm, we hope it is always kept in the active set, and we control the probability of accidental elimination using the definition of Algorithm \ref{alg:learning-sequential-elimination} as following.
\begin{align*}
\PR \left( 1 \notin \Ascr_{s+1},  ~1 \in \Ascr_{s} \right) & ~\leq~ \PR \left(1 \in \Ascr_{s}, ~ \exists k \in \Ascr_{s}\setminus\{1\}: ~ \banditmean_{k}^{(s)}- \banditmean_{1}^{(s)} > 2^{-s}\sigma  \right)\nonumber \\
& ~\leq~ \sum_{k=2}^{K} \PR \left(\banditmean_{k}^{(s)}- \banditmean_{1}^{(s)} > 2^{-s}\sigma  \right)\nonumber \\
& ~\leq~ \sum_{k=2}^{K} \PR \left( \frac{1}{2\sigma}\left( (\banditmean_{k}^{(s)}-\bar{\mu}_k) - (\banditmean_{1}^{(s)}-\bar{\mu}_1) \right) > \frac{1}{2}(\bar{\Delta}_k/\sigma + 2^{-s})  \right)\nonumber \\
& ~\leq~ K\exp\left(-\frac{n_s(\bar{\Delta}_k/\sigma + 2^{-s})^2}{4} \right) ,
\end{align*}
where the last inequality is due to the concentration for the mean of $2n_s$ independent unit-variance sub-Gaussian random variables.
For the event that arm 1 is eliminated, since we choose $n_s \ge 2^{2+2s}\log(KTs^2)$, we have
\begin{align}
\PR \left(\exists s: 1 \notin \Ascr_{s} \right) \leq \sum_{s=1}^{\infty}\PR \left( 1 \notin \Ascr_{s+1},  1 \in \Ascr_{s} \right) \leq K \sum_{s=1}^{\infty} \exp\left(-\frac{n_s2^{-2s}}{4} \right) \leq K \sum_{s=1}^{\infty} \frac{1}{KTs^2}  ~\leq~ \frac{2}{T}. \label{eq:alg3-term1}
\end{align}

Next we study the elimination of a suboptimal arm $k$.
Similarly, we can establish
\begin{align*}
\PR \left( k \in \Ascr_{s+1}, ~ 1 \in \Ascr_{s},  ~ k \in \Ascr_{s} \right) & ~\leq~ \PR \left(1 \in \Ascr_{s}, ~ k \in \Ascr_{s}, ~ \banditmean_{k}^{(s)}- \banditmean_{1}^{(s)} > -2^{-s}\sigma  \right)\nonumber \\
& ~\leq~ \exp\left(-\frac{n_s(\bar{\Delta}_k/\sigma - 2^{-s})^2}{4} \right).
\end{align*}
Define $s_k \coloneqq \min \left\{s \ge 1: 2^{-s+1} \le \bar{\Delta}_k/\sigma \right\}$, which is a round highly unlikely to be reached by arm $k$.
For the event that arm $k$ stays in the active set and is not eliminated in round $s_k$, we have
\begin{align}
\PR \left(k \in \Ascr_{s_k + 1} \right) & ~\leq~ \PR \left( k \in \Ascr_{s_k+1}, 1 \in \Ascr_{s_k},  k \in \Ascr_{s_k} \right) + \PR \left(1 \notin \Ascr_{s_k}\right) \nonumber \\
& ~\leq~ \exp\left(-\frac{n_{s_k}(\bar{\Delta}_k/\sigma - 2^{-s_k})^2}{4} \right) + \frac{2}{T} \nonumber \\
& ~\leq~ \exp\left(-\frac{n_{s_k} 2^{-2s_k}}{4} \right) + \frac{2}{T} \nonumber \\
& ~\leq~ \frac{3}{T}  . \label{eq:alg3-term2}
\end{align}
There is also a natural upper bound $\lceil T/(KT_1)\rceil$ for the number of screening rounds that arm $k$ can survive, because each round takes at least $KT_1$ epochs.
Here we make a mild assumption that the decision horizon is not too short, so at least one screen round can be completed for all arms, i.e., $T \geq KT_1 (1 + 4 \log(KT))$.

Let $C_k(T)$ count the total number of epochs that arm $k$ has been pulled.
We can derive the following by using \eqref{eq:alg3-term2} and $n_s \le 2^{2+2s}\log(KTs^2) + T_1$.
\begin{align}
\E \left[C_k(T)\right] & ~\leq~  T \PR \left(k \in \Ascr_{s_{k+1}} \right) ~+ \sum_{s=1}^{s_k \land \lceil  T/(KT_1)\rceil } n_s \nonumber \\
& ~\leq~ 3 ~+ \sum_{s=1}^{s_k \land \lceil  T/(KT_1)\rceil } 2^{2+2s}\log(KTs^2) + T_1 s_k . \nonumber
\end{align}
If $\bar\Delta_k \geq \sigma$, then $s_k = 1$. We have
\begin{align}
\E \left[C_k(T)\right] ~\leq~  3 + 4\log(KT) + T_1.  \label{eq:alg3-Ck-term1}
\end{align}
If $\bar\Delta_k < \sigma$, the definition of $s_k$ implies $\bar\Delta_k/\sigma < 2^{-s_k + 2}$, which leads to $2^{2s_k} < 16\sigma^2/\bar{\Delta}_k^2$ and $s_k < \log_{2}(4\sigma/\bar{\Delta}_k)$. We have
\begin{align}
\E \left[C_k(T)\right] & ~\leq~ 3 + \log(KT\ceil{T/(KT_1)}^2)\sum_{s=1}^{s_k} 2^{2+2s} + T_1 \log_{2}(4\sigma/\bar{\Delta}_k)  \nonumber \\
& ~\leq~ 3 + \frac{256\sigma^2}{3\bar\Delta_k^2} \log(KT^3) + T_1 \log_{2}(4\sigma/\bar{\Delta}_k) , \label{eq:alg3-Ck-term2}
\end{align}
where we bound the geometric series and use a simple fact $\ceil{T/(KT_1)} < T$ in the last step.
After combining \eqref{eq:alg3-Ck-term1} and \eqref{eq:alg3-Ck-term2}, we can bound the expected number of pulls as
\begin{align}
\E \left[C_k(T)\right] & ~\leq~ 3 + 4\log(KT) + \frac{256\sigma^2}{3\bar\Delta_k^2} \log(KT^3) + T_1 \left(1 \lor \log_{2}(4\sigma/\bar{\Delta}_k)\right) . \label{eq:alg3-Ck}
\end{align}

The proof relies on the regret decomposition (Lemma 4.5 in LS2020).
Due to the special design of Algorithm \ref{alg:learning-sequential-elimination} that
each arm is explored for a certain amount of complete periods, i.e., $n_s$ chosen in Step \ref{alg:learning-sequential-elimination-numberpulled} is a multiple of $T_1$, we may utilize the average suboptimal gap over a period to express the total regret incurred by pulling arm $k$ as $\bar\Delta_k \E[C_k(T)]$.
There is a trivial issue that the last period of the arm pulled at the end of $T$ may be incomplete if $T$ is not a multiple of $T_1$.
It can be simply fixed by extending the decision horizon $T$ to the nearest multiple of $T_1$, namely $\tilde{T} = \lceil T/T_1 \rceil T_1$, and then we can show
\begin{align}
\E[R^\pi_T] \leq \E[R_{\tilde T}^\pi] = \sum_{k =1}^K \bar\Delta_k \E[C_k(\tilde{T})] & \leq \sum_{k =1}^K \bar\Delta_k \E[C_k(T)] + \sum_{k = 1}^K \E[C_k(\tilde{T}) - C_k(T)], \nonumber \\
\E[R^\pi_T] & \leq \sum_{k =1}^K \bar\Delta_k \E[C_k(T)] + T_1. \label{eq:alg3-regret-decomposition}
\end{align}
Above inequalities hold because $\bar\Delta_k \leq  1$ and the total number of epochs added from $T$ to $\tilde{T}$, i.e., $\sum_{k = 1}^K \E[C_k(\tilde{T}) - C_k(T)]$, is less than $T_1$.

Finally, we bound the regret by plugging \eqref{eq:alg3-Ck} into \eqref{eq:alg3-regret-decomposition} as
\begin{equation}
\label{eq:alg3-regret-termall}
\begin{split}
\E[R^\pi_T]  ~\leq~ & \underbrace{\sum_{k =1}^K \bar\Delta_k \left(3 + 4\log(KT)\right)}_{\text{1st term}} ~+~ \underbrace{\sum_{k =1}^K \left( \frac{256\sigma^2}{3\bar\Delta_k} \log(KT^3) \right) \land \left(\bar\Delta_k \E[C_k(T)]\right)}_{\text{2nd term}}  \\
& ~+~ \underbrace{\sum_{k =1}^K \bar\Delta_k \left( T_1 \left(1 \lor \log_{2}(4\sigma/\bar{\Delta}_k)\right) \right) + T_1}_{\text{3rd term}} .
\end{split}
\end{equation}
For terms in \eqref{eq:alg3-regret-termall}, we have
\begin{align}
\label{eq:alg3-regret-term1}
\text{1st term} & \le  \left(3 + 4\log(KT)\right)\sum_{k =1}^K \bar\Delta_k \le  \left(3 + 4\log(KT)\right)K, \\
\label{eq:alg3-regret-term2}
\text{2nd term} & \le  \frac{16\sigma}{\sqrt{3}} \sqrt{\log(KT^3)} \sum_{k =1}^K \sqrt{\E[C_k(T)]}  \le  \frac{16\sigma}{\sqrt{3}} \sqrt{\log(KT^3)} \sqrt{TK}, \\
\label{eq:alg3-regret-term3}
\text{3rd term} & \le T_1 \sum_{k =1}^K \left(\bar\Delta_k \lor 4\sigma \right) + T_1 \le T_1 K (1 + 4\sigma)+ T_1.
\end{align}
When deriving \eqref{eq:alg3-regret-term2}, we use the following facts, $\min\{a/\bar{\Delta}_k, b\bar{\Delta}_k\} \le \sqrt{ab}$ for $a, b \ge 0$,
Jensen's inequality $\sum_{k =1}^K \sqrt{\E[C_k(T)]} \le  \sqrt{K\sum_{k =1}^K\E[C_k(T)]} $ and $\sum_{k =1}^K\E[C_k(T)]=T$.

The 1st term in \eqref{eq:alg3-regret-term1} and the 2nd term in \eqref{eq:alg3-regret-term2} are $\tilde{O}{(\sqrt{TK})}$.
The 3rd term in \eqref{eq:alg3-regret-term3} is also $\tilde{O}{(\sqrt{TK})}$ if $T_1 \le \sqrt{T/K}$, which is actually a milder requirement than Assumption \ref{asp:period}. Therefore, we conclude that Algorithm \ref{alg:learning-sequential-elimination} attains a regret upper bound $\tilde{O}{(\sqrt{TK})}$.
\end{proof}

\end{document}